\documentclass[10pt]{article} 
\usepackage[accepted]{tmlr}

\usepackage{amsmath,amssymb,amsthm}
\usepackage{mathtools}
\usepackage{thmtools}  
\usepackage{bm}
\usepackage{booktabs}
\usepackage{multirow}
\usepackage{graphicx}
\graphicspath{{figures/}{./figures/}}
\newcommand{\figorbox}[2][0.92\linewidth]{%
  \IfFileExists{#2}{\includegraphics[width=#1]{#2}}%
  {\fbox{\parbox{0.9\linewidth}{\centering\vspace{1.2cm}\texttt{#2} not found\\(upload the \texttt{figures/} folder to render)\vspace{1.2cm}}}}}
\usepackage{tikz}
\usetikzlibrary{arrows.meta,positioning,decorations.pathreplacing,shapes.geometric,shapes.misc,fit,backgrounds}
\usepackage{algorithm}
\usepackage{algpseudocode}
\usepackage{xcolor}
\usepackage{colortbl}
\newif\ifrevmark
\revmarkfalse
\newcommand{\rev}[1]{\ifrevmark{\color{blue}#1}\else#1\fi}
\newcounter{savedtable}
\usepackage[font=small,skip=6pt]{caption}
\AtBeginDocument{%
  \setlength{\abovedisplayskip}{5pt plus 1pt minus 3pt}%
  \setlength{\belowdisplayskip}{5pt plus 1pt minus 3pt}%
  \setlength{\abovedisplayshortskip}{2pt plus 1pt minus 1pt}%
  \setlength{\belowdisplayshortskip}{3pt plus 1pt minus 2pt}%
}
\usepackage{hyperref}
\usepackage[capitalise,noabbrev]{cleveref}

\crefname{theorem}{Theorem}{Theorems}
\Crefname{theorem}{Theorem}{Theorems}
\crefname{proposition}{Proposition}{Propositions}
\Crefname{proposition}{Proposition}{Propositions}
\crefname{corollary}{Corollary}{Corollaries}
\Crefname{corollary}{Corollary}{Corollaries}
\crefname{lemma}{Lemma}{Lemmas}
\Crefname{lemma}{Lemma}{Lemmas}
\crefname{definition}{Definition}{Definitions}
\Crefname{definition}{Definition}{Definitions}
\crefname{assumption}{Assumption}{Assumptions}
\Crefname{assumption}{Assumption}{Assumptions}
\crefname{remark}{Remark}{Remarks}
\Crefname{remark}{Remark}{Remarks}

\newtheorem{theorem}{Theorem}
\newtheorem{proposition}[theorem]{Proposition}
\newtheorem{corollary}[theorem]{Corollary}
\newtheorem{lemma}[theorem]{Lemma}
\newtheorem{definition}[theorem]{Definition}
\newtheorem{assumption}{Assumption}
\newtheorem{remark}{Remark}

\newcommand{\E}{\mathbb{E}}
\newcommand{\Var}{\mathrm{Var}}
\newcommand{\Cov}{\mathrm{Cov}}

\newcommand{\R}{\mathbb{R}}

\newcommand{\Dtrain}{\mathcal{D}_{\mathrm{train}}}
\newcommand{\Dclean}{\mathcal{D}_{\mathrm{clean}}}
\newcommand{\Dbridge}{\mathcal{D}_{\mathrm{bridge}}}

\newcommand{\yclean}{\hat{y}_{\mathrm{clean}}}
\newcommand{\ycorrupt}{\hat{y}_{\mathrm{corrupt}}}
\newcommand{\ybayes}{\hat{y}_{\mathrm{Bayes}}}
\newcommand{\ycleanbayes}{\hat{y}_{\mathrm{Bayes,clean}}}

\newcommand{\mudelta}{\mu_\delta}
\newcommand{\tnu}{\tilde{\nu}}
\newcommand{\gperp}{g_\perp}
\newcommand{\rclean}{r_{\mathrm{clean}}}
\newcommand{\Psignal}{P_{\mathrm{signal}}}
\newcommand{\Rshift}{R_{\mathrm{shift}}}
\newcommand{\Rnoise}{R_{\mathrm{noise}}}
\newcommand{\Rcross}{R_{\mathrm{cross}}}

\newcommand{\Radj}{\mathrm{Risk}_{\mathrm{adj}}}
\newcommand{\Rmeas}{\mathrm{Risk}_{\mathrm{meas}}}
\newcommand{\Bhat}{\widehat{B}}

\newcommand{\bary}{\bar{y}}
\newcommand{\PRC}{\mathrm{PRC}}  
\newcommand{\PRChat}{\widehat{\mathrm{PRC}}}

\title{Temporal Leakage in LLM Backtesting:\\Measurement, Validation, and Adjusted Scores}

\author{\name Zeyu Zhang\thanks{Corresponding author.} \email zeyuzhang2028@u.northwestern.edu \\
      \addr Department of Statistics and Data Science\\
      Northwestern University
      \AND
      \name Bradly C. Stadie \email bstadie@northwestern.edu \\
      \addr Department of Statistics and Data Science\\
      Northwestern University}

\def\month{11}
\def\year{2026}
\def\openreview{\url{https://openreview.net/forum?id=x3qagPELrp}}

\begin{document}

\maketitle

\begin{abstract}
The standard check for contamination in LLM backtests is simple: compare scores before and
after the training cutoff. We show this check is uninformative. Four flagship models fail it
on questions they cannot have memorized: every scored question resolved after their cutoffs.
The reason is structural. Models legitimately know more about times near their cutoff, so
recency mimics leakage, and we prove no passive backtest can separate the two from genuine
skill. Measurement, not just detection, requires information from outside the backtest. We
supply it in two forms. A known cutoff identifies leakage at the boundary; a matched clean
control identifies it globally and yields a leakage-adjusted score. \rev{Under a convex-pull model, leakage concentrates on outcomes that surprised the crowd and that training covered well. Partial memorization gives a large part of the full inflation.} \rev{In twin models with planted leakage, the estimators detect and localize it, track its dose, and return null on clean questions.} \rev{On frontier models, the estimators detect one cutoff-localized signature. They find no significant leakage at the tested boundaries in five models with positive naive gaps.} Backtests need not be discarded; they
need one defensible reference.

\end{abstract}

\section{Introduction}
\label{sec:intro}
\label{sec:exp:m1}

Large language models are increasingly evaluated by \emph{backtesting}: the model is asked,
as of a historical date, about events whose outcomes have since resolved, and its score is
read as evidence of how it would forecast the genuine future
\citep{halawi2024forecasting,karger2025forecastbench,lopezlira2025memorization,paleka2025pitfalls,gao2025lookahead}.
The threat is temporal leakage: a web-scale training corpus may already describe the outcome
being ``predicted,'' so a high score can reflect memorization rather than skill. The standard
defense is a pre/post check, which compares scores before and after the model's training
cutoff and flags a model whose advantage sits on the pre-cutoff side
\citep{livecodebench2025,lookaheadbench2026}. \Cref{fig:m1forest} shows that check failing.
We scored five flagship models on the official ForecastBench panel, restricted to questions
that resolved after their documented training cutoffs, so no outcome can be in their training
data. Four of the five fail the check anyway, with spurious ``leakage'' gaps as large as
$+0.061$. We read the documented cutoffs as bounding all training data. The matched-window
check of \Cref{app:m1} supports this reading: in a common window, all five gaps collapse
together.

\begin{figure}[t]\centering
\figorbox[0.59\linewidth]{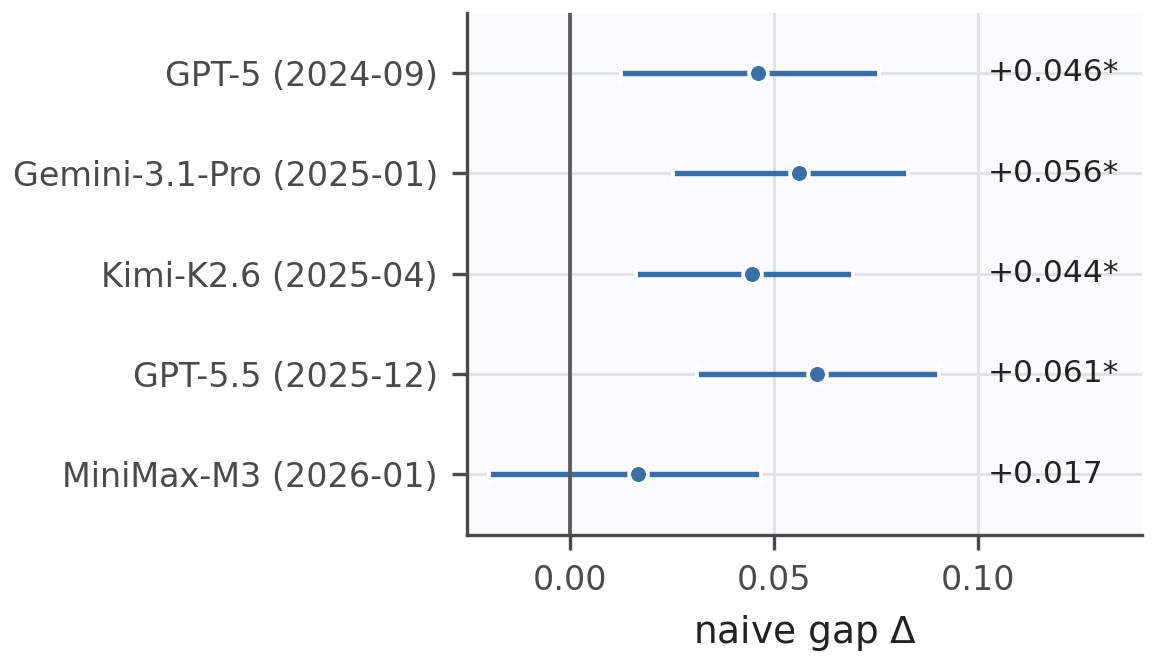}
\caption{\textbf{M1: models that cannot have leaked still fail the naive contamination check.}
Each row is one model (cutoff month in parentheses), scored only on questions resolving at
least 30 days after its cutoff, so no outcome can be in its training data. The x-axis is its
naive gap $\Delta=\bar s_{\mathrm{early}}-\bar
s_{\mathrm{late}}$, the window split at its median resolution date, where
$s=(c_0-Y)^2-(P-Y)^2$ is the Brier improvement of the model's forecast $P$ over the crowd
forecast $c_0$ (higher is better); $\Delta>0$ means better on older questions, the pattern an
audit reads as leakage. Whiskers: cluster-bootstrap 95\% CIs; $^\ast$: CI excludes
zero. Result: four of five flagships are flagged,
$\Delta=+0.044$ to $+0.061^\ast$; recency alone produces the leakage pattern.}
\label{fig:m1forest}
\end{figure}

The failure is structural, not a flaw of one benchmark. A backtest score mixes three
components. \emph{Skill} is what the evaluation wants to measure. \emph{Recency} is
legitimate knowledge that varies with distance from the cutoff: training data cover recent
periods more densely, and post-cutoff knowledge grows stale \citep{lazaridou2021mind}.
\emph{Leakage} is memorized outcome information, and it is illegitimate. The same cutoff
drives the last two: a question resolving near the cutoff is one the model legitimately knows
more about and one whose outcome is likelier to sit in the corpus. A
pre/post contrast therefore measures recency and leakage jointly, and recency alone
reproduces the leakage signature. Other failure modes---hallucination, miscalibration,
prompt sensitivity \citep{huang2025survey,xiong2024can,sclar2024quantifying}---degrade
honest performance on both sides of the cutoff;
leakage alone is specific to backtesting, and it inflates the score.

The question that matters for practice is therefore not whether contamination exists---it
usually does \citep{sainz2023nlp,golchin2024time}---but how much it inflates the score, and
whether the score can be corrected. Existing tools do not answer
it. Detection methods establish that contamination exists, not what it is worth
\citep{oren2024proving,codec2026}, and prompt-level remedies fail to suppress recall
\citep{lopezlira2025memorization}. Clean-reference estimates on static benchmarks do not
transfer, because a backtest defines ``clean'' by a date rather than by membership
\citep{singh2024contamination,haimes2024benchmark}. Retraining a truly clean reference model
fails, because its cutoff would have to predate every evaluation question, leaving a model
too outdated to stand in for the one under audit \citep{drinkall2024time}.

This paper answers the quantitative question in four steps: we prove the inflation is
non-identifiable from forecasts and outcomes alone, derive where it concentrates, show that
one external reference restores measurement, and validate the estimators against ground
truth by planting leakage in twin models. \rev{On frontier models, the estimators detect one cutoff-localized signature. They find no significant leakage at the tested boundaries in five models with positive naive gaps.}

\paragraph{Contributions.}
\begin{itemize}
\item \textbf{The standard check is uninformative (\Cref{sec:impossibility}).} Leakage
inflation cannot be identified from backtest scores alone: the identified set is a sharp
interval that more data does not shrink. A flat pre/post profile is not evidence of a clean
backtest, and \Cref{fig:m1forest} shows the converse failure on real flagships.
\item \textbf{Leakage is not a rate. \rev{Under convex pull, it obeys a closed-form law}
(\Cref{sec:concentration}).} Per
question, inflation equals honest uncertainty times an extraction factor with a double
benefit: half the leaked signal already yields three quarters of the full inflation. Leakage
concentrates where the crowd was surprised and training coverage was dense, so an average
contamination rate is the wrong audit object.
\item \textbf{A validated audit recipe
(\Cref{sec:estimation,sec:validation,sec:deployment}).} Choose a reference---a known cutoff
(regression discontinuity) or a matched clean control (difference-in-differences)---run the
matching estimator, and report the leakage-adjusted score with its assumption; paraphrase
probes only detect. In twin models the estimators \rev{localize planted leakage and track its dose and per-question profile} and return null on a clean twin.
\end{itemize}
\section{Problem Setup}
\label{sec:setup}

\paragraph{The mechanism.}
Two dates govern temporal leakage, and \Cref{fig:dag} makes the mechanism precise. The
\emph{as-of date} $t_0$ is the information horizon a genuine forecast may use; the deployed
model's \emph{training cutoff} is $T>t_0$. The world state up to $t_0$ generates the training
corpus $\Dclean$ available to a genuine forecaster, while the later state $W_+$ is a latent
common cause of the ``bridge'' training data $\Dbridge$ (text written after $t_0$ but before
$T$) and of the realized outcome $Y$. A model trained on $\Dclean\cup\Dbridge$ can therefore
acquire information about $Y$ through the back-door path $W_+\to\Dbridge\to\ycorrupt$, the
red leakage path in \Cref{fig:dag}, without the benchmark question ever appearing verbatim in
training. Temporal knowledge leakage is this back-door flow; it separates a model trained
with the bridge data ($\ycorrupt$) from an otherwise identical one trained without it
($\yclean$). A question such as ``\emph{will country $X$ enter a recession by Q4 2023?}'' is
leakage-eligible if it resolves before the cutoff, where the model may simply recall the
reported outcome, and clean if it resolves after, where the model must reason.

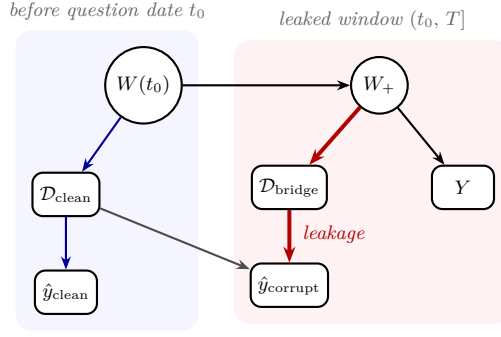
\begin{figure}[t]
\centering
\begin{tikzpicture}[scale=0.82, transform shape,
    node distance=1.0cm and 1.5cm,
    every node/.style={font=\small},
    latent/.style={draw, circle, minimum size=0.85cm, thick, align=center},
    observed/.style={draw, rectangle, rounded corners, minimum height=0.7cm, minimum width=1.0cm, thick, align=center, fill=white},
    worldpath/.style={-{Stealth[length=5pt]}, thick},
    trainarrow/.style={-{Stealth[length=5pt]}, thick, black!70},
    leakpath/.style={-{Stealth[length=6pt]}, line width=1.5pt, red!70!black},
    cleanpath/.style={-{Stealth[length=5pt]}, thick, blue!60!black},
    regionlabel/.style={font=\small\itshape, text=black!60},
  ]
  \node[latent, fill=white] (W0) {$W(t_0)$};
  \node[latent, fill=white, right=2.7cm of W0] (Wplus) {$W_+$};
  \node[observed, below left=0.95cm and 0.25cm of W0] (Dclean) {$\Dclean$};
  \node[observed, below right=0.95cm and 0.5cm of Wplus] (fstar) {$Y$};
  \node[observed, below left=0.95cm and 0.5cm of Wplus] (Dbridge) {$\Dbridge$};
  \node[observed, below=0.85cm of Dclean] (yclean) {$\yclean$};
  \node[observed, below=0.85cm of Dbridge] (ycorrupt) {$\ycorrupt$};
  \draw[worldpath] (W0) -- (Wplus);
  \draw[worldpath] (Wplus) -- (fstar);
  \draw[cleanpath] (W0) -- (Dclean);
  \draw[cleanpath] (Dclean) -- (yclean);
  \draw[trainarrow] (Dclean) -- (ycorrupt);
  \draw[leakpath] (Wplus) -- (Dbridge);
  \draw[leakpath] (Dbridge) --
    node[right=2pt, font=\small\itshape, text=red!70!black, pos=0.45] {leakage}
    (ycorrupt);
  \begin{scope}[on background layer]
    \node[fit=(W0)(Dclean)(yclean), inner sep=6pt, rounded corners=6pt,
          fill=blue!4, draw=none] (preregion) {};
    \node[fit=(Wplus)(Dbridge)(fstar)(ycorrupt), inner sep=6pt, rounded corners=6pt,
          fill=red!5, draw=none] (postregion) {};
  \end{scope}
  \node[regionlabel, above=1pt of preregion.north] {before question date $t_0$};
  \node[regionlabel, above=1pt of postregion.north] {leaked window $(t_0,\,T]$};
\end{tikzpicture}
\caption{\textbf{Temporal leakage is an indirect information flow through a common cause of
the training text and the outcome, not verbatim overlap.}
The model's training cutoff $T$ postdates the as-of date $t_0$; a question is
leakage-eligible when it resolves inside $(t_0,T]$. The later world state $W_+$
resolves the outcome $Y$ (black, top) and also generates the bridge data $\Dbridge$, text
written inside $(t_0,T]$. The deployed model trains on $\Dclean$ \emph{plus}
$\Dbridge$, so the red path $W_+\!\to\!\Dbridge\!\to\!\ycorrupt$ carries outcome information
into it without the benchmark question ever appearing in training; the counterfactual clean
model $\yclean$ (blue) learns from $\Dclean$ alone.}
\label{fig:dag}
\end{figure}

\paragraph{Notation.}
We audit one backtested model with reference cutoff $T$; uncertainty in $T$ is handled by
sensitivity analysis rather than treated as known. For a random question we observe the
realized outcome $Y\in\{0,1\}$; the model's forecast $P\in[0,1]$ under a fixed deterministic
protocol; the \emph{running variable} $G$, the signed gap between the question's resolution
date and the cutoff, negative for leakage-eligible questions ($g<0$) and positive for clean
ones; and a leakage-free difficulty covariate $X$ observable without the model, such as the
contemporaneous crowd forecast. We score with the
Brier loss $\ell=(P-Y)^2$, a strictly proper score; the identification results of
\Cref{sec:estimation} require only a bounded loss. The central observable is the conditional
mean loss
\begin{equation}
  m(x,g)\;=\;\E\big[(P-Y)^2 \mid X=x,\ G=g\big],
  \label{eq:m}
\end{equation}
estimable by regression wherever the backtest has data.

\paragraph{The target.}
Leakage lowers the pre-cutoff loss below what the model's legitimate information could
achieve. Write $m_0(x,g)$ for the \emph{honest surface}: the loss the deployed model would
incur using legitimate information only. It still varies with $g$, because legitimate
knowledge depends on temporal distance from the cutoff; that variation is recency, and the
theory imposes no shape or direction on it. The paper's target is the \emph{inflation}
\begin{equation}
  B\;=\;\E\big[\,m_0(X,G)-m(X,G)\mid G<0\,\big]\;\ge\;0,
  \label{eq:Bdef}
\end{equation}
the average amount by which the backtest understates the model's honest error on
leakage-eligible questions. The goal of the paper is an estimate of the inflation with a
confidence interval, and the resulting \emph{leakage-adjusted} score. \Cref{app:notation}
collects symbols and conventions; \Cref{app:proofs} has the proofs.
\section{Scores Alone Cannot Measure Leakage}
\label{sec:impossibility}
\label{sec:nonident}
\label{sec:nonident:struct}
\label{sec:nonident:nonid}

A completed backtest yields \emph{passive data}: draws of $(P,Y,X,G)$. This section proves
that the inflation $B$ is not a functional of their joint law, however many questions are
scored, and isolates exactly what is missing, which dictates the form of the remedies in
\Cref{sec:estimation}.

The definition of $B$ in \eqref{eq:Bdef} compares the observable surface $m$ with the honest
surface $m_0$, so the content lies in which pairs $(m_0,L)$ are admissible.

\begin{definition}[Operational risk decomposition]
\label{def:struct}
A pair of functions $m_0:\mathcal{X}\times\R\to[0,1]$ and $L:\mathcal{X}\times\R\to\R$ is an
\emph{operational decomposition} of $m$ if
\begin{equation}
  m(x,g)=m_0(x,g)-L(x,g).
  \label{eq:struct}
\end{equation}
\end{definition}

\noindent Admissible pairs always exist ($(m,0)$ is one); content enters through a sign and
a support restriction.

\begin{assumption}[Nonnegative pre-cutoff leakage]
\label{ass:struct}
There exists an operational decomposition $(m_0,L)$ such that
$L(x,g)\ge0$ and $L(x,g)=0$ for every $g\ge0$.
\end{assumption}

\noindent The support restriction says the future cannot leak. For a \emph{frozen} model (no
retrieval tools, no post-cutoff fine-tuning; cutoff uncertainty is handled by sensitivity
analysis, \Cref{sec:setup}), an outcome resolving after $T$ cannot have been seen, while
before the cutoff, memorized information lowers the loss by $L\ge0$. Recency lives in $m_0$:
the honest surface may depend on $g$ in any way, and in particular the theory never assumes
the model knows more in aggregate about recent times. Under \Cref{ass:struct}, the inflation
is $B=\E[L(X,G)\mid G<0]$.\footnote{\rev{Spillover onto unmemorized questions violates
\Cref{ass:struct}. \Cref{app:m3} gives what the RD jump then measures.}}

\begin{restatable}[Sharp partial identification from passive scores]{theorem}{thmnonident}
\label{thm:nonident}
Under \Cref{ass:struct} alone, $B$ is not a functional of the law of $(P,Y,X,G)$: every
continuation $\tilde m_0:\mathcal{X}\times\R\to[0,1]$ with $m\le\tilde m_0\le 1$ on $\{g<0\}$ and
$\tilde m_0=m$ on $\{g\ge0\}$ defines an admissible pair $(\tilde m_0,\ \tilde L=\tilde m_0-m)$
consistent with the same passive law, with corresponding inflation
$\tilde B=\E[\tilde m_0-m\mid G<0]$. The sharp identified set is
\[
  \big[\,0,\ \ \E[\,1-m\mid G<0\,]\,\big],
\]
and every value in this interval is attainable.
\end{restatable}

\noindent Post-cutoff, $m_0=m$ is observed. Pre-cutoff, the
data reveal only the difference $m_0-L=m$, and the one-parameter family
$\tilde m_{0,t}=m+t(1-m)$, $t\in[0,1]$, sweeps the entire interval
\rev{(\Cref{app:proofs-nonident})}.

The content lies in the model rather than the argument. \Cref{ass:struct} restricts almost
nothing---recency is entirely free, and leakage carries only a sign and a support
restriction---so the impossibility cannot be blamed on a restrictive model: once recency and
leakage are indexed by the same cutoff, it follows. Nothing is Brier-specific: for any loss
bounded by $M$, the upper endpoint becomes $\E[M-m\mid G<0]$.

The practical consequence corrects common practice: \emph{a flat pre/post performance profile
is not evidence of a clean backtest}, because an honest model with the same recency profile
produces identical scores. \Cref{fig:m1forest} is the converse error on real models: recency
alone produced the failing pattern of four flagships that cannot have memorized any scored
outcome.

The theorem also specifies its own remedy. The identified set is exactly the freedom of the
honest surface on the leaked side, so identification must import information that pins $m_0$
there: a clean control model, a continuity restriction at a known cutoff, or an intervention
on the question itself. Those are the three routes of \Cref{sec:estimation}. First, we
characterize what $L$ is, which tells an auditor where to look.
\section{Where Leakage Concentrates}
\label{sec:concentration}
\label{sec:nonident:perq}

Leakage is not a uniform rate across questions. It requires both that an outcome was
\emph{worth} knowing and that the model \emph{absorbed and used} it. A minimal working model
of how leaked information moves a forecast makes this precise and gives the inflation a
closed form.

\begin{assumption}[Convex-pull leakage]
\label{ass:pull}
Under the fixed deterministic prediction protocol, a leakage-eligible question $q$ ($g<0$) satisfies
\[
P(q)=(1-w(q))\,P_{\mathrm{hon}}(q)+w(q)\,Y(q)
\]
for an \emph{extraction weight}
$w(q)\in[0,1]$, where $P_{\mathrm{hon}}$ is the honest (legitimate-information) forecast.
\end{assumption}

\noindent \Cref{ass:pull} is a modeling assumption, not a consequence of the leakage problem.
Leaked information enters as a convex blend of the honest forecast and the recalled outcome:
exact for answer memorization ($w{=}1$, where $P\to Y$), first-order for graded evidence
leakage. Its substantive restriction is directionality: any forecast between
$P_{\mathrm{hon}}$ and the outcome is a convex pull for some $w\in[0,1]$, and what is
excluded is leaked information that \emph{misleads} ($w<0$), the same direction
\Cref{ass:struct} fixes in aggregate. We test the model's consequences against ground truth
in \Cref{sec:validation}. \rev{Misleading leakage does not invalidate the estimators of \Cref{sec:estimation}. They use only
the support restriction, so they identify \emph{net} leakage. Without spillover, a significant positive estimate is a
lower bound on helpful leakage. A null can miss effects below its power floor, and
cancellation. A significant negative estimate means a failed control or net misleading leakage (\Cref{tab:m5}).}

Writing the honest Brier $b_0(q)=(P_{\mathrm{hon}}(q)-Y(q))^2$, the leakage has a closed form.

\begin{restatable}[Per-question and aggregate leakage]{proposition}{propperq}
\label{prop:perq}
Under \Cref{ass:pull}, the per-question leakage is $b_0(q)\,w(q)\,(2-w(q))$; the pair
$m_0(x,g)=\E[b_0\mid X{=}x,G{=}g]$ and $L(x,g)=\E[b_0\,w(2-w)\mid X{=}x,G{=}g]$ is an
operational decomposition satisfying \Cref{ass:struct}, and hence
\begin{equation}
  B\;=\;\E\big[\,b_0\,w(2-w)\,\big|\,G<0\big].
  \label{eq:perq}
\end{equation}
\end{restatable}

\noindent The reason is one line: the pull leaves residual error $(1-w)(P_{\mathrm{hon}}-Y)$,
so the leaked Brier is $(1-w)^2 b_0$ and the saving is $b_0\,w(2-w)$. Averaging the saving
given $(X,G)$ furnishes a canonical decomposition, so the working model automatically
satisfies \Cref{ass:struct}. Per question, inflation is \emph{stakes times extraction}:
$b_0$ measures how much there was to know, and $w(2-w)$ how much of it the model took.

\paragraph{The double benefit.}
The factor $w(2-w)=1-(1-w)^2$ is why partial leakage is disproportionately rewarded
(plotted in \Cref{fig:geometric}b, \Cref{app:twin}). Pulling the forecast toward the truth earns a benefit linear in
$w$; the residual cost is only quadratic in $w$. A model that extracts half the leaked
signal ($w=0.5$) already gains $75\%$ of the full-memorization inflation, and $w=0.3$ still
gains $51\%$. Leakage need not be ``memorize the answer'' to distort a backtest; a faint,
partial trace suffices.

\begin{remark}[The counterfactual view]
\label{rem:twin}
\Cref{app:twin} derives the same double benefit from a counterfactual comparison with a
clean retrained model, shows that \eqref{eq:Bdef} coincides with the counterfactual inflation
(\Cref{lem:estimand-alignment}), and that \emph{uncorrelated} change strictly deflates the
score (\Cref{cor:no-free}).
\end{remark}

\paragraph{Two kinds of leakage, two consequences.}
The extraction weight separates \emph{answer} leakage, where the model has memorized the
outcome itself ($w\to1$), from \emph{evidence} leakage, where it absorbed leaked-window
information correlated with $Y$ ($w\in(0,1)$); a finer interpretive factorization of $w$ is
given in \Cref{app:w-interp}. Two consequences of \eqref{eq:perq} organize the rest of the
paper. First, leakage \emph{concentrates}: it is large only where the outcome was surprising
(high $b_0$) and absorbed (high $w$), and vanishes on easy or obscure questions, so an
``average leakage rate'' is the wrong audit object. Second, detecting $w>0$
requires controlling difficulty, skill, recency, and overconfidence together---a
multi-signal estimator, not a pre/post score gap---which \Cref{sec:estimation} builds.
\section{Three References, Two Estimators, One Diagnostic}
\label{sec:estimation}

\Cref{thm:nonident} specifies what identification must import: information that pins the
honest surface on the leaked side. Three practical references supply it: a known training
cutoff or a clean control model point-identifies the inflation under explicit assumptions,
and paraphrase-query access detects evidence leakage without identifying it. The results are
identification statements for the conditional risk $m$ under a bounded loss; the matching
estimators and bootstrap intervals are in \Cref{sec:validation,app:delta-machinery}.

\subsection{Route 1, a known cutoff: regression discontinuity}
\label{sec:estimation:rd}

Training cutoffs are published in model cards as documented approximations (hence the
sensitivity treatment of \Cref{sec:setup}), and the route needs no queries beyond a dated
benchmark spanning the cutoff. One restriction turns the cutoff into a reference: honest
competence drifts continuously through time, while memorization switches off at the cutoff.
Any jump in the loss there is therefore leakage.

\begin{assumption}[Honest-risk continuity]
\label{ass:cont}
For each $x$, $m_0(x,\cdot)$ is continuous at $g=0$, and $L(x,0^-)$ exists.
\end{assumption}

\begin{restatable}[RD identifies boundary leakage]{theorem}{thmrd}
\label{thm:rd}
Under \Cref{ass:struct,ass:cont} and the support regularity of
\Cref{app:proofs-routes}, for any bounded loss, the boundary leakage is identified by the
upward jump in observed risk at the cutoff,
\[
  J(x)\;=\;\lim_{g\downarrow0}m(x,g)-\lim_{g\uparrow0}m(x,g)\;=\;L(x,0^-)\;\ge\;0 .
\]
\end{restatable}

\noindent \rev{The boundary jump needs the fewest assumptions, and it is largest under full memorization ($w=1$), where PRC (Route 3) is blind. The sign restriction gives only the conclusion $J\ge0$. Without it, $J$ is net boundary leakage. Quasi-resolved outcomes and ingestion lag decrease the
jump, so a small $\widehat J$ does not exclude larger leakage (\Cref{sec:wild}).} Recovering the \emph{global}
inflation needs unique clean-side extrapolation of the honest surface (\Cref{ass:smooth}),
under which $B$ is identified (\Cref{prop:rd-global} in \Cref{app:proof-rd-global}).

\subsection{Route 2, a clean control model: difference-in-differences}
\label{sec:estimation:did}

The second route uses a second model: families ship vintages with different cutoffs, so
clean controls are standard artifacts, \emph{selected} rather than retrained, and they need
to be clean only on the evaluation window. The control need not match the target's
capability, because level differences cancel. It must instead match the \emph{response} of
honest performance to question timing: if the control shares the target's recency profile,
its pre/post change measures recency alone, and differencing pins the target's honest
change.

\begin{assumption}[Transportable recency]
\label{ass:transport}
The control $M_0$ is clean on the evaluation support ($L^{M_0}\equiv0$); the observations of
both models cover the target's pre-cutoff question mix on both sides of the cutoff; and,
once standardized to that mix, honest performance changes by the same amount from pre to
post for target and control (formal statement in \Cref{app:proofs-routes}).
\end{assumption}

\begin{restatable}[DiD identification]{proposition}{propdid}
\label{prop:did}
Under \Cref{ass:struct,ass:transport}, standardize both models' risks to the target's
pre-cutoff question mix and let $\Delta^A$ denote model $A$'s resulting pre-to-post change.
Then the inflation \eqref{eq:Bdef} is identified: $B=\Delta^{M}-\Delta^{M_0}$.
\end{restatable}

\noindent The parallel-change condition is not fully testable, because the target's
pre-cutoff cell is contaminated by construction; its testable implications are in
\Cref{app:delta-machinery}. Its dominant violation is signed: a control too weak to express
recency \emph{overstates} $B$---the weak-control trap of \Cref{sec:deployment}---so a
positive must carry the matching checks. \rev{The opposite mismatch hides leakage. A trend test measures the monthly slope $s$ of the target's advantage where both models are clean. This slope shifts an estimate by $-s\,\Delta t$, with $\Delta t$ the gap in mean post- and pre-cutoff dates (\Cref{app:e9-robust}).} A boundary variant
requiring only a no-discontinuity control underwrites the deployment matrix
(\Cref{prop:did-boundary}). How to read wild estimates is deferred to \Cref{sec:wild}.

\subsection{Route 3, paraphrase probes: detection only}
\label{sec:estimation:prc}

The third route needs only black-box access---a batch of paraphrase calls plus leakage-free
calibration anchors---and is a diagnostic, not an estimator. It yields the
prediction-residual covariance $\PRC=\Cov(P,\,Y-P\mid G<0)$ of the paraphrase-consensus
forecast, positive under convex-pull leakage with partial extraction
(\Cref{prop:efficiency}). Three limits keep it secondary: it requires calibration, since LLM
overconfidence makes the raw covariance negative (\Cref{rem:calib}); it is blind at full
memorization ($w{=}1$), where the RD jump is maximal (\Cref{prop:connection}); and on
frontier data it failed its pre-specified calibration-stability gate (\Cref{app:prc-real}).
We use it to detect, never to estimate $B$ (\Cref{app:delta-machinery}).

\subsection{One reference suffices; the adjusted score}
\label{sec:estimation:minimal}
\label{sec:estimation:deliverable}

Either of Routes 1--2 point-identifies $B$ under its stated assumptions, and Route 3 detects
without identifying (\Cref{prop:minimal}); one suitable external reference therefore
suffices, though we do not claim the family is exhaustive. \Cref{tab:routes} states the
input, guarantee, and check for each. When global $B$ is point-identified, report a
leakage-adjusted score with a confidence interval; otherwise report the boundary or
detection estimand without converting it into a global correction.

\begin{table}[t]\centering
\caption{\textbf{Three identification routes: input, guarantee, and check.} R1--R2
point-identify $B$ (or boundary $J$) under stated assumptions; R3 is detection only. Each
assumption has a falsification check (\Cref{sec:validation,sec:deployment}).}
\label{tab:routes}
\footnotesize
\begin{tabular}{@{}p{2.3cm}p{2.6cm}p{3.2cm}p{3.0cm}@{}}
\toprule
\textbf{Route (input)} & \textbf{Identifies} & \textbf{Key assumption} & \textbf{Check} \\
\midrule
R1: known cutoff (RD)
  & boundary $J$; global $B$ with extrapolation
  & continuity (\Cref{ass:cont}); extrapolation (\Cref{ass:smooth})
  & placebo jump; bandwidth; date sensitivity \\[2pt]
R2: clean control (DiD)
  & global $B$; boundary $J$ (\Cref{prop:did-boundary})
  & matched clean control for $B$; shared jump for $J$
  & placebo cutoffs; control balance \\[2pt]
R3: paraphrases (PRC)
  & detection only (not $B$)
  & calibration (\Cref{rem:calib})
  & calibration on leakage-free anchors; pairing with R1 covers $w{=}1$ \\
\bottomrule
\end{tabular}
\end{table}

\begin{corollary}[Leakage-adjusted performance]
\label{cor:honest}
Whenever $B$ is point-identified (Route~1 or~2), the leakage-adjusted performance removes $B$
in the appropriate direction: for a risk (any bounded loss, lower is better),
$\Radj=\Rmeas+B$; for a higher-is-better score (e.g.\ accuracy or pass@1),
$\mathrm{Score}_{\mathrm{adj}}=\mathrm{Score}_{\mathrm{meas}}-B$, with $B$ in score units.
\end{corollary}

\noindent The route results use only $m=m_0-L$ and hold for any bounded loss, licensing the
pass@1 correction of \Cref{sec:exp:e3}. The law $L=b_0\,w(2-w)$ and $\PRC$ are
Brier-specific (\Cref{rem:metric}). The practitioner's recipe and decision flowchart
(\Cref{fig:decision}) are in \Cref{app:delta-machinery}.
\section{Validation with Planted Leakage}
\label{sec:validation}

The estimators of \Cref{sec:estimation} claim to measure leakage. Before deploying them on
frontier models, we test them where the answer is known. M2 re-analyzes a public
controlled-contamination study at pretraining scale. M3 plants temporal leakage in twin
models on real forecasting questions and asks the estimators to find it. A claim-to-evidence
contract table mapping each theoretical claim to the experiment that carries it---across this
section, the deployment of \Cref{sec:deployment}, and the M1 check of
\Cref{sec:intro}---is \Cref{tab:contract}.

\paragraph{Data: the forecasting panel.}
All forecasting experiments use one panel built from the public archive of ForecastBench
\citep{karger2025forecastbench}, the field-standard dynamic forecasting benchmark:
$1{,}646$ resolved binary market questions (Polymarket, Metaculus, Manifold, INFER) resolving
2024Q3 through 2026Q3, each carrying its resolution date, its realized outcome $Y$, and a
contemporaneous crowd forecast $c_0$ frozen before resolution. The score is the
crowd-anchored Brier reduction $s=(c_0-Y)^2-(P-Y)^2$; positive $s$ means the model beats the
crowd. Confidence intervals use a cluster bootstrap over source-by-month clusters; boundary
tests add date-permutation and placebo-date checks; configurations, prices, and the
pre-registration are in \Cref{app:experiments} (coverage: \Cref{app:coverage}).

\paragraph{A pretraining-scale anchor (M2).}
\label{sec:exp:e2}
The Hubble suite \citep{hubble2025} pretrains a \emph{perturbed} model with benchmark
documents inserted at known duplication counts $r\in\{0,1,4,16,64,256\}$ and a
\emph{standard} twin trained identically without them. Re-analyzing its published accuracies
(no new compute) anchors the causal end: benchmark documents inserted into pretraining
inflate measured scores. The placebo-centered contrast rises
strictly monotonically with dose, from $+0.011$ at $r=1$ to $+0.403^\ast$ at $r=256$, with
sixteen duplicates already worth about 21 accuracy points (\Cref{fig:m2} in \Cref{app:e2};
Spearman $\rho=1.0$, permutation $p=0.0083$). The rise is contamination, not composition:
the clean twin is nearly flat across dose bins, and a self-baseline check agrees
(\Cref{app:e2}). Hubble is question
answering with no time axis, so it cannot test the recency confound; that is what M3 adds
(full analysis: \Cref{app:e2,app:supp-hubble}).

\subsection{Twin models with known temporal leakage (M3)}
\label{sec:exp:m3}

No observational study can supply per-question ground truth for leakage, so we created it. We continue-train two LoRA twins of Qwen3.5-35B-A3B-Base on
the same 95-million-token corpus, which interleaves public filler text with injected
news-style documents about panel questions that resolved before a pseudo-cutoff $T^\ast$
(April 2026). In the \emph{treatment} twin the documents state realized outcomes; in the
\emph{control} twin the same documents appear with outcomes scrubbed, matched on source,
topic, and length, so both twins gain the same recency and only the treatment twin gains
leakage. Injection covers the $988$ questions on which the untuned base model is demonstrably
uncertain, with dose levels $r\in\{0,1,4,16,64\}$ randomized across questions, so the full
dose grid costs two training runs. Probabilities are read from paired ``Yes''/``No''
continuation log-probabilities under a fixed few-shot prompt; a ten-percent pilot passed
five pre-registered gates before the full runs, and the protocol and logged deviations are
in \Cref{app:m3}.

\paragraph{Manipulation checks.}
Three checks confirm the design does what it claims: the treatment twin memorizes its
documents (per-token log-probability exceeds the control twin's by $0.76$, CI $[0.72,0.81]$);
the movement is outcome-directed (at dose 64 the treatment twin's probability moves toward
the realized outcome by $+0.139$, CI $[+0.119,+0.159]$); and recency is genuinely present (on
injected-topic questions the scrubbed control beats the untuned base by $+0.036$ Brier, CI
$[+0.028,+0.044]$). The confound is in the data.

\paragraph{Inflation rises with dose and is null at dose zero.}
The twin contrast rises strictly monotonically in dose and is null at dose zero
(\Cref{fig:m3}, left). Contamination does not need to be heavy: a single document exposure
is worth about one point of crowd-anchored Brier, sixteen about four. The dose-zero null
shows the inflation is question-specific, not a generic gain from training, and the curve
reproduces M2's pretraining dose-response on real forecasting questions, now with per-item
ground truth.

\begin{figure}[t]\centering
\figorbox[0.415\linewidth]{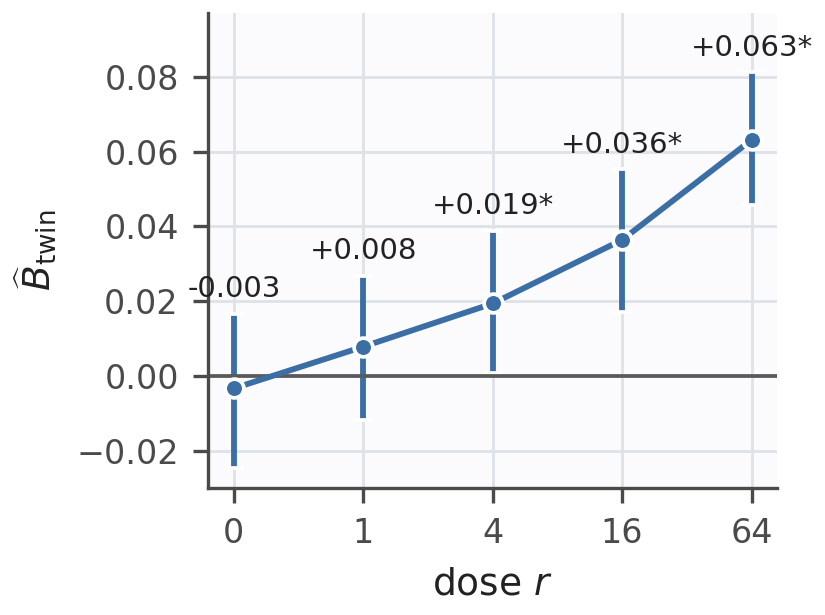}\hfill
\figorbox[0.555\linewidth]{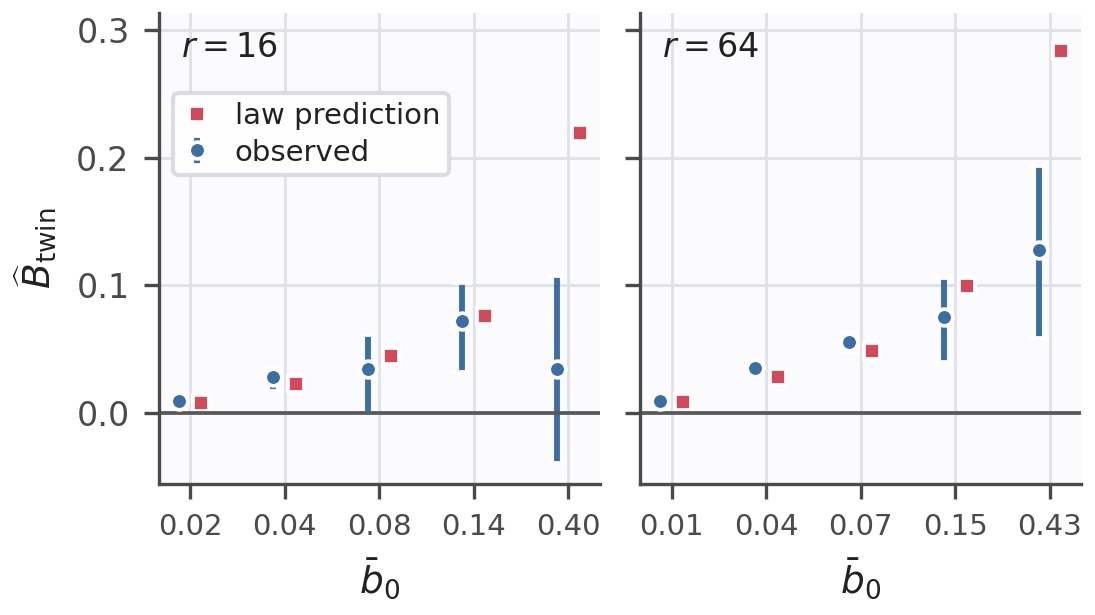}
\caption{\textbf{M3: injected leakage produces a strictly monotone dose-response (left) and
concentrates on hard questions as the law $b_0\,w(2-w)$ predicts (right).} Treatment and
control twins of one base model train on the same documents; only the treatment copy sees
realized outcomes. Both panels plot the twin contrast
$\widehat B_{\mathrm{twin}}=\bar\ell_{\mathrm{ctrl}}-\bar\ell_{\mathrm{treat}}$
($\ell=(P-Y)^2$; positive = leakage inflation); whiskers: bootstrap 95\% CIs; $^\ast$: CI
excludes zero. \emph{Left}: by dose $r$, the number of outcome-stating documents injected
per question (randomized across $988$ questions): null at $r=0$ ($-0.003$), rising strictly
monotonically to $+0.063^\ast$ at $r=64$. \emph{Right}: by difficulty, questions sorted into
quintiles of $b_0=(P_{\mathrm{ctrl}}-Y)^2$ (ticks: bin mean $\bar b_0$), at doses 16 and 64.
Red squares: the zero-free-parameter prediction $\bar b_0\,\hat w(2-\hat w)$, with $\hat w$
fitted from the treatment twin's probability movement alone, never from the plotted
contrasts. Inflation grows $13\times$ from easiest to hardest bin ($+0.010$ to
$+0.128$ at $r=64$); \rev{the largest miss is the hardest bin}, where extraction
weakens (\rev{solved per bin, $w$} drops from ${\approx}0.5$ to $0.16$) and the law acts as an upper bound.}
\label{fig:m3}
\label{fig:m3dose}
\label{fig:m3conc}
\end{figure}

\paragraph{The concentration law holds, and its \rev{largest} miss is interpretable.}
\Cref{prop:perq} predicts $L=b_0\,w(2-w)$. \Cref{fig:m3} (right) tests the law with zero
free parameters: $\hat w$ is fitted from the treatment twin's probability movement alone,
never from the plotted contrasts. The prediction falls inside the observed CI in four of
five difficulty quintiles at dose 16 and three of five at dose 64, and the globally fitted
extraction weight rises with dose ($\hat w=0.24$, $0.23$, $0.33$, $0.42$ at $r=1,4,16,64$);
tercile and ex-ante binnings agree (\Cref{app:m3}). \rev{The validation covers one base model and LoRA continued training. The fitted $\hat w$
describe this setup, and the paper uses them only in M3.} The \rev{largest} miss is the hardest quintile, where extraction weakens and the
homogeneous-$w$ law overshoots, so there it acts as an upper bound---the conservative
direction for score correction. Contamination distorts a backtest most where the model is
weakest.

\paragraph{The boundary estimators \rev{detect} the injection; recency alone shows no jump.}
An analyst unaware of the injection would measure the boundary statistics of
\Cref{sec:estimation}: on the treatment twin the naive pre/post gap at $T^\ast$ is $+0.053$
and the \rev{jump within $\pm90$ days} is $+0.052$; on the control twin the same statistics
are $+0.001$ and $+0.012$, and a placebo boundary 120 days earlier is null
($-0.014$). Detection and localization both work: injected leakage produces a starred jump
at the right date, and recency alone produces none---the control twin absorbed the same 95
million tokens and shows no inflation signature, the real-model face of the
no-free-inflation asymmetry (\Cref{cor:no-free}).

\paragraph{PRC detects the evidence, once calibrated.}
The prediction-residual covariance of \Cref{prop:efficiency} had not previously been tested
against known contamination. Raw PRC fails: it is negative in every cell of the twin-by-dose
grid, the miscalibration failure \Cref{rem:calib} warns about. The twin-differenced,
temperature-calibrated statistic behaves as predicted: $+0.0069^\ast$ on injected questions,
null at dose zero, monotone in dose (\Cref{fig:m3prc} in \Cref{app:m3}). PRC fires where
evidence leakage was injected and nowhere else, and only with differencing and calibration.

\subsection{How to read a wild estimate}
\label{sec:wild}

The twins license two reading rules for the deployment results of \Cref{sec:deployment}.

\paragraph{A boundary estimate is a footprint, not a per-question inflation.}
On the injected population, per-item ground-truth inflation averages $+0.024$, roughly half
the boundary estimate of $+0.052$; \rev{the rest is spillover and the control twin's own jump.}
The injected documents state outcomes that are $83\%$ NO, so the treatment twin acquires a
base-rate lean that is nearly free on pre-cutoff questions but depresses post-cutoff scores,
and a pre/post estimator cannot tell inflated pre from depressed post (the law analyses are
immune: dose-zero differencing removes the lean). A wild RD or DiD estimate is therefore the
total contamination footprint at the boundary\rev{: the memorization gain plus the spillover change at the cutoff (\Cref{app:m3})}. 

\rev{Spillover can move the jump in either direction. In M3, spillover costs more after the cutoff, so it inflates the jump. Ingestion lag and quasi-resolved outcomes attenuate the jump. A starred jump survives attenuation, but a small jump does not exclude larger leakage (\Cref{app:m3}). Global $B$ needs \Cref{ass:smooth}, which fails on our code data
(\Cref{app:supp-e3global}), so we state every wild claim below at the boundary.}

\paragraph{A null bounds; it does not certify.}
Every null in \Cref{sec:deployment} comes with a power floor: \rev{the normal-approximation
minimum detectable effect ($2.8\times$ the clustered SE). On the forecasting matrix, the floors
are} $0.05$ to $0.11$ anchored-Brier units per cell. A null therefore \rev{has power against}
code-domain-sized leakage (about $0.09$) for the best-powered cells and is silent on smaller
leakage; no cleanliness certificate is issued below the floor.

\section{Deployment on Frontier Models}
\label{sec:deployment}
\label{sec:experiments}
\label{sec:exp:e3}

We now run the audit where nobody knows the answer. Results are ordered as an auditor should
read them: cross-model nulls, disciplined nulls under the matched control, one detection, and
a positive control on documented contamination.

\subsection{The forecasting matrix: no flagship flagged at its own cutoff}
\label{sec:exp:m4f}

We query four frontier targets with in-window documented cutoffs (rows of \Cref{tab:m4f})
against the near-clean controls MiniMax-M3 and Claude-Opus-4.7 (their January 2026 cutoffs
leave no leakage discontinuity \rev{at the tested boundaries}), computing the boundary jump at
every assumed boundary. A genuine signal must jump at the
model's own documented cutoff and nowhere else. No own-cutoff cell is starred, and the joint
diagonal-versus-off-diagonal permutation test is null ($p=0.53$). Under the power floors of
\Cref{sec:wild}, the matrix \rev{has power against} code-domain-sized effects for the best-powered cells
while remaining uninformative about smaller leakage. An exploratory Gemini-3.1-Pro row is in
\Cref{app:m4}. Run honestly, the cross-model route substantiates no leakage claim for these
flagships.

\begin{table}[t]\centering
\caption{\textbf{M4-F: no own-cutoff jump among four flagships; the within-family arm flags
GPT-5.5.} Cell: $\widehat J_m(t)=\rev{(\bar s_m-\bar s_c)_{\mathrm{pre}}-(\bar s_m-\bar
s_c)_{\mathrm{post}}}$ over $\pm90$ days of assumed boundary $t$ ($s$ as in
\Cref{fig:m1forest}; $c$: top block = MiniMax-M3/Claude-Opus-4.7; bottom = GPT-5.5 family
archived real-time forecasts). Shaded: own documented cutoff. $^\ast$: CI excludes zero.
Dashes: \rev{the archive does not cover these boundaries.}
Result: no shaded top-block cell starred (permutation $p=0.53$); bottom block starred at the
shaded cell (pre-specified unbanded / horizon-banded, \Cref{sec:exp:m4anchor}).}
\label{tab:m4f}
\footnotesize
\begin{tabular}{lcccc}
\toprule
 & \multicolumn{4}{c}{assumed boundary $t$} \\
\cmidrule(lr){2-5}
model (own cutoff) & Mar'25 & Apr'25 & Aug'25 & Dec'25 \\
\midrule
\multicolumn{5}{l}{\emph{cross-model: flagship vs.\ clean controls}} \\
\addlinespace[1pt]
DeepSeek-V3.1 (Mar'25) & \cellcolor{gray!18}$-0.009$ & $+0.015$ & \rev{$-0.065^\ast$} & $+0.018$ \\
Kimi-K2.6 (Apr'25)     & $+0.034$ & \cellcolor{gray!18}$+0.022$ & $-0.020$ & $-0.013$ \\
GPT-5.4 (Aug'25)       & $-0.012$ & $+0.022$ & \cellcolor{gray!18}$-0.051$ & $+0.030$ \\
GPT-5.5 (Dec'25)       & $-0.008$ & $+0.010$ & $-0.011$ & \cellcolor{gray!18}$+0.034$ \\
\addlinespace[2pt]
\multicolumn{5}{l}{\emph{within-family: GPT-5.5 vs.\ archived real-time forecasts}} \\
\addlinespace[1pt]
GPT-5.5 (Dec'25) & --- & --- & --- & \cellcolor{gray!18}\rev{$+0.078^\ast$ / $+0.108^\ast$} \\
\bottomrule
\end{tabular}
\end{table}

\subsection{Disciplined nulls under the matched control}
\label{sec:exp:m5}

The primary target is Qwen3.5-35B-A3B (year-level 2026 cutoff), evaluated at a July 2025
boundary on panel questions resolving 2025 through mid-2026. Because the cutoff is
year-level, the boundary lies inside the plausible training window, so both sides are
leakage-eligible: the design detects inflation that \emph{differs} across the boundary;
uniform inflation would evade it at every sweep placement. The control is selected
\emph{before} reading any DiD result by a pre-registered profile-matching protocol on
post-boundary-clean questions, which picks GPT-5 over Gemini-3.1-Pro (profile distance
$0.0051$ vs.\ $0.0088$; \Cref{app:m5}). Secondary targets are the four other flagships of
\Cref{fig:m5dissolve}; GPT-3.5-Turbo is a deliberately mismatched weak control.

Naive gaps again \rev{look like leakage, but} after the paired,
profile-matched DiD, no adjusted estimate is significantly positive (\Cref{fig:m5dissolve}).
\rev{Substituting GPT-3.5-Turbo raises the primary estimate to $+0.055$, as a weak control should. GPT-5 fails the trend test against Qwen3.5. If the sign holds at the boundary, the null stays conservative (\Cref{app:e9-robust}).}
\rev{The power floor of a null is its minimum detectable effect, $2.8$ times its clustered
SE. For these nulls, the floors are between $0.046$ and $0.098$.} Full estimates
are in \Cref{app:m5}.

\begin{figure}[t]\centering
\figorbox[0.59\linewidth]{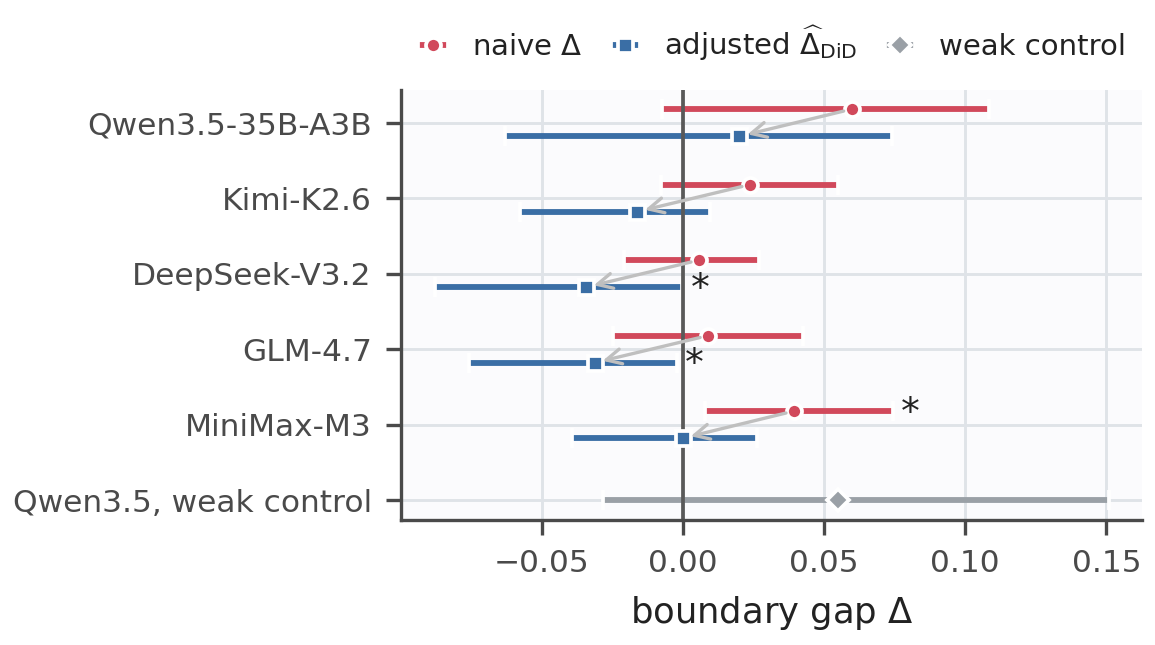}
\caption{\textbf{M5: naive \rev{gaps do not survive} the matched adjustment.} Red: naive
gap $\Delta=\bar s_{\mathrm{pre}}-\bar s_{\mathrm{post}}$. Blue: adjusted
$\widehat\Delta_{\mathrm{DiD}}=\Delta_{\mathrm{target}}-\Delta_{\mathrm{control}}$ against
pre-registered matched control GPT-5. Grey diamond: primary target vs.\ mismatched
GPT-3.5-Turbo. Whiskers: 95\% CIs; $^\ast$: CI excludes zero. Result: no adjusted estimate
significantly positive; mismatched control \rev{raises it} to $+0.055$.}
\label{fig:m5dissolve}
\end{figure}

\subsection{One detection: a boundary signature for GPT-5.5}
\label{sec:exp:m4anchor}

ForecastBench archives real-time forecasts, clean by construction, enabling a within-family
design: comparing retrospective GPT-5.5 scores with archived scores of GPT-5-Mini and GPT-5.1
cancels the crowd anchor. \rev{The jump at GPT-5.5's December 2025 cutoff is $+0.078$ (CI $[+0.057,+0.124]$,
$n=60/339$ at $\pm90$ days). In a 10--75-day horizon band, a logged deviation, the jump is $+0.108$ (CI $[+0.057,+0.192]$).} 

\rev{The estimate rests on three assumptions. First, the archived forecasts are free of leakage
by construction. Second, without leakage, the difference between GPT-5.5 and the archive has no
jump or trend across the window. Its slope after the cutoff is $-0.0015$ per month (adjusted jump $+0.074$). Third, the documented cutoff is correct to
about one month. Shifted boundaries agree with a single December 2025 cutoff, except at $-60$ days (18 questions before that boundary). Jumps at dates without a cutoff reach $0.048$ in size, below the primary CI (\Cref{app:e9-robust}).}

\rev{A DeepSeek-V3.1 placebo at the same boundary is null but imprecise, and the smaller cross-model cell
($+0.034$) is attenuated (\Cref{app:m4}).} We scope the finding as \Cref{sec:wild} instructs: a boundary leakage
signature at GPT-5.5's documented cutoff---a contamination footprint, not an accusation.
\rev{GPT-5.4 shows no significant jump at $\pm90$ days ($+0.037$, CI $[-0.015,+0.158]$). Its unbanded jump at $\pm120$ days ($+0.048$) disappears when the horizons are balanced, which is consistent with horizon imbalance.}

\subsection{Positive control: documented contamination on dated code}
\label{sec:exp:m4c}

LiveCodeBench \citep{livecodebench2025} dates $1{,}055$ problems by contest release and
publishes per-problem pass@1 (\Cref{rem:metric}), so the same matrix runs at near-zero cost
on a benchmark era where contamination is documented. GPT-4o and Claude-3.5-Sonnet are
starred exactly at their own cutoffs (date-permutation $p=0.023$, $p=0.017$), with all
off-diagonal cells null (\Cref{tab:m4c}). A self-generated extension to 2024-cutoff targets
is in \Cref{app:e3}; coverage across M1--M5 spans 2023--2026 vintages of six-plus vendors
(\Cref{app:coverage}).

\begin{table}[t]\centering
\caption{\textbf{M4-C: positive control---the same test fires at both models' own cutoffs
and nowhere else.} Cell: $\widehat J_m(t)$ of \Cref{tab:m4f} on LiveCodeBench pass@1 within
$\pm160$ days of $t$; baseline = 2025-cutoff models. Shaded: own cutoff. $^\ast$: CI
excludes zero ($p=0.023$, $p=0.017$). The large unstarred placebo (Claude at Dec'24,
$+0.128$) is noise: its CI $[-0.23,+0.33]$ is five times wider than the others because the
$\pm160$-day window overruns the panel's end.}
\label{tab:m4c}
\footnotesize
\begin{tabular}{lcccccc}
\toprule
 & \multicolumn{6}{c}{assumed boundary $t$} \\
\cmidrule(lr){2-7}
model (own cutoff) & Oct'23 & Apr'24 & May'24 & Jul'24 & Sep'24 & Dec'24 \\
\midrule
GPT-4o-0806 (Oct'23) & \cellcolor{gray!18}$+0.090^\ast$ & $-0.050$ & $-0.048$ & $-0.056$ & $-0.006$ & $-0.029$ \\
Claude-3.5-Sonnet (Apr'24) & $+0.017$ & \cellcolor{gray!18}$+0.095^\ast$ & $+0.069$ & $-0.014$ & $-0.019$ & $+0.128$ \\
\bottomrule
\end{tabular}
\end{table}
\section{Related Work}
\label{sec:related}

\paragraph{Detecting contamination.}
A large literature asks \emph{whether} evaluation data was seen in training: $n$-gram and
longest-substring screens \citep{brown2020gpt3,singh2024contamination}, membership-inference
and likelihood tests \citep{shokri2017membership,carlini2021extracting,shi2024detecting},
exchangeability tests \citep{oren2024proving}, and guided or perturbation-based black-box
probes \citep{golchin2024time,codec2026}; surveys catalogue detectors and their failure
modes \citep{sainz2023nlp,xu2024benchmark,ravaut2024survey}. All are membership-oriented,
and membership inference is near chance at pretraining scale, its apparent successes an
artifact of temporal shift \citep{duan2024membership}. Temporal leakage largely escapes
these tests, acting through parametric memory along an indirect causal channel
(\Cref{fig:dag}), and detection does not measure the inflation.

\paragraph{Quantifying contamination's score effect.}
Controlled training studies insert benchmark data deliberately
\citep{magar2022data,jiang2024investigating,hubble2025}; Hubble grounds our M2, M3 adapts
its dose design to the temporal setting, and concurrent work fits dose--response curves for
generative evaluations \citep{schaeffer2026quantifying}. On deployed models, clean-reference estimators
\citep{singh2024contamination,haimes2024benchmark,zhang2024gsm1k,dekoninck2024constat} are
instances of our DiD route, overlap audits price contamination on code benchmarks
\citep{riddell2024quantifying}, and rephrased samples evade $n$-gram screens
\citep{yang2023rethinking}. Live and date-stamped benchmarks avoid contamination by
construction \citep{white2024livebench,livecodebench2025,wu2025antileak}. All define
``clean'' by membership in a static test set; a backtest defines it by a \emph{date}, so
every pre-cutoff question may be exposed and the reference must come from elsewhere
(\Cref{sec:estimation}).

\paragraph{Temporal leakage and lookahead bias.}
Cutoff-based natural experiments document the symptom
\citep{roberts2024cutoff,li2024task}; without a recency model, the same gap is consistent
with legitimate recency (\Cref{thm:nonident}). \rev{\citet{zhang2026testoftime} ask if post-cutoff decay shows contamination, and we ask how much
leakage inflates a backtest score. They show that LLM-written or rewritten questions show no decay despite exposure. In our model, a rewrite lowers $w$ and thus the inflation.
\Cref{thm:nonident} adds that a flat profile alone cannot certify a clean score.} In financial and forecasting backtests,
lookahead bias has been separated by anonymization \citep{glasserman2024lookahead}, tested
on unpredictable events \citep{sarkar2024lookahead}, argued to undermine LLM forecaster
evaluation \citep{paleka2025pitfalls}, and measured or mitigated on the defining benchmarks
\citep{zou2022autocast,halawi2024forecasting,karger2025forecastbench,exante2026,lookaheadbench2026,gao2025lookahead,shapleydclr2025,fincad2026}.
Retraining avoids leakage by construction
\citep{lazaridou2021mind,dhingra2022time,drinkall2024time,he2025chronologically,yan2026datedgpt}
but cannot evaluate a given frontier model. Closest to us,
\citet{lopezlira2025memorization} prove counterfactual forecasting ability is non-identified
once a model trains on realized outcomes; we add a distinct recency channel, derive where
leakage concentrates, and show that one external reference restores measurement. Classical
identification tools are catalogued in \Cref{app:related-tools}.
\section{Limitations and Conclusion}
\label{sec:discussion}
\label{sec:conclusion}

Our guarantees rest on explicit assumptions, tested where possible. The per-question law excludes misleading leakage ($w<0$, \Cref{ass:pull})\rev{. RD and DiD do not need this restriction and measure} \emph{net} inflation; the twins impose rather than test this restriction,
and they show that homogeneous-$w$ overstates
inflation on the hardest questions (\Cref{sec:exp:m3}). DiD requires a clean,
recency-matched control (\Cref{ass:transport}): a weak control overstates $B$, a failure we
demonstrate with GPT-3.5-Turbo and guard against with pre-registered profile matching
(necessary, not sufficient) \rev{and, in this revision, a trend test}. RD rests on continuity
(\Cref{ass:cont}); converting the boundary jump into global $B$ adds \Cref{ass:smooth},
which our code data do not support. A wild boundary jump is
best read as a total contamination footprint, and quasi-resolved outcomes and ingestion lag
attenuate it toward zero (\Cref{sec:wild}). PRC passed ground truth
but failed its real-data calibration-stability gate (\Cref{app:prc-real}), so it
remains a detection tool. Empirically, the twins are LoRA continued-training runs \rev{of one
base model} (M2 anchors pretraining), documented cutoffs are month-resolution, and the M5
nulls \rev{have power against} boundary-differential inflation above \rev{their floors ($0.046$ to
$0.098$)}, not \rev{against} uniform inflation over the training window (\Cref{sec:exp:m5}). Natural next steps are
planted leakage at pretraining scale, a capability-matched dated model pair that
point-identifies global $B$, content perturbations that identify $w$, and transferable PRC
calibration.

Backtest scores mix skill, recency, and leakage, and we proved that no passive backtest can
separate the three: a flat pre/post profile is not evidence of a clean backtest. The
inflation is nevertheless systematic: under a minimal model of partial recall it obeys
$B=\E[b_0\,w(2-w)\mid G<0]$, concentrating where the crowd was surprised and training
coverage was dense. One external
reference restores measurement: a known cutoff identifies leakage at the boundary, and a
matched clean control identifies it globally and yields a leakage-adjusted score. \rev{Against ground truth, the estimators detected and localized planted leakage and tracked its
dose and per-question profile. The boundary statistic ($+0.052$) is
about twice the per-item inflation ($+0.024$), because of spillover and the clean twin's own change. In deployment, the estimators detected one cutoff-localized signature and the documented code contamination, but no significant leakage in five models with positive naive gaps.} Backtests need not be discarded; they need one
defensible reference.

\subsubsection*{Broader Impact and Reproducibility Statements}
This paper reports a contamination finding about a named commercial model (GPT-5.5), scoped
as a boundary leakage signature, not an accusation of intent or a global capability claim
(\Cref{sec:exp:m4anchor}); it raises no concerns beyond those of the underlying public
benchmarks. Code, data, configurations, pre-registration, and results are at
\url{https://github.com/ZeyuZhang1901/Temporal-Leakage-Backtesting};
\Cref{app:provenance} maps each result to its script and output.

\bibliographystyle{tmlr}
\bibliography{references}

\newpage
\appendix
\section*{Appendix roadmap}

Nothing in this appendix is optional filler: it holds the twin formalism demoted from the
main text, every proof, the practitioner recipe, full experimental configurations, demoted
main-text detail, and the provenance map. Read by role:

\begin{itemize}
\item \textbf{A.} \Cref{app:notation}: notation and standing conventions.
\item \textbf{B.} \Cref{app:proofs}: counterfactual twin theory (\Cref{app:twin}) plus all
proofs for \Cref{sec:impossibility,sec:concentration,sec:estimation}, in main-text order,
with technical complements (including the interpretive factorization of $w$).
\item \textbf{C.} \Cref{app:delta-machinery}: practitioner recipe, control validation,
finite-$K$ PRC machinery, paraphrase protocol, and the classical identification-tool map
(\Cref{app:related-tools}).
\item \textbf{D.} \Cref{app:experiments}: full configurations and complete results for
M1--M5, including material demoted from the main text (M1 matched-window check, Gemini
exploratory row, clean-anchor banding and design disagreement, code-arm extension).
\item \textbf{E.} \Cref{app:supp}: synthetic validation (E1) and per-experiment robustness.
\item \textbf{F.} \Cref{app:provenance}: every reported number mapped to its script and
output file.
\end{itemize}

A reader checking the theory needs A--B; a reader assessing or replicating the experiments
needs D--F; a practitioner deploying the routes needs C.

\section{Notation and standing conventions}
\label{app:notation}

\begin{center}
\small
\begin{tabular}{@{}p{3.6cm}p{9.6cm}@{}}
\toprule
\textbf{Symbol} & \textbf{Meaning} \\
\midrule
$M$, $T$ & backtested model and its training cutoff (\Cref{sec:setup}) \\
$Y$, $\tau$ & realized binary outcome; its resolution time (\Cref{sec:setup}) \\
$G=\tau-T$ & running variable; $g<0$ leakage-eligible, $g>0$ clean (\Cref{sec:setup}) \\
$X$ & leakage-free difficulty covariate, e.g.\ crowd forecast $c_0$ (\Cref{sec:setup}) \\
$P$ & the model's forecast under a fixed protocol (\Cref{sec:setup}) \\
$m(x,g)$ & observable conditional mean loss, \eqref{eq:m} (\Cref{sec:setup}) \\
$\Dclean$, $\Dbridge$, $\Dtrain$ & pre-$t_0$ corpus; bridge data from the leaked window $(t_0,T]$; training data (\Cref{sec:setup,app:twin}) \\
$\ycorrupt$, $\yclean$ & corrupt-trained and clean-trained predictions (\Cref{sec:setup,app:twin}) \\
$\ybayes$, $\ycleanbayes$ & Bayesian predictors with/without the bridge data (\Cref{app:twin}) \\
$S=\ybayes-\ycleanbayes$ & leakage signal (\Cref{app:twin}) \\
$\delta=\ycorrupt-\yclean$ & perturbation of the deployed model (\Cref{app:twin}) \\
$\lambda$, $\mudelta$, $\tnu$ & loading of $\delta$ on $S$; conditional shift; orthogonal residual (\Cref{app:twin}) \\
$B_{\mathrm{twin}}$ & counterfactual inflation, \eqref{eq:bias-def} (\Cref{app:twin}) \\
$m_0(x,g)$, $L(x,g)$ & latent leakage-free risk surface; leakage contribution (\Cref{sec:setup,sec:impossibility}) \\
$B=\E[L\mid G<0]$ & operational inflation, \eqref{eq:Bdef} (\Cref{sec:setup}) \\
$P_{\mathrm{hon}}$, $w$ & honest forecast; per-question extraction weight (\Cref{sec:concentration}) \\
$b_0=(P_{\mathrm{hon}}-Y)^2$ & honest Brier, the ``stakes'' (\Cref{sec:concentration}) \\
$J(x)$ & leakage jump at the cutoff boundary (\Cref{sec:estimation}) \\
$\PRC$ & prediction-residual covariance, a detection statistic (\Cref{sec:estimation}) \\
$\Radj=\Rmeas+\Bhat$ & leakage-adjusted score (\Cref{sec:estimation}) \\
\bottomrule
\end{tabular}
\end{center}

\paragraph{Standing regularity.}
All predictions and outcomes have finite second moments. Lowercase $x,g$ denote realizations
of $X,G$. Conditional expectations, variances, and
projections are defined almost surely on their relevant support. Whenever
$\Var(S\mid X,\Dtrain,G<0)=0$, the projection coefficient $\lambda=\Cov(\delta,S)/\Var(S)$ is a
$0/0$ expression; we adopt the convention $\lambda=0$ there, since a signal with no conditional
variation carries no leakage to load on. Common support and density conditions are invoked explicitly for the
identification route that requires them.

\section{The twin theory, proofs, and technical complements}
\label{app:proofs}

The counterfactual twin theory summarized in \Cref{rem:twin} comes first (\Cref{app:twin});
the proofs of all formal results follow, grouped by the section in which each result appears.
Every entry states the context of the result, restates it, and
then gives the proof. Interspersed with the proofs are short technical complements that
belong with the mathematics rather than the experiments: auxiliary lemmas stated only here, a
bound on the systematic shift, and an interpretive parameterization of the extraction weight
(\Cref{app:w-interp}).

\subsection{The counterfactual twin theory}
\label{app:twin}
\label{sec:decomposition}

The counterfactual frame compares the deployed model with a clean twin retrained without the
leaked data and asks what leaked information does to their score difference; \Cref{rem:twin}
in the main text summarizes the result. The argument moves from an ideal benchmark
(\Cref{sec:decomposition:bayes}) to general predictors (\Cref{sec:decomposition:general}) to
the reason the construction cannot be run in practice
(\Cref{sec:decomposition:obstruction}). With the two models of \Cref{fig:dag}, we write
$\delta=\ycorrupt-\yclean$ for their disagreement and $\mathcal H=\sigma(X,\Dtrain)$ for the
pre-cutoff information; $\ycorrupt$ and $\yclean$ denote the corrupt- and clean-trained
predictions, and for a probabilistic protocol $\ycorrupt=P$. All statements and conditional
moments in this subsection are under the conditional law of leakage-eligible questions
($G<0$). The target is the
\emph{counterfactual inflation}
\begin{equation}
  B_{\mathrm{twin}}
  \;=\;\E\!\left[(\yclean-Y)^2-(\ycorrupt-Y)^2\mid G<0\right],
  \label{eq:bias-def}
\end{equation}
positive when the deployed model looks more accurate than its counterfactual clean version.
Every result below is measured against the \emph{leakage signal}
\begin{equation}
  S\;=\;\underbrace{\E[Y\mid X,\Dtrain,\Dbridge]}_{\ybayes}
  \;-\;\underbrace{\E[Y\mid X,\Dtrain]}_{\ycleanbayes},
  \label{eq:signal-def}
\end{equation}
the amount by which an ideal learner would revise its forecast upon seeing the leaked data:
large exactly when the bridge data would move an ideal forecaster's belief about $Y$, zero when
they carry no information about the question. Because the conditional expectation is the best
possible use of the bridge data, $S$ summarizes all any model could extract from $\Dbridge$
about $Y$---the natural yardstick for leakage.

\subsubsection{Bayes-optimal extraction: the value of leaked information}
\label{sec:decomposition:bayes}

Suppose first that both predictors are Bayes-optimal, $\yclean=\ycleanbayes$ and
$\ycorrupt=\ybayes$: the deployed model extracts everything the leaked data contain, and its
disagreement equals the signal ($\delta=S$).

\begin{restatable}[Bayesian value of bridge information]{theorem}{thmbayesian}
\label{thm:bayesian}
For the Bayes-optimal pair,
\[
B_{\mathrm{twin,Bayes}}
=\E\!\left[\Var(S\mid\mathcal H)\mid G<0\right]\ge0.
\]
\end{restatable}

\noindent The inflation is \emph{always} non-negative for the full extractor, and its magnitude
equals how much the post-cutoff world deviates from the pre-cutoff world, as seen through the
outcome: volatile, surprising periods (large $\Var(S\mid\mathcal H)$) admit large inflation,
stable ones almost none. This is the value of bridge information for Bayes-optimal prediction,
not a universal upper bound over arbitrary algorithm pairs.

\subsubsection{General predictors: decomposition along the signal}
\label{sec:decomposition:general}

The model under backtest is not Bayes-optimal: it extracts only a fraction of $S$, and not
necessarily along $S$. To see how much of the ideal value it captures---and at what cost---we
take the conditional $L^2$ projection of its perturbation $\delta$ onto the signal,
\begin{equation}
  \delta \;=\; \lambda\,S \;+\; \mudelta \;+\; \tnu,
  \qquad
  \lambda=\frac{\Cov(\delta,S\mid\mathcal H)}{\Var(S\mid\mathcal H)},
  \qquad
  \mudelta=\E[\delta\mid\mathcal H],
  \label{eq:delta-decomp}
\end{equation}
with $\lambda=0$ on strata where the denominator vanishes; the residual $\tnu$ has conditional
mean zero and is conditionally orthogonal to $S$. Two benchmark deviations complete the
bookkeeping: $\rclean=\yclean-\ycleanbayes$, the clean model's deviation from its own Bayes
benchmark, and $\gperp=Y-\ybayes$, the Bayesian residual---the part of the outcome that no
training data can predict; the clean gap then decomposes as
$\gamma=Y-\yclean=S+\gperp-\rclean$. \Cref{fig:geometric} depicts the geometry of the two
decompositions and why only the $S$-aligned component of $\delta$ can pay off. Substituting
\eqref{eq:delta-decomp} into \eqref{eq:bias-def} gives the central decomposition.

\begin{figure}[t]
\centering
\begin{tikzpicture}[scale=1.0, >=Stealth, baseline=(current bounding box.center)]
  \node[font=\small\bfseries] at (-0.15,2.45) {(a)};
  \fill[black] (0,0) circle (1.5pt);
  \draw[->, very thick, blue!70!black, line width=1.5pt] (0,0) -- (4.2,0)
    node[pos=0.83, above, font=\small, yshift=1pt] {$S$};
  \draw[->, very thick, orange!80!black, line width=1.5pt] (0,0) -- (2.1,0)
    node[pos=0.45, above, font=\small, yshift=1pt] {$\lambda S$};
  \draw[->, very thick, red!60!black, line width=1.5pt] (4.2,0) -- (4.2,2.3)
    node[midway, right, font=\small, xshift=2pt] {$\gperp$};
  \draw[->, very thick, black!65!green, line width=1.5pt] (0,0) -- (4.2,2.3)
    node[pos=0.42, above left, font=\small, xshift=2pt, yshift=-2pt] {$\gamma$};
  \draw[gray!55, thick] (4.0, 0) -- (4.0, 0.2) -- (4.2, 0.2);
  \draw[->, dashed, thick, purple!55!black, line width=1.1pt] (2.1,0) -- (2.1,-0.65)
    node[midway, right, font=\small, xshift=2pt] {$\tnu$};
  \draw[->, very thick, purple!70!black, line width=1.5pt] (0,0) -- (2.1,-0.65)
    node[pos=0.42, below, font=\small, yshift=-2pt] {$\delta$};
  \draw[purple!40, thick] (1.92, 0) -- (1.92, -0.18) -- (2.1, -0.18);
\end{tikzpicture}
\hspace{1.1cm}
\begin{tikzpicture}[>=Stealth, baseline=(current bounding box.center)]
  \node[font=\small\bfseries] at (-0.55,2.45) {(b)};
  \draw[->, thick] (0,0) -- (3.5,0) node[below, font=\small, xshift=-4pt, yshift=-1pt] {$\lambda$};
  \draw[->, thick] (0,0) -- (0,2.5);
  \node[anchor=west, font=\scriptsize] at (0.02,2.42) {inflation $/\,\Var(S)$};
  \foreach \x/\l in {1.5/0.5, 3.0/1}
    \draw[thin] (\x,0.05) -- (\x,-0.05) node[below, font=\scriptsize] {$\l$};
  \foreach \y/\l in {0.825/0.75, 1.1/1}
    \draw[thin] (0.05,\y) -- (-0.05,\y) node[left, font=\scriptsize] {$\l$};
  \draw[gray!75, thick] (0,0) -- (3.0,2.2);
  \node[font=\scriptsize, gray!75!black, rotate=36] at (1.52,1.36) {benefit $2\lambda$};
  \draw[gray!75, thick, densely dashed] plot[domain=0:3.0, samples=40] (\x,{1.1*(\x/3)*(\x/3)});
  \node[font=\scriptsize, gray!75!black] at (2.42,0.4) {cost $\lambda^2$};
  \draw[blue!70!black, line width=1.5pt] plot[domain=0:3.0, samples=60] (\x,{1.1*(\x/3)*(2-\x/3)});
  \node[font=\scriptsize, blue!70!black, anchor=east] at (3.4,1.33) {$\lambda(2-\lambda)$};
  \draw[densely dashed, red!65!black] (1.5,0) -- (1.5,0.825) -- (0,0.825);
  \fill[red!65!black] (1.5,0.825) circle (1.8pt);
\end{tikzpicture}
\caption{\textbf{Why leakage inflates a backtest: the decomposition (a) and the double benefit
it implies (b).} \textbf{(a)}~What the clean model does not know about the outcome---the clean
gap $\gamma$ (green)---splits into the leakage signal $S$ (blue), everything the leaked data
reveal about $Y$, and the residual $\gperp$ (red), which no training data can predict. Training
on leaked data moves the model by $\delta$ (purple), of which only the extracted component
$\lambda S$ (orange) overlaps $\gamma$ and can raise the score; the noise $\tnu$ (dashed)
strictly deflates it (\Cref{cor:no-free}). \textbf{(b)}~Extracting a fraction $\lambda$ of the
signal buys a score benefit linear in $\lambda$ at a variance cost only quadratic in $\lambda$,
so the net inflation $\lambda(2-\lambda)\Var(S)$ rises steeply from zero: extracting half the
signal (dashes) already yields $75\%$ of the full-memorization inflation. A faint leakage trace
suffices to distort a backtest. The per-question factor $w(2-w)$ of \Cref{prop:perq} traces the
identical curve with $w$ in place of $\lambda$.}
\label{fig:geometric}
\end{figure}
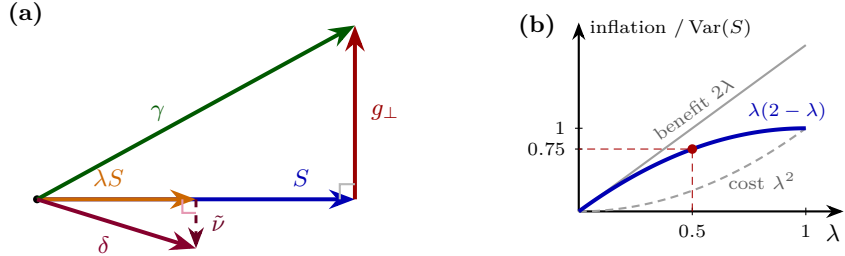

\begin{restatable}[Leakage inflation decomposition]{theorem}{thmdecomposition}
\label{thm:decomposition}
For any clean/corrupt predictor pair with finite second moments, with $\delta$, $S$, $\lambda$,
$\mudelta$, $\tnu$ as in \eqref{eq:delta-decomp} and $\rclean$, $\gperp$ as above,
\begin{equation}
  B_{\mathrm{twin}}\;=\;\Psignal\;-\;\Rshift\;-\;\Rnoise\;+\;\Rcross,
  \label{eq:four-term}
\end{equation}
where, with all outer expectations conditional on $G<0$,
$\Psignal=\E[\lambda(2-\lambda)\Var(S\mid\mathcal H)]$,
$\Rshift=\E[\mudelta^2+2\mudelta\rclean]$,
$\Rnoise=\E[\Var(\tnu\mid\mathcal H)]$, and
$\Rcross=2\,\E[\tnu\gperp]$. For $\lambda\in[0,2]$, $\Psignal\ge0$.
\end{restatable}

\begin{restatable}[The cross-term vanishes]{lemma}{lemcrossterm}
\label{thm:cross-term}
For any such pair, $\Rcross=0$; hence $B_{\mathrm{twin}}=\Psignal-\Rshift-\Rnoise$.
\end{restatable}

\noindent The cross-term vanishes identically---not by assumption---because the Bayesian residual
$\gperp$ is orthogonal to every square-integrable
$\sigma(X,\Dtrain,\Dbridge)$-measurable variable, and $\tnu$ is such a variable. The identity itself imposes no sign
restriction on $\lambda$. Under the additional non-adversarial condition $\lambda\in[0,2]$,
$\Psignal\ge0$; $\Rnoise\ge0$ deflates the contrast, and the shift $\Rshift$ is bounded and vanishes
when $\mudelta=0$ (\Cref{app:proof-shift}). Reading the signs together yields the asymmetry
summarized in \Cref{sec:concentration}.

\begin{corollary}[No free inflation]
\label{cor:no-free}
If $\lambda=0$ and $\mudelta=0$ almost surely, then $B_{\mathrm{twin}}=-\Rnoise\le0$.
\end{corollary}

\noindent The corollary is immediate from
\Cref{thm:decomposition,thm:cross-term}. A perturbation uncorrelated with the leakage signal cannot inflate the backtest:
change orthogonal to $S$ strictly \emph{deflates} the score, so observed inflation is evidence of
extraction. This asymmetry is what makes detection meaningful; the detection
statistic of \Cref{sec:estimation} is built on it.

\paragraph{The double benefit in the twin frame.}
The signal term carries the factor $\lambda(2-\lambda)=1-(1-\lambda)^2$, the twin-frame
analogue of the per-question factor $w(2-w)$ in \Cref{prop:perq}. The arithmetic is the same:
extracting a fraction $\lambda$ of $S$ earns a covariance gain of $2\lambda\Var(S)$---the
forecast now moves \emph{with} the truth, a benefit linear in $\lambda$---while the variance
cost of the added component is only $\lambda^2\Var(S)$, quadratic in $\lambda$. The net is
$\lambda(2-\lambda)\Var(S)$ (\Cref{fig:geometric}). The two factors coincide under the
loading correspondence of \Cref{lem:loading-correspondence}: $\lambda$ loads $\delta$ onto
$S$ with signal strength $\Var(S\mid\mathcal H)$, $w$ is the per-question extraction with
signal strength $b_0$, and when the crowd forecast proxies the honest forecast, the
observable crowd surprise plays the role of the latent $\Var(S\mid\mathcal H)$.

\subsubsection{Infeasibility of the counterfactual comparison}
\label{sec:decomposition:obstruction}

The decomposition is exact and cleanly isolates leakage: the deployed model and its
counterfactual clean version share the same pre-$t_0$ data, so everything they have in
common---including recency---cancels in $\delta$, leaving the pure effect of the bridge data.
The obstruction is that \emph{the counterfactual model cannot be instantiated for a deployed
system}: one cannot un-train a released black box, and the Bayes-optimal predictors, the signal
$S$, and the loading $\lambda$ are all uncomputable. Without the counterfactual model, the only
within-model comparison left is \emph{across time}---pre- versus post-cutoff on the one
deployed model---and there the cancellation fails, because the model is genuinely fresher near
its cutoff: the time comparison reintroduces recency as a confound. The main text therefore
works with the operational inflation $B$ of \eqref{eq:Bdef}, defined on observables of a
single model with recency modeled explicitly; it coincides with $B_{\mathrm{twin}}$ under the
alignment conditions of \Cref{lem:estimand-alignment} (\Cref{app:proofs-nonident}), and the
price of losing the counterfactual comparison is the non-identifiability of
\Cref{thm:nonident}.

\subsection{Proofs for \texorpdfstring{\Cref{app:twin}}{the twin theory}: the leakage
inflation mechanism}
\label{app:proofs-mechanism}

All expectations in this subsection are under the conditional law given $G<0$. Let
$\mathcal H=\sigma(X,\Dtrain)$ and $\mathcal F=\sigma(X,\Dtrain,\Dbridge)$, with
$\mathcal H\subseteq\mathcal F$. The leakage signal is
$S=\ybayes-\ycleanbayes$, where $\ybayes=\E[Y\mid\mathcal F]$ and
$\ycleanbayes=\E[Y\mid\mathcal H]$. Further,
$\delta=\ycorrupt-\yclean$, $\gamma=Y-\yclean$, $\gperp=Y-\ybayes$, and
$\rclean=\yclean-\ycleanbayes$. Under the fixed deterministic protocol of \Cref{sec:setup},
predictions are measurable with respect to the information available to the model:
$\yclean$ is $\mathcal H$-measurable and $\ycorrupt$ is $\mathcal F$-measurable, so $\delta$,
$\rclean$, and (below) $\tnu$ are $\mathcal F$-measurable. The decomposition
$\delta=\lambda S+\mudelta+\tnu$ is the conditional $L^2$ projection given $\mathcal H$, with
$\lambda$ and $\mudelta$ both $\mathcal H$-measurable.

\subsubsection{Proof of \texorpdfstring{\Cref{thm:bayesian}}{Theorem 1}}
\label{app:proof-bayesian}

\Cref{thm:bayesian} is the ideal-case benchmark of \Cref{sec:decomposition:bayes}: when both
predictors are Bayes-optimal, the inflation equals the full information value of the bridge
data.

\thmbayesian*

\begin{proof}
For the Bayes-optimal pair, $\yclean=\ycleanbayes=\E[Y\mid\mathcal H]$ and
$\ycorrupt=\ybayes=\E[Y\mid\mathcal F]$, so the $\mathcal H$-conditional inflation is
\[
b(\mathcal H)
=\E\!\left[(\ycleanbayes-Y)^2\mid\mathcal H\right]
 -\E\!\left[(\ybayes-Y)^2\mid\mathcal H\right].
\]
By the minimum-mean-squared-error property of conditional expectation, the first term equals
$\Var(Y\mid\mathcal H)$, while $\E[(\ybayes-Y)^2\mid\mathcal F]=\Var(Y\mid\mathcal F)$;
conditioning the latter on $\mathcal H\subseteq\mathcal F$ and invoking the law of total
variance,
\[
\Var(Y\mid\mathcal H)
=\E\!\left[\Var(Y\mid\mathcal F)\mid\mathcal H\right]
 +\Var(\ybayes\mid\mathcal H),
\]
we conclude that $b(\mathcal H)=\Var(\ybayes\mid\mathcal H)$. Moreover
$\E[S\mid\mathcal H]=\E[\ybayes\mid\mathcal H]-\ycleanbayes=0$ by the tower property, so
$\Var(S\mid\mathcal H)=\Var(\ybayes\mid\mathcal H)=b(\mathcal H)$. Taking expectations over
$\mathcal H$ yields $B_{\mathrm{twin,Bayes}}=\E[\Var(S\mid\mathcal H)]\ge0$, with equality if and
only if $S=0$ almost surely.
\end{proof}

\subsubsection{Proof of \texorpdfstring{\Cref{thm:decomposition}}{Theorem 2}}
\label{app:proof-decomposition}

\Cref{thm:decomposition} is the central decomposition of \Cref{sec:decomposition:general}: a
general predictor's inflation splits into a signal term, a shift term, a noise term, and a cross
term. \Cref{fig:geometric} in the main text depicts the geometry.

\thmdecomposition*

\begin{proof}
Since $\yclean-Y=-\gamma$ and $\ycorrupt-Y=\delta-\gamma$, definition \eqref{eq:bias-def} reads
\[
B_{\mathrm{twin}}
=\E\!\left[\gamma^2-(\delta-\gamma)^2\right]
=\E\!\left[2\delta\gamma-\delta^2\right],
\]
and it suffices to compute the two conditional moments $\E[\delta^2\mid\mathcal H]$ and
$\E[\delta\gamma\mid\mathcal H]$. Throughout, recall that $\lambda$, $\mudelta$, and $\rclean$
are $\mathcal H$-measurable, that $\E[S\mid\mathcal H]=0$ by the tower property, and that
$\E[\tnu\mid\mathcal H]=\E[S\tnu\mid\mathcal H]=0$ by construction of the projection
\eqref{eq:delta-decomp}; we suppress the conditioning on $\mathcal H$ in the displays below.

The three components of $\delta=\lambda S+\mudelta+\tnu$ are conditionally uncorrelated, so
\[
  \E[\delta^2]=\lambda^2\Var(S)+\mudelta^2+\Var(\tnu).
\]
For the cross moment, substituting $\gamma=S+\gperp-\rclean$ and expanding the product
$\delta\gamma=(\lambda S+\mudelta+\tnu)(S+\gperp-\rclean)$ produces nine terms, of which all but
three vanish: $\E[S\gperp]=\E[\gperp]=0$ because $\gperp=Y-\ybayes$ is the
minimum-mean-squared-error residual, orthogonal to every square-integrable
$\mathcal F$-measurable variable (in particular to $S$ and to constants), while
$\E[S]=\E[\tnu]=\E[S\tnu]=0$ as noted above. What survives is
\[
  \E[\delta\gamma]=\lambda\Var(S)-\mudelta\rclean+\E[\tnu\gperp].
\]
Combining the two moments,
\[
  \E\!\left[2\delta\gamma-\delta^2\mid\mathcal H\right]
  =\lambda(2-\lambda)\Var(S\mid\mathcal H)
   -\left(\mudelta^2+2\mudelta\rclean\right)
   -\Var(\tnu\mid\mathcal H)
   +2\,\E[\tnu\gperp\mid\mathcal H],
\]
and taking expectations over $\mathcal H$ yields \eqref{eq:four-term}. For $\lambda\in[0,2]$ the
factor $\lambda(2-\lambda)$ is non-negative, whence $\Psignal\ge0$. Specializing to the
Bayes-optimal pair ($\lambda=1$, $\mudelta=0$, $\tnu=0$) recovers \Cref{thm:bayesian}.
\end{proof}

\subsubsection{Proof of \texorpdfstring{\Cref{thm:cross-term}}{the cross-term lemma}}
\label{app:proof-cross}

\Cref{thm:cross-term} removes the only unsigned term of \Cref{thm:decomposition}, reducing the
inflation to signal minus shift minus noise.

\lemcrossterm*

\begin{proof}
The Bayesian residual $\gperp=Y-\E[Y\mid\mathcal F]$ satisfies $\E[\gperp h]=0$ for every
square-integrable $\mathcal F$-measurable $h$. As recorded in the preamble to this appendix,
$\tnu=\delta-\lambda S-\mudelta$ is $\mathcal F$-measurable, so taking $h=\tnu$ gives
$\E[\tnu\gperp]=0$, and therefore $\Rcross=2\,\E[\tnu\gperp]=0$; the identity
$B_{\mathrm{twin}}=\Psignal-\Rshift-\Rnoise$ then follows from \Cref{thm:decomposition}. The same
argument applied conditionally shows $\E[\tnu\gperp\mid\mathcal H]=0$ almost surely.
\end{proof}

\subsubsection{The systematic shift is bounded}
\label{app:proof-shift}

The following proposition, stated and proved only here, bounds the shift term $\Rshift$ of
\Cref{thm:decomposition} and identifies when it vanishes.

\begin{proposition}[Bounded systematic shift]
\label{prop:shift-bound}
For any model, $|\Rshift|\le \E[\mudelta^2]+2\sqrt{\E[\mudelta^2]\,\E[\rclean^2]}$, where
$\E[\mudelta^2]\le\E[\delta^2]$ and $\E[\rclean^2]$ is the clean model's excess risk relative to the
clean Bayesian predictor. In particular $\Rshift=0$ for any predictor with $\mudelta\equiv0$
(including the Bayesian and any conditionally unbiased predictor).
\end{proposition}
\begin{proof}
Since $\Rshift=\E[\mudelta^2]+2\,\E[\mudelta\rclean]$, the Cauchy--Schwarz inequality
$|\E[\mudelta\rclean]|\le(\E[\mudelta^2]\,\E[\rclean^2])^{1/2}$ together with the triangle
inequality yields the stated bound. Jensen's inequality applied to
$\mudelta=\E[\delta\mid\mathcal H]$ gives $\mudelta^2\le\E[\delta^2\mid\mathcal H]$ pointwise,
whence $\E[\mudelta^2]\le\E[\delta^2]$. Finally, if $\mudelta\equiv0$ then both terms of
$\Rshift$ vanish; this holds in particular for the Bayes-optimal pair, for which
$\mudelta=\E[S\mid\mathcal H]=0$.
\end{proof}

\subsection{Results of \texorpdfstring{\Cref{sec:impossibility,sec:concentration}}{Sections 3
and 4}: non-identifiability and concentration}
\label{app:proofs-nonident}

Throughout this subsection, $m(x,g)=\E[(P-Y)^2\mid X=x,G=g]$ is the observable conditional mean
loss \eqref{eq:m}, $(m_0,L)$ an operational decomposition satisfying \Cref{ass:struct}, and the
\emph{passive law} is the joint distribution of $(P,Y,X,G)$.

\subsubsection{The alignment lemma}
\label{app:proof-alignment}

The following lemma, referenced in \Cref{rem:twin}, bridges the two halves of the
theory: under its hypotheses, the counterfactual inflation of \Cref{app:twin} and the
operational estimand of \Cref{sec:setup} are the same number.

\begin{lemma}[Alignment of counterfactual and operational estimands]
\label{lem:estimand-alignment}
Suppose that, on the same pre-cutoff question distribution,
$\E[(\yclean-Y)^2\mid X,G]=m_0(X,G)$ almost surely on $\{G<0\}$. Then $B_{\mathrm{twin}}=B$.
\end{lemma}

\noindent The deployed model's side needs no hypothesis, since
$\E[(\ycorrupt-Y)^2\mid X,G]=m(X,G)$ holds by definition of $m$ ($\ycorrupt=P$,
\Cref{sec:setup}). The one substantive condition equates the counterfactual clean model's risk
with the honest surface $m_0$, plausible when the pair shares the training algorithm and
protocol of \Cref{sec:decomposition}.

\begin{proof}
The identity $\E[(\ycorrupt-Y)^2\mid X,G]=m(X,G)$ holds by definition of the observable surface
\eqref{eq:m}, since $\ycorrupt=P$ under the probabilistic protocol of \Cref{sec:setup}; the
hypothesis supplies the matching identity $\E[(\yclean-Y)^2\mid X,G]=m_0(X,G)$ for the clean
model on $\{G<0\}$. Taking expectations over $(X,G)$ under the
pre-cutoff law and applying iterated expectations,
\begin{align*}
B_{\mathrm{twin}}
&=\E\!\left[(\yclean-Y)^2-(\ycorrupt-Y)^2\mid G<0\right]
=\E\!\left[m_0(X,G)-m(X,G)\mid G<0\right]\\
&=\E\!\left[L(X,G)\mid G<0\right]=B,
\end{align*}
where the third equality is the operational decomposition $m=m_0-L$ (\Cref{sec:impossibility}).
\end{proof}

\subsubsection{Proof of \texorpdfstring{\Cref{thm:nonident}}{the non-identifiability theorem}}
\label{app:proof-nonident}

\Cref{thm:nonident} is the paper's central negative result: passive backtest data determine the
observable surface $m$ but nothing more, so the inflation $B$ ranges over a sharp interval.

\thmnonident*

\begin{proof}
The argument has three parts: observational equivalence, attainability, and sharpness.

\emph{Observational equivalence.} The passive law of $(P,Y,X,G)$ determines the conditional mean
loss $m$ through \eqref{eq:m}, but neither $m_0$ nor $L$ separately: these are not functions of
the observables, and any admissible pair $(\tilde m_0,\tilde L)$ with
$\tilde m_0-\tilde L=m$ is consistent with the same passive law. Consequently, if two admissible
pairs yield different values of $B=\E[\tilde m_0-m\mid G<0]$, no functional of the passive law
can equal $B$ for both, and $B$ is not identified.

\emph{Attainability.} Fix $t\in[0,1]$ and define
\[
\tilde m_{0,t}=m+t(1-m)\ \text{on}\ \{g<0\},\qquad
\tilde m_{0,t}=m\ \text{on}\ \{g\ge0\}.
\]
Since $m\in[0,1]$, we have $m\le\tilde m_{0,t}\le1$, so $\tilde m_{0,t}$ is a valid risk
surface; the induced $\tilde L_t=t(1-m)\ge0$ on $\{g<0\}$ and $\tilde L_t=0$ on $\{g\ge0\}$, so
\Cref{ass:struct} holds. The corresponding inflation is
$\tilde B_t=t\,\E[1-m\mid G<0]$, which is continuous and increasing in $t$ and traces every
value in $[0,\E[1-m\mid G<0]]$ as $t$ ranges over $[0,1]$.

\emph{Sharpness.} Conversely, any admissible pair satisfies, pointwise on $\{g<0\}$,
$0\le\tilde L=\tilde m_0-m\le1-m$, the upper bound because $\tilde m_0\le1$; taking expectations
under the pre-cutoff law places $\tilde B$ in the stated interval. The identified set is
therefore exactly $[0,\E[1-m\mid G<0]]$, and additional passive draws refine the estimate of $m$
without shrinking it.
\end{proof}

\subsubsection{Proof of \texorpdfstring{\Cref{prop:perq}}{the per-question law}}
\label{app:proof-perq}

\Cref{prop:perq} converts the convex-pull model into the stakes$\times$extraction law: leakage
per question is the honest Brier times the double-benefit factor.

\propperq*

\begin{proof}
Fix a leakage-eligible question $q$. Under \Cref{ass:pull} and the fixed deterministic protocol,
\[
P(q)-Y(q)=(1-w(q))\,\big(P_{\mathrm{hon}}(q)-Y(q)\big),
\]
so the realized Brier loss is $(P-Y)^2=(1-w)^2 b_0$ with $b_0=(P_{\mathrm{hon}}-Y)^2$, while the
honest loss is $b_0$; the per-question saving is $b_0-(1-w)^2b_0=b_0\,w(2-w)$. Taking
conditional expectations given $(X,G)=(x,g)$ yields the surfaces
\[
m(x,g)=\E\big[(1-w)^2b_0\mid X{=}x,G{=}g\big],\qquad
m_0(x,g)=\E\big[b_0\mid X{=}x,G{=}g\big],
\]
the latter being the loss the model would incur using only legitimate information, whence
$L(x,g)=m_0(x,g)-m(x,g)=\E[b_0\,w(2-w)\mid X{=}x,G{=}g]$. This pair satisfies
\Cref{ass:struct}: $w\in[0,1]$ implies $L\ge0$, and post-cutoff no leaked information exists, so
$P=P_{\mathrm{hon}}$ and $L=0$ there. Finally, iterated expectations and \eqref{eq:Bdef} give
\[
B=\E\big[L(X,G)\mid G<0\big]=\E\big[b_0\,w(2-w)\mid G<0\big],
\]
which is \eqref{eq:perq}.
\end{proof}

\subsubsection{The loading correspondence}
\label{app:proof-loading}

The following lemma, referenced in \Cref{app:twin}, connects the per-question
extraction weight $w$ to the population loading $\lambda$ of \Cref{app:twin} in the
homogeneous special case.

\begin{lemma}[Loading correspondence in the homogeneous pull model]
\label{lem:loading-correspondence}
Suppose within a conditioning stratum that $P_{\mathrm{hon}}=\yclean$ and
$w(q)\equiv w$ is constant. Define the pull direction
$S_{\mathrm{pull}}=Y-P_{\mathrm{hon}}$. Then
$\delta=P-P_{\mathrm{hon}}=wS_{\mathrm{pull}}$ and the $L^2$ loading of $\delta$ onto
$S_{\mathrm{pull}}$ is $w$. If, additionally, $S_{\mathrm{pull}}$ equals the Bayesian leakage
signal $S$ of \Cref{sec:decomposition}, then $\lambda=w$ and $\mudelta=\tnu=0$.
\end{lemma}

\begin{proof}
Within the stratum, \Cref{ass:pull} with $P_{\mathrm{hon}}=\yclean$ and constant $w$ gives
\[
\delta=P-\yclean=w\,\big(Y-P_{\mathrm{hon}}\big)=w\,S_{\mathrm{pull}},
\]
so $\delta$ is proportional to $S_{\mathrm{pull}}$ with coefficient $w$: the $L^2$ loading of
$\delta$ onto $S_{\mathrm{pull}}$ is $w$ and the projection residual vanishes. If in addition
$S_{\mathrm{pull}}=S$, then $\delta=wS$, whence
$\lambda=\Cov(\delta,S\mid\mathcal H)/\Var(S\mid\mathcal H)=w$,
$\mudelta=\E[\delta\mid\mathcal H]=w\,\E[S\mid\mathcal H]=0$, and
$\tnu=\delta-\lambda S-\mudelta=0$.
\end{proof}

\subsubsection{Interpreting the extraction weight}
\label{app:w-interp}

A useful, non-unique parameterization of the extraction weight in \Cref{ass:pull} is
\begin{equation}
  w(q)\;=\;\underbrace{c(q)}_{\text{knowledge \emph{exists}}}\cdot
           \underbrace{u(q)}_{\text{\emph{usable}/deployed}}\cdot
           \underbrace{\kappa(\mathrm{type})}_{\text{answer vs.\ evidence}}.
  \label{eq:w-decomp}
\end{equation}
The factorization is interpretive rather than identified from backtest scores. Existence $c$
asks whether the model absorbed the relevant fact at all, and is probeable leak-free by
self-consistency; usability $u$ captures the well-documented gap between possessing a fact and
deploying it; and $\kappa$ encodes the leakage type of \Cref{sec:nonident:perq}, near $1$ for
answer leakage and intermediate for evidence leakage. The controlled contamination experiment
(E2, \Cref{app:e2}) uses this reading: inserted duplicates raise $c$, and the measured
dose--response tracks exposure that is both absorbed and usable.

\subsection{Results of \texorpdfstring{\Cref{sec:estimation}}{Section 5}: identification
routes}
\label{app:proofs-routes}

Throughout this subsection, $(m_0,L)$ is an operational decomposition satisfying
\Cref{ass:struct} for the model under audit, with superscripts ($m^A$, $m_0^A$, $L^A$)
distinguishing models where needed. For the DiD route, fix the target pre-cutoff covariate
law $Q$ and write, for a model $A$,
\[
\mu^A_{\pm}=\int\E\big[m^A(X,G)\mid X=x,\ G\gtrless0\big]\,dQ(x),
\]
with $\mu^A_{0,\pm}$ the analogous averages of the honest surface $m_0^A$; standardizing all
four cells to $Q$ removes composition differences, and $\Delta^A=\mu^A_+-\mu^A_-$ is the
standardized pre-to-post change of \Cref{prop:did}. In this notation, the formal content of
\Cref{ass:transport} is: $L^{M_0}\equiv0$ on the evaluation support; target-post and control
pre/post observations have common support with $Q$; and the honest changes are parallel,
$\mu^{M}_{0,+}-\mu^{M}_{0,-}=\mu^{M_0}_{0,+}-\mu^{M_0}_{0,-}$. The regression-discontinuity
statements assume the following support regularity.

\begin{assumption}[RD support regularity]
\label{ass:rd-reg}
The density of $G$ is positive and continuous near zero, and the conditional support of $X$
overlaps on both sides of the cutoff. Conditional limits are taken on this common support.
\end{assumption}

\subsubsection{Proof of \texorpdfstring{\Cref{thm:rd}}{the RD theorem}}
\label{app:proof-rd}

\Cref{thm:rd} is the boundary guarantee of Route 1: continuity of the honest surface converts
the observed risk jump at the cutoff into the left-limit leakage.

\thmrd*

\begin{proof}
For $g>0$, $L\equiv0$ by \Cref{ass:struct}, so $m=m_0$ there, and continuity (\Cref{ass:cont})
gives $\lim_{g\downarrow0}m(x,g)=m_0(x,0)$ on the common support of \Cref{ass:rd-reg}. For
$g<0$, $m=m_0-L$, so $\lim_{g\uparrow0}m(x,g)=m_0(x,0)-L(x,0^-)$, the left limit existing by
\Cref{ass:cont}. Subtracting the two limits gives $J(x)=L(x,0^-)$, which is non-negative by
\Cref{ass:struct}.
\end{proof}

\subsubsection{Proof of \texorpdfstring{\Cref{prop:rd-global}}{global extrapolation}}
\label{app:proof-rd-global}

\Cref{prop:rd-global}, stated here in full, extends Route 1 from the boundary to the global
estimand, at the price of the extrapolation class.

\begin{assumption}[Uniquely extrapolable honest surface]
\label{ass:smooth}
$m_0(x,g)=\alpha(x)+\varphi(g)$, where $\varphi$ belongs to a known finite-dimensional class
$\Phi$ whose restriction to the observed clean support uniquely determines it on the target
pre-cutoff support (e.g.\ a fixed-degree polynomial), and the class is correctly specified.
\end{assumption}

\begin{restatable}[Global leakage by clean-side extrapolation]{proposition}{proprdglobal}
\label{prop:rd-global}
Under the assumptions of \Cref{thm:rd} and \Cref{ass:smooth}, the honest surface on $\{g<0\}$ is
identified by extrapolating the fit of $m_0$ from $\{g>0\}$, and
$B=\E[\,m_0(X,G)-m(X,G)\mid G<0\,]$ is identified.
\end{restatable}

\begin{proof}
On $\{g>0\}$, $m=m_0=\alpha(x)+\varphi(g)$ is observed; the two components are identified only
up to an additive normalization, but their sum---the clean surface---is identified. By
\Cref{ass:smooth}, the restriction of $\varphi$ to the observed clean support uniquely
determines its continuation to the target pre-cutoff support, and hence determines $m_0$ there.
Consequently $L=m_0-m$ is identified on $\{g<0\}$, and so is $B=\E[L\mid G<0]$.
\end{proof}

\subsubsection{Proof of \texorpdfstring{\Cref{prop:did}}{DiD identification}}
\label{app:proof-did}

\Cref{prop:did} is the global guarantee of Route 2: a clean, recency-matched control converts
the difference of pre/post changes into the standardized inflation.

\propdid*

\begin{proof}
Target leakage vanishes on $\{G>0\}$ by \Cref{ass:struct}, so $\mu^M_+=\mu^M_{0,+}$. On the pre
side, $m^M=m_0^M-L^M$ integrates to $\mu^M_-=\mu^M_{0,-}-\int\E[L^M\mid X{=}x,G{<}0]\,dQ(x)$,
and because $Q$ is the target pre-cutoff covariate law, the integral equals
$\E[L\mid G<0]=B$ by iterated expectations. Hence
$\Delta^M=\big(\mu^M_{0,+}-\mu^M_{0,-}\big)+B$. The control is clean ($L^{M_0}\equiv0$), so
$\Delta^{M_0}=\mu^{M_0}_{0,+}-\mu^{M_0}_{0,-}$. Subtracting and applying the parallel honest
change of \Cref{ass:transport} cancels the honest terms, leaving
$\Delta^M-\Delta^{M_0}=B$.
\end{proof}

\subsubsection{Boundary DiD with a no-discontinuity control}
\label{app:proof-did-boundary}

The following variant, stated and proved only here, weakens the clean-control requirement to
continuity of the control's leakage at the target cutoff; it identifies the boundary estimand
only, and underwrites the control checks of M4 (\Cref{sec:exp:e3}).

\begin{proposition}[Boundary DiD with a no-discontinuity control]
\label{prop:did-boundary}
At the target cutoff, define the observed risk jump
$J_{\mathrm{obs}}^A(x)=m^A(x,0^+)-m^A(x,0^-)$. Suppose
(i) target leakage is zero on the post side and has left limit $L^M(x,0^-)$;
(ii) control leakage is continuous at the target cutoff; and
(iii) target and control share the same honest/composition jump,
$m_0^M(x,0^+)-m_0^M(x,0^-)=m_0^{M_0}(x,0^+)-m_0^{M_0}(x,0^-)$.
Then
\[
J_{\mathrm{obs}}^M(x)-J_{\mathrm{obs}}^{M_0}(x)=L^M(x,0^-).
\]
\end{proposition}
\begin{proof}
Write $J_0^A$ for the honest jump of model $A$ appearing in condition (iii). For the target,
condition (i) gives $J_{\mathrm{obs}}^M=J_0^M+L^M(x,0^-)$: post-cutoff leakage vanishes while
the pre-side limit contributes $L^M(x,0^-)$. For the control, condition (ii) makes its leakage
difference across the cutoff vanish, so $J_{\mathrm{obs}}^{M_0}=J_0^{M_0}$. Condition (iii)
equates the honest jumps, and differencing leaves $L^M(x,0^-)$.
\end{proof}

\subsubsection{Proof of \texorpdfstring{\Cref{prop:efficiency}}{PRC detection}}
\label{app:proof-efficiency}

\Cref{prop:efficiency}, stated here in full, gives the population signature of Route 3: under
calibration and homogeneous extraction, the residual covariance is the hump-shaped function of
the extraction weight that makes PRC a detector of evidence leakage. It rests on the following
population conditions.

\begin{assumption}[PRC population conditions]
\label{ass:prc-pop}
Within the pre-cutoff population, $P_{\mathrm{hon}}=c_0$,
$\E[Y\mid c_0,G<0]=c_0$, and the extraction weight is homogeneous:
$w(q)\equiv\bar w$.
\end{assumption}

\noindent Conditional calibration is achievable by construction---recalibrate on leakage-free
anchors (\Cref{rem:calib})---and $P_{\mathrm{hon}}=c_0$ takes the crowd as the honest
benchmark, appropriate when the model holds no legitimate edge over public information.
Homogeneous $w$ is a first-order simplification: under heterogeneous $w$ the covariance is a
nonlinear functional of the joint law of $(w,c_0,Y)$, not of $B$, and the endpoint cases below
show it cannot separate pure answer leakage ($w\equiv1$) from no leakage at all. In the
population result, $P$ denotes the framing-invariant limit $P_\infty$ of the paraphrase
consensus; the finite-$K$ deployment, including the framing-invariance conditions, is
developed in \Cref{app:delta-machinery}.

\begin{restatable}[Residual-covariance detection under convex pull]{proposition}{propefficiency}
\label{prop:efficiency}
Under \Cref{ass:pull,ass:prc-pop},
\[
\Cov(P,Y-P\mid G<0)
=\bar w(1-\bar w)\,\E[c_0(1-c_0)\mid G<0]\;\ge\;0,
\]
and the covariance vanishes whenever $w\equiv0$ or $w\equiv1$.
\end{restatable}

\begin{proof}
All moments are conditional on $G<0$. Under \Cref{ass:pull} with homogeneous weight $\bar w$ and
$P_{\mathrm{hon}}=c_0$, we have $P=(1-\bar w)c_0+\bar wY$ and $Y-P=(1-\bar w)(Y-c_0)$.
Conditional calibration ($\E[Y\mid c_0,G<0]=c_0$) gives $\Cov(c_0,Y-c_0)=0$ and, for binary
$Y$, $\Cov(Y,Y-c_0)=\E[\Var(Y\mid c_0)]=\E[c_0(1-c_0)]$. Hence
\[
\Cov(P,Y-P)
=(1-\bar w)\,\Cov\big((1-\bar w)c_0+\bar wY,\;Y-c_0\big)
=\bar w(1-\bar w)\,\E[c_0(1-c_0)]\;\ge\;0 .
\]
If $w\equiv0$ then $P=P_{\mathrm{hon}}=c_0$ and the covariance equals $\Cov(c_0,Y-c_0)=0$; if
$w\equiv1$ then $P=Y$, so $Y-P=0$ and the covariance vanishes as well.
\end{proof}

\begin{remark}[Calibration is mandatory]
\label{rem:calib}
\Cref{prop:efficiency} presumes a calibrated honest forecast. Real LLMs are markedly
overconfident (on our data, fitted temperature $T\approx8$; provenance in \Cref{app:e4}), so
the \emph{raw} residual covariance is dominated by miscalibration and is in fact
negative---the opposite of the leakage signature. One must recalibrate $P$ on leakage-free
anchors first, or use a differenced form; a positive \emph{calibrated} value indicates
evidence leakage.
\end{remark}

\subsubsection{A sufficient external reference}
\label{app:proof-minimal}

The following corollary, referenced in \Cref{sec:estimation:minimal}, records that a single
external reference restores the point identification that \Cref{thm:nonident} denies to
passive data.

\begin{corollary}[A sufficient external reference]
\label{prop:minimal}
With only the passive backtest distribution of $(P,Y,X,G)$ (including the crowd anchor $c_0$),
$B$ is set-identified, not point-identified (\Cref{thm:nonident}). Each of the following
suffices to point-identify $B$ under its stated assumptions: \textnormal{(R1)} a known cutoff
with questions on both sides (\Cref{prop:rd-global}); \textnormal{(R2)} a clean,
capability-matched control on the same questions (\Cref{prop:did}). The intervention
\textnormal{(R3)} detects evidence leakage (\Cref{prop:efficiency}) but does not point-identify
it.
\end{corollary}
\begin{proof}
Immediate from \Cref{thm:nonident,prop:rd-global,prop:did,prop:efficiency}.
\end{proof}

\begin{remark}[Metric scope]
\label{rem:metric}
The route results use only the structure $m=m_0-L$ and hold for any bounded loss---Brier,
$1-$accuracy, or $1-\text{pass@}1$---which licenses the pass@1 correction of
\Cref{sec:exp:e3}; the law $L=b_0\,w(2-w)$ and the $\PRC$ statistic are Brier-specific.
\rev{For generative tasks, we model the success probability. Let $p_0$ be the honest probability
that the model solves a problem, and let leakage increase it to $p=(1-w)p_0+w$. This form holds if
the model reproduces a memorized solution with probability $w$ and otherwise attempts the problem
honestly. As in \Cref{sec:concentration}, any success rate between $p_0$ and $1$ has this
form for some $w\in[0,1]$. Thus the form includes partial recall of a multi-step solution.}

\rev{The inflation is then $w(1-p_0)$, which is stakes times extraction as in the binary case. Thus
leakage still concentrates on problems that the model otherwise fails. The double benefit
$w(2-w)$ does not apply, because it comes from the quadratic Brier loss. $\PRC$ has no direct
analog without probabilistic forecasts.}
\end{remark}

\subsubsection{Complementary signatures}
\label{app:proof-connection}

The following proposition, referenced in \Cref{sec:estimation:minimal}, combines the boundary
and residual signatures into the complementarity statement that motivates pairing the probes.
We make no formal claim about a consistent test; a sampling-level analysis would require
thresholds, finite-sample noise models, and the availability of both probes.

\begin{proposition}[Complementary signatures under homogeneous extraction]
\label{prop:connection}
Under \Cref{ass:struct,ass:pull,ass:cont,ass:prc-pop} and the support conditions of
\Cref{app:proofs-routes}, suppose the stakes and homogeneous extraction have left limits $b_0$
and $w$ at the cutoff, and let $\sigma_S^2=\E[c_0(1-c_0)\mid G<0]$. Then
\[
\begin{aligned}
  J(w)&=b_0\,w(2-w)
    &&(\text{increasing; maximal at }w=1),\\
  \PRC(w)&=w(1-w)\sigma_S^2
    &&(\text{maximal at }w=\tfrac12;\ \text{zero at }w\in\{0,1\}).
\end{aligned}
\]
\end{proposition}

\begin{proof}
By \Cref{thm:rd} and \Cref{prop:perq}, the boundary jump equals the left-limit leakage,
$J=L(\cdot,0^-)=b_0\,w(2-w)$, which is increasing in $w$ on $[0,1]$ and maximal at $w=1$. By
\Cref{prop:efficiency} in the homogeneous case, $\PRC=w(1-w)\sigma_S^2$, maximal at
$w=\tfrac12$ and zero at $w\in\{0,1\}$. In particular $\PRC(1)=0$ while $J(1)=b_0>0$, and
$\PRC(\tfrac12)=\sigma_S^2/4>0$.
\end{proof}

\noindent The complementarity is verified on the synthetic sweep of \Cref{app:e1}; the
answer-leakage regime $w\to1$ does not arise in the ground-truth experiment (measured
$w\le0.42$, \Cref{sec:exp:m3}), so its real-data face is untested.

\section{Deploying the routes}
\label{app:delta-machinery}

\Cref{fig:decision} summarizes the practitioner's choice of route given the available resources;
it is the decision-oriented complement of \Cref{tab:routes}, which states what each route
guarantees and how its assumption is checked.

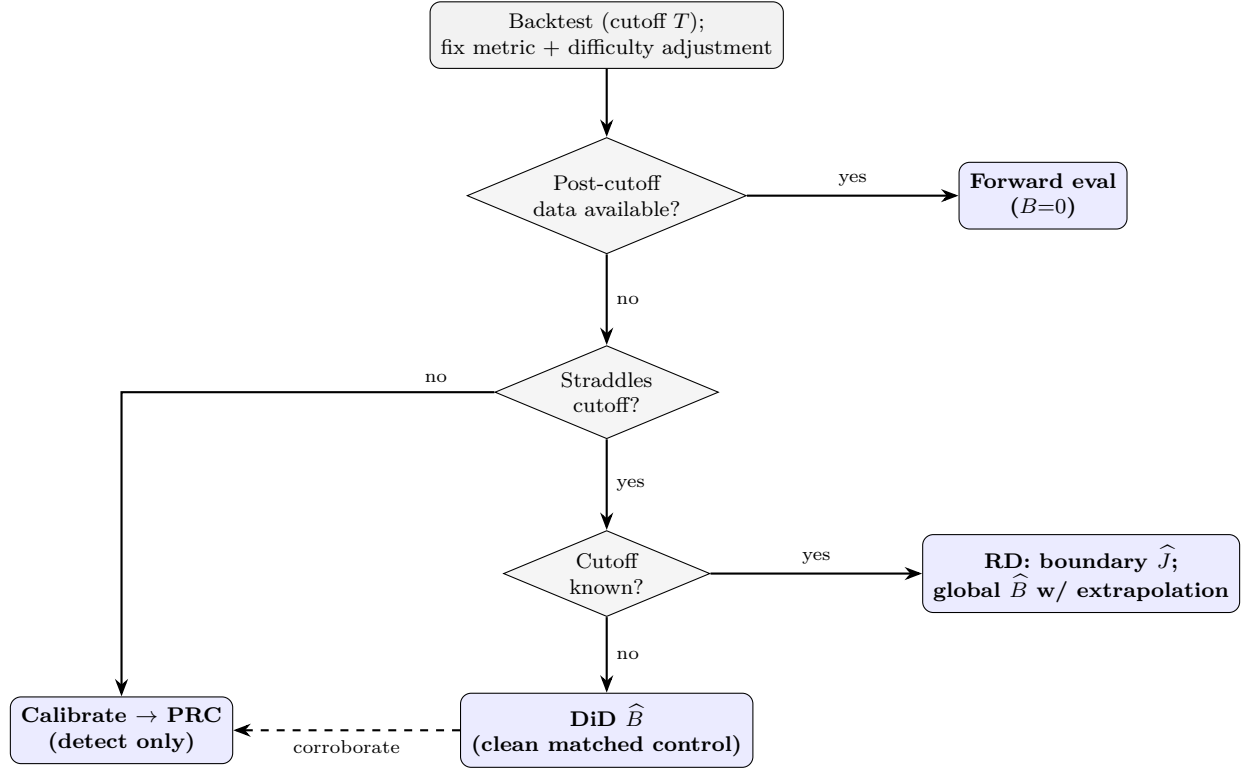
\begin{figure}[t]
\centering
\begin{tikzpicture}[
    >=Stealth, node distance=9mm,
    box/.style={draw, rounded corners, align=center, font=\footnotesize, inner sep=4pt, minimum height=7mm},
    route/.style={draw, rounded corners, align=center, font=\footnotesize\bfseries, fill=blue!8, inner sep=4pt},
    q/.style={draw, diamond, aspect=2.4, align=center, font=\footnotesize, inner sep=2pt, fill=gray!8},
    arr/.style={->, thick},
  ]
  \node[box, fill=black!5] (start) {Backtest (cutoff $T$);\\ fix metric $+$ difficulty adjustment};
  \node[q, below=of start] (wait) {Post-cutoff\\ data available?};
  \node[route, right=28mm of wait] (fwd) {Forward eval\\ ($B{=}0$)};
  \node[q, below=12mm of wait] (straddle) {Straddles\\ cutoff?};
  \node[q, below=12mm of straddle] (known) {Cutoff\\ known?};
  \node[route, right=28mm of known] (rd) {RD: boundary $\widehat J$;\\ global $\Bhat$ w/ extrapolation};
  \node[route, below=10mm of known] (did) {DiD $\Bhat$\\ {\footnotesize(clean matched control)}};
  \node[route, left=30mm of did] (prc) {Calibrate $\to$ PRC\\ (detect only)};

  \draw[arr] (start) -- (wait);
  \draw[arr] (wait) -- node[above,font=\scriptsize]{yes} (fwd);
  \draw[arr] (wait) -- node[right,font=\scriptsize]{no} (straddle);
  \draw[arr] (straddle) -- node[right,font=\scriptsize]{yes} (known);
  \draw[arr] (known) -- node[above,font=\scriptsize]{yes} (rd);
  \draw[arr] (known) -- node[right,font=\scriptsize]{no} (did);
  \draw[arr] (straddle.west) -| node[above left,font=\scriptsize,pos=0.05]{no}
        (prc.north);
  \draw[arr, dashed] (did) -- node[below,font=\scriptsize]{corroborate} (prc);
\end{tikzpicture}
\caption{\textbf{Choosing a route: report a leakage-adjusted score only when global $B$ is
identified; otherwise report the boundary or detection estimand.}
Forward evaluation ($B=0$) is preferred when post-cutoff outcomes can be awaited. Otherwise, RD
identifies the boundary jump when the cutoff is known (and global $B$ under extrapolation), and
DiD identifies global $B$ with a clean, capability/recency-matched control; if no data straddle
the cutoff, only the calibrated PRC detection statistic remains. The flow is a simplification:
with both a known cutoff and a matched control available, DiD may be preferred for global $B$.
\Cref{tab:routes} states each route's assumptions and checks.}
\label{fig:decision}
\end{figure}

\subsection{The practitioner's recipe and control validation}
\label{app:recipe}

\paragraph{The practitioner's recipe.}
\textbf{(1)} Fix the evaluation metric and a leakage-free
difficulty adjustment; for probabilistic forecasts, a contemporaneous crowd forecast provides
the stakes anchor. \textbf{(2)} Estimate $\Bhat$ by a point-identifying route: DiD requires a
capability/recency-matched clean control (selection and validation protocol below\rev{,
including the trend test}), while RD
identifies the boundary jump and needs the
extrapolation assumption for global $B$. \textbf{(3)} If using $\PRC$, first calibrate on
leakage-free anchors and report it only as a detection statistic
(\Cref{rem:calib,prop:connection}); calibration is not required for accuracy- or pass@1-based
RD/DiD contrasts. \textbf{(4)} When global $B$ is identified, report $\Rmeas+\Bhat$ (or the
score form of \Cref{cor:honest}) with a confidence interval; otherwise report the boundary or
detection estimand without converting it into a global correction. \textbf{(5)} Prefer
\emph{forward} evaluation when feasible, for which $B=0$ by construction; and if the cutoff is
uncertain or staged, report sensitivity to its placement rather than treating a single date as
known.

\paragraph{Selecting and validating a DiD control (Route 2).}
The control should be a same-family snapshot or a vintage-matched peer of the target---same
capability tier and generation, similar training-data distribution---whose documented cutoff
predates the evaluation window; its cleanliness should be verified empirically rather than
assumed, for example by date-only recall probes on the window's questions. \rev{A match in capability tier is a heuristic for similar recency, not a requirement.
\Cref{ass:transport} asks only for parallel change, and level differences cancel.} The parallel-change
condition of \Cref{ass:transport} is not fully testable, since the target's pre-cutoff cell is
contaminated by construction, but it has three testable implications: on the jointly clean
region $\{G>0\}$ both models' risk trends are honest and should move in parallel \rev{(the
trend test of \Cref{app:e9-robust} measures the gap)}; placebo
cutoffs at dates with no leakage discontinuity should return estimates near zero; and covariate
balance should hold across the four standardized cells. When several candidate controls are
available, agreement of $\Bhat$ across them is an over-identification check, and their spread
is a systematic-uncertainty band. Finally, the dominant violation is signed
(\Cref{sec:estimation:did}): a weak control inflates $\Bhat$, so a null finding is conservative
under this failure mode, whereas a positive finding should be reported together with the
matching checks. \rev{The opposite mismatch hides leakage. The trend test measures this mismatch where both models are clean, with only the scores that we already collected.} M5 (\Cref{app:m5}) implements this protocol\rev{, and its trend test (\Cref{app:e9-robust}) finds a negative gap between Qwen3.5 and GPT-5}.

\subsection{Finite-\texorpdfstring{$K$}{K} PRC machinery (Route 3)}
\label{app:prc-machinery}

The deployable form of the paraphrase route (\Cref{sec:estimation:prc}) uses a finite number
of queries. Query the model with $K$ paraphrases per question, form the
consensus $\bary_{K,i}=K^{-1}\sum_k P_i^{(k)}$, and compute
$\widehat\Cov(\bary_K,Y-\bary_K)$. The target population quantity is
$\PRC_\infty=\Cov(P_\infty,Y-P_\infty\mid G<0)$ from \Cref{prop:efficiency}, where $P_\infty$ is
the framing-invariant consensus defined by the following condition (moved here from the main
text because only the finite-$K$ analysis uses it).

\begin{assumption}[Framing-invariant paraphrase noise]
\label{ass:framing}
Within the pre-cutoff population, for paraphrase $\pi_k$, write
$P^{(k)}=P_\infty+\varepsilon_k$, where the leaked signal and
$P_\infty$ do not depend on framing. Conditional on the question, the deviations are identically
distributed, mean zero, mutually uncorrelated, and uncorrelated with $P_\infty$ and
$Y-P_\infty$. They have finite fourth moments.
\end{assumption}

\begin{lemma}[Finite-paraphrase attenuation]
\label{lem:finite-k}
Under \Cref{ass:framing}, let
$W=\E[\Var_\pi(P^{(k)}\mid q,G<0)\mid G<0]$. For constant $K$,
\[
\Cov(\bary_K,Y-\bary_K\mid G<0)=\PRC_\infty-\frac{W}{K}.
\]
\end{lemma}
\begin{proof}
Write $\bary_K=P_\infty+\bar\varepsilon_K$. Conditional mean-zero and mutual
uncorrelatedness give $\Var(\bar\varepsilon_K\mid q)=\Var(\varepsilon_k\mid q)/K$.
The remaining orthogonality conditions in \Cref{ass:framing} give
\[
\Cov(P_\infty+\bar\varepsilon_K,\,
Y-P_\infty-\bar\varepsilon_K)
=\Cov(P_\infty,Y-P_\infty)-\Var(\bar\varepsilon_K).
\]
Taking the question-level expectation yields the result.
\end{proof}

\paragraph{Estimator and inference.}
Let
\[
\widehat W=\frac1m\sum_{i=1}^m\frac1{K-1}
\sum_{k=1}^K(P_i^{(k)}-\bary_{K,i})^2.
\]
The bias-corrected estimator is
\[
\PRChat_{\mathrm{bc}}
=\widehat\Cov(\bary_K,Y-\bary_K)+\widehat W/K.
\]
For varying $K_i$, replace $\widehat W/K$ by the average of the per-question sample
variances divided by $K_i$. Under independent questions, \Cref{ass:framing}, and finite fourth
moments, the vector of sample moments defining $\PRChat_{\mathrm{bc}}$ obeys a multivariate
central limit theorem, and the delta method yields
$\sqrt m(\PRChat_{\mathrm{bc}}-\PRC_\infty)\Rightarrow\mathcal N(0,\sigma_\PRC^2)$; in practice
we bootstrap questions and refit any calibrator inside each replicate. $\PRC$ remains a
detection statistic, not an estimator of $B$ (\Cref{prop:efficiency}), and paraphrase diversity
must come from framing rather than sampling temperature (\Cref{app:paraphrase}).

\subsection{Paraphrase set and generation protocol}
\label{app:paraphrase}

Route 3 (\Cref{sec:estimation}) requires paraphrases that vary framing without revealing the
outcome; the set documented here illustrates the design principles and the generation
protocol.

\paragraph{Design principles.}
Effective paraphrases must (i) preserve the event and resolution criteria (same answer), (ii)
maximize variation in the reasoning channel, and (iii) reveal nothing about the outcome. We vary
six dimensions independently: lexical choice, sentence form, level of specificity, entity
reference, framing/tone, and which background context is emphasized.

\paragraph{Generation and validation.}
Paraphrases are generated by an auxiliary model (distinct from the model under test) instructed to
preserve meaning and resolution criteria while varying the six dimensions and concealing the
outcome. Candidates are validated by three checks: semantic preservation (embedding cosine
similarity above a threshold), lexical diversity (low pairwise self-BLEU), and an outcome-leakage
screen rejecting any phrasing that encodes the answer's direction. Querying uses temperature $0$
with reasoning disabled and a strict final-answer parser; truncated or out-of-range generations
are retried once and otherwise discarded. The protocol is instantiated, with the specific $K$
and models used, in the $\PRC$ experiments of \Cref{app:e1,app:prc-real}.

\subsection{Classical identification tools used in the paper}
\label{app:related-tools}

The machinery instantiates standard tools. The residual covariance is a crowd-anchored
analogue of Mincer--Zarnowitz efficiency tests and of Murphy's excess resolution
\citep{mincer1969evaluation,elliott2016economic,patton2012forecast,murphy1973new,gneiting2007strictly}.
The routes are regression discontinuity \citep{imbens2008regression,lee2010regression} and
difference-in-differences, with a fixed-question intervention in the spirit of control
functions and specification tests
\citep{wooldridge2015control,wu1973alternative,hausman1978specification}. The sharp-bounds
framing of \Cref{thm:nonident} follows partial identification
\citep{manski2003partial,imbens2004confidence}. The finite-$K$ correction is classical
errors-in-variables disattenuation
\citep{fuller1987measurement,spearman1910correlation}. Two superficially appealing
instruments do not transfer and serve only as motivation: temporal-head ablation removes
legitimate temporal reasoning along with leakage, and item-preknowledge item response theory
requires cross-examinee variation that a single deployed model lacks.

\section{Experiment configurations and full results}
\label{app:experiments}

Each subsection below gives, for one main-text experiment (M1--M5, in order), the full
configuration---data provenance, model inventories, estimator definitions, and inference
procedures---together with the complete results summarized in \Cref{sec:experiments}. The synthetic validation study E1,
the per-experiment robustness analyses, and all supplementary arms are collected in
\Cref{app:supp}. \Cref{tab:contract} is the map: each theoretical claim of
\Cref{sec:impossibility,sec:concentration,sec:estimation} and the experiment whose headline
result carries it.

\begin{table}[t]\centering
\caption{\textbf{Claim-to-evidence contract.} Each theoretical claim of
\Cref{sec:impossibility,sec:concentration,sec:estimation} and the experiment whose headline
result carries it (M1: \Cref{sec:intro}; M2, M3: \Cref{sec:validation}; M4, M5:
\Cref{sec:deployment}). $^\ast$: 95\% CI excludes zero.}
\label{tab:contract}
\footnotesize
\begin{tabular}{@{}p{6.0cm}p{1.3cm}p{5.5cm}@{}}
\toprule
Claim & Carried by & Headline result \\
\midrule
Recency confounds naive backtests (\Cref{thm:nonident}) & M1 &
4/5 cutoff-clean flagships show starred naive gaps ($+0.044$ to $+0.061^\ast$) \\
\addlinespace[2pt]
Inflation rises with contamination dose (\Cref{thm:decomposition}, \Cref{app:twin}) & M2, M3 &
monotone in $r$ at pretraining scale and in forecasting twins
(dose-64 twin contrast $+0.063^\ast$; dose-0 null) \\
\addlinespace[2pt]
Per-question law $L=b_0\,w(2-w)$ (\Cref{prop:perq}, \Cref{ass:pull}) & M3 &
concentration $13\times$ across $b_0$ quintiles; \rev{law inside the CI in 7 of 10 quintile cells, largest miss in the top quintile};
measured $w$ declines with difficulty \\
\addlinespace[2pt]
Change orthogonal to the signal cannot inflate (\Cref{cor:no-free}, \Cref{app:twin}) & M3 &
outcome-scrubbed control twin gains recency but shows no inflation (naive $+0.001$; RD
$+0.012$, no jump) \\
\addlinespace[2pt]
RD detects boundary leakage; recency alone gives no jump (\Cref{thm:rd}) & M3, M4 &
twin RD $+0.052$ vs control $+0.012$, placebo null; wild diagonals starred on code
($+0.090^\ast$, $+0.095^\ast$) and via the clean anchor (\rev{$+0.078^\ast$}) \\
\addlinespace[2pt]
DiD with a matched control returns disciplined nulls (\Cref{prop:did}, \Cref{ass:transport}) & M5 &
\rev{no target significantly positive} after adjustment; weak control re-inflates to $+0.055$; \rev{power floors
$0.046$--$0.098$} \\
\addlinespace[2pt]
PRC detects evidence leakage, gated on calibration (\Cref{prop:efficiency},
\Cref{app:proof-efficiency}) & M3 &
differenced PRC $+0.0069^\ast$ only on injected cells; raw PRC negative everywhere \\
\bottomrule
\end{tabular}
\end{table}

\subsection{Scope and coverage}
\label{app:coverage}

\Cref{tab:coverage} summarizes, for each experiment, the data source, sample size, models, metric,
and the machinery that carries statistical significance. It is the single reference for what each
reported interval or $p$-value is computed from.

\begin{table}[h]\centering
\small
\caption{\textbf{Scope and coverage of the empirical program.} One row per experiment; the last
column names the inference machinery behind every significance claim in \Cref{sec:experiments}.}
\label{tab:coverage}
\begin{tabular}{@{}p{1.55cm}p{2.6cm}p{2.2cm}p{3.3cm}p{1.5cm}p{2.6cm}@{}}
\toprule
Experiment & Data source & $n$ & Models & Metric & Inference \\
\midrule
M1 & ForecastBench market panel & $332$--$767$ post-cutoff questions/model & GPT-5,
Gemini-3.1-Pro, Kimi-K2.6, GPT-5.5, MiniMax-M3 & anchored Brier reduction & source$\times$month
cluster bootstrap \\
\addlinespace[4pt]
M2 & Hubble published accuracies & $4$ tasks $\times$ $6$ doses & Hubble 1B/8B (perturbed,
standard) & accuracy & exact permutation trend test; SEM-propagated intervals \\
\addlinespace[4pt]
M3 & ForecastBench panel + Tinker LoRA twins & $988$ injected-pool + $543$ post-cutoff questions
& Qwen3.5-35B-A3B-Base twins (trained here) & anchored Brier reduction; log-prob & item bootstrap;
paired bootstrap; placebo boundaries \\
\addlinespace[4pt]
M4-F & ForecastBench panel + archived real-time forecasts & $1{,}646$ questions; \rev{$97$--$399$}
anchored & DeepSeek-V3.1, Kimi-K2.6, GPT-5.4, GPT-5.5; controls MiniMax-M3, Claude-Opus-4.7;
family anchors GPT-5-Mini, GPT-5.1, GPT-4.1, o4-mini & anchored Brier reduction & cluster
bootstrap; diagonal permutation test; horizon banding \\
\addlinespace[4pt]
M4-C & LiveCodeBench per-problem pass@1 & $401$--$437$ problems per window & GPT-4o-0806,
Claude-3.5-Sonnet (published); GPT-5, GPT-5-Mini, DeepSeek-V3-0324, Kimi-K2-0905 (generated);
controls: 2025-cutoff pool & pass@1 & problem bootstrap ($4000$); date-permutation test \\
\addlinespace[4pt]
M5 & ForecastBench market panel & $265$ pre- / $1{,}207$ post-boundary paired questions &
Qwen3.5-35B-A3B (target), GPT-5 (matched control), Kimi-K2.6, DeepSeek-V3.2, GLM-4.7, MiniMax-M3,
GPT-3.5-Turbo (weak) & anchored Brier reduction & cluster bootstrap; boundary sweep;
\rev{power floors} \\
\addlinespace[4pt]
E1 (app.) & synthetic, known DGP & $400$--$8000$/rep; $200$--$300$ reps & simulated forecasters &
Brier reduction & bootstrap CIs; MC coverage \\
\bottomrule
\end{tabular}
\end{table}

\subsection{The forecasting panel and M1 (clean-flagship naive check)}
\label{app:m1}

\paragraph{Panel construction.} The panel is built from the public ForecastBench archive
\citep{karger2025forecastbench}: the nightly question banks and the resolution files published in
the project's datasets repository. We retain market-source binary questions (Polymarket,
Metaculus, Manifold, INFER) that are resolved, have a valid frozen market probability
$c_0\in(0,1)$ recorded before resolution, and have a resolution date between July 2024 and July
2026. Combination questions (logical conjunctions of base questions) are excluded. After
deduplication by question identifier this yields $1{,}646$ unique questions. Dataset-source
questions (for example weather and economic series) are excluded from the headline analyses
because their answers are tabular rather than event outcomes; they are retained in the archive
audit. The panel, the raw downloads, and the build script are versioned with the code.

\paragraph{Elicitation protocol.} Every model is queried once per question through OpenRouter
with temperature $0$, reasoning disabled where the endpoint permits, and a fixed prompt that
states the question, its resolution criteria, and the resolution date, and requests a single
probability. Responses are parsed to a float; refusals and parse failures are dropped and
reported per model (the retained fraction appears in each experiment's sample sizes). Scores are
crowd-anchored Brier reductions $L(q)=(c_0-Y)^2-(P-Y)^2$.

\paragraph{M1 design.} For each of the five flagships (GPT-5, cutoff 2024-09-30; Gemini-3.1-Pro,
2025-01-31; Kimi-K2.6, 2025-04-30; GPT-5.5, 2025-12-01; MiniMax-M3, 2026-01-31) we keep questions
resolving at least 30 days after the documented cutoff, so that no outcome can be in training
data. Cutoffs are taken from vendor documentation collected in the pre-registration; Kimi-K2.6's
cutoff is documented at coarser (tier-two) resolution, a fact shown not to be load-bearing by
the cutoff-shift check of \Cref{app:m4}. The
naive statistic is the pre/post gap in mean $L$ around the median resolution date of each model's
clean window. Confidence intervals use a cluster bootstrap over source$\times$month cells
($4{,}000$ resamples).

\paragraph{M1 headline results (demoted detail).}
Four of the five clean flagships show starred naive gaps of $+0.044$ to $+0.061$
(\Cref{fig:m1forest}): an auditor running the standard pre/post check would flag four model
families as leaking on questions they cannot possibly have seen. The fifth model is the
exception that proves the mechanism. MiniMax-M3's null ($+0.017$) is a property of its
\emph{window}, not of the model: restricting all five models to its short clean window
(March through July 2026) collapses every gap to insignificance---GPT-5.5 falls from
$+0.061^\ast$ to $+0.008$, and MiniMax's gap becomes the largest of the five. The naive
statistic therefore measures how much recency contrast an evaluation window spans, not
whether the audited model leaked. This is the empirical face of \Cref{thm:nonident}: a
passive backtest gap carries recency and leakage in a single number, and at roughly $0.05$
Brier-reduction units in every family we test, the confound is far too large to ignore.

\paragraph{Matched-window check.} Because each model's clean window starts at its own cutoff,
gap sizes are not comparable across models: a longer window spans more recency contrast between
its pre and post halves. Restricting all five models to the shortest clean window (MiniMax-M3's:
resolutions from 2026-03-02, split at the common median 2026-05-14, $n\approx333/337$) makes
every gap statistically null and nearly equal: GPT-5 $+0.001$ $[-0.053,+0.042]$, Gemini-3.1-Pro
$+0.005$ $[-0.055,+0.047]$, Kimi-K2.6 $+0.015$ $[-0.025,+0.058]$, GPT-5.5 $+0.008$
$[-0.034,+0.049]$, MiniMax-M3 $+0.017$ $[-0.020,+0.047]$. The spread across models in
\Cref{fig:m1forest} is therefore driven by the windows, not the model families, which is the
recency-confound reading the theory predicts. The same check backs the cutoff-integrity
reading of the documented cutoffs: had a documented cutoff materially understated a model's
training data, that model's gap would not collapse alongside the others in the common window.

\subsection{M2 (Hubble QA)}
\label{app:e2}

\begin{figure}[t]\centering
\figorbox[0.62\linewidth]{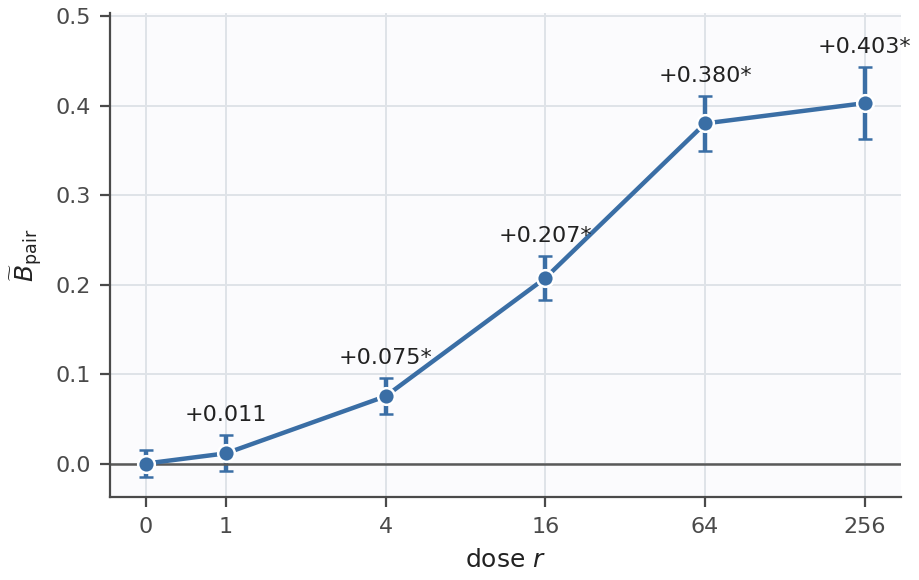}
\caption{\textbf{M2: at pretraining scale, measured accuracy inflates strictly monotonically
with the known contamination dose.} Hubble pretrains a \emph{perturbed} model with benchmark
documents inserted $r$ times and a \emph{standard} twin trained identically without them. The
x-axis is the dose $r$ (log scale). The y-axis is the placebo-centered contrast
$\widetilde B_{\mathrm{pair}}(r)=\widehat B_{\mathrm{pair}}(r)-\widehat B_{\mathrm{pair}}(0)$,
where $\widehat B_{\mathrm{pair}}(r)$ is the perturbed-minus-standard accuracy on items
inserted at dose $r$; $\widetilde B_{\mathrm{pair}}>0$ means contamination inflates measured
accuracy, and the $r=0$ point is the centering reference (zero by construction). Whiskers:
95\% intervals propagated from published SEMs; $^\ast$: interval excludes zero. Result:
strictly monotone from $+0.011$ at $r=1$ to $+0.403^\ast$ at $r=256$; sixteen duplicates
already inflate accuracy by ${\approx}21$ points.}
\label{fig:m2}
\end{figure}

\paragraph{The Hubble suite.} The Hubble suite \citep{hubble2025} is a controlled pretraining
experiment: models are trained from scratch on a web corpus into which benchmark evaluation
documents (MMLU, PIQA, HellaSwag, WinoGrande questions with answers) are deliberately inserted at
\emph{known} duplication counts $r\in\{0,1,4,16,64,256\}$. For each configuration, two models are
released: a \emph{perturbed} model (trained on the corpus with benchmark documents inserted at rate
$r$) and a \emph{standard} model (trained on the same corpus \emph{without} the benchmark
insertions). The standard model is a perfect minimal-pair clean control---it shares the same
architecture, hyperparameters, and training data except for the inserted documents. Model sizes are
1B and 8B parameters, trained for 100B and 500B tokens. We use the published per-(model, task,
duplication rate) accuracies, requiring no additional compute. Headline results use 8B at 500B
tokens, aggregated as an unweighted mean over the four tasks; per-task results are reported as a
robustness check.

\paragraph{Estimation.} Three quantities are reported:
\begin{itemize}
\item \emph{Minimal-pair contrast} (raw):
$\widehat B_{\mathrm{pair}}(r)=\mathrm{acc}_{\mathrm{pert}}(r)-\mathrm{acc}_{\mathrm{std}}(r)$.
\item \emph{Placebo-centered contrast}:
$\widetilde B_{\mathrm{pair}}(r)=\widehat B_{\mathrm{pair}}(r)-\widehat B_{\mathrm{pair}}(0)$,
which removes the small baseline offset between the two separately trained model instances
($\widehat B_{\mathrm{pair}}(0)=-0.022$ for 8B-500B).
\item \emph{Self-baseline estimator} (no separate control):
$\widehat B_{\mathrm{self}}(r)=\mathrm{acc}_{\mathrm{pert}}(r)-\mathrm{acc}_{\mathrm{pert}}(0)$.
\end{itemize}

\paragraph{Trend test.} The monotone dose-response is assessed by an exact one-sided Spearman
permutation test over the nonzero duplication rates ($r\in\{1,4,16,64,256\}$). Robustness checks
(difficulty matching, scale/token settings, per-task consistency) are in \Cref{app:supp-hubble}.

\paragraph{Results.} \Cref{tab:e2} tabulates the three estimators as a function of the known
dose; \Cref{fig:e2} shows the dose-response with propagated intervals, the replication across
scale and token settings, per-task consistency, and the clean-control accuracy range bounding
difficulty imbalance. The placebo-centered contrast is strictly monotone over $r>0$
($+0.011,\,+0.075,\,+0.207,\,+0.380,\,+0.403$; Spearman $\rho=1.0$, $p=0.0083$); the clean
control's accuracy varies by at most $0.051$ across duplication bins; the self-baseline tracks
the minimal-pair contrast within $0.035$ at every positive dose.

\begin{table}[t]\centering
\caption{\textbf{M2 (Hubble, 8B, 500B tokens): three estimators of the dose effect agree.}
Rows: the raw minimal-pair contrast $\widehat B_{\mathrm{pair}}$, the placebo-centered contrast
$\widetilde B_{\mathrm{pair}}$ (the primary estimand, plotted in \Cref{fig:m2}), and the
self-baseline $\widehat B_{\mathrm{self}}$, by duplication dose $r$. Result: all three rise
strictly monotonically over $r>0$ (placebo-centered: Spearman $\rho=1.0$, $p=0.0083$) and agree
within $0.04$ at every positive dose; 95\% intervals are shown in \Cref{fig:e2}(a).}
\label{tab:e2}
\begin{tabular}{lcccccc}
\toprule
duplication $r$ & 0 & 1 & 4 & 16 & 64 & 256 \\
\midrule
$\widehat B_{\mathrm{pair}}$ (raw) & $-0.022$ & $-0.010$ & $+0.053$ & $+0.186$ & $+0.358$ & $+0.381$ \\
$\widetilde B_{\mathrm{pair}}$ (centered) & $+0.000$ & $+0.011$ & $+0.075$ & $+0.207$ & $+0.380$ & $+0.403$ \\
$\widehat B_{\mathrm{self}}$ (self-baseline) & $+0.000$ & $+0.017$ & $+0.065$ & $+0.216$ & $+0.394$ & $+0.413$ \\
\bottomrule
\end{tabular}
\end{table}

\begin{figure}[t]\centering
\figorbox[0.85\linewidth]{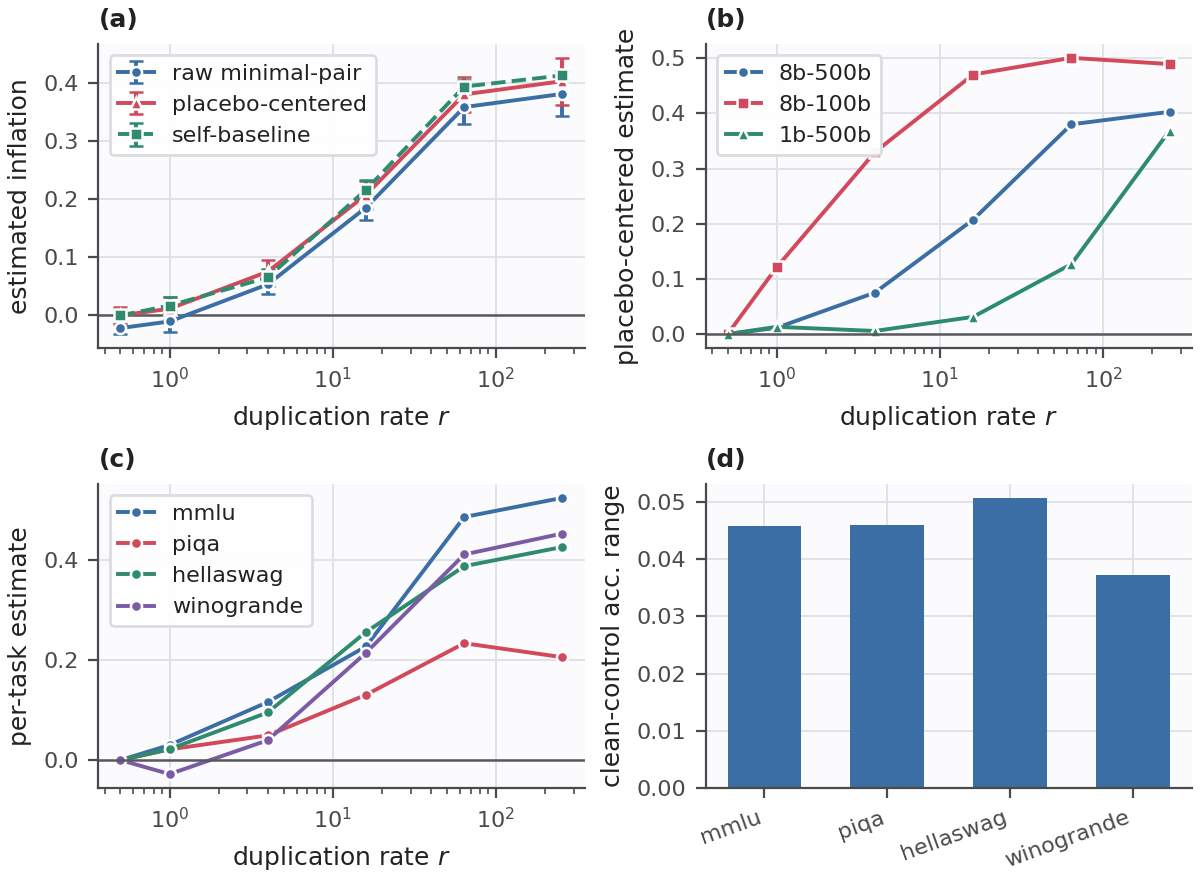}
\caption{\textbf{M2: the Hubble dose-response is robust across estimators, scales, and tasks.}
(a) Raw minimal-pair, placebo-centered, and self-baseline estimates for the 8B/500B pair
(whiskers: 95\% intervals propagated from published SEMs). (b) The placebo-centered estimate
across model scale and token budget. (c) Per-task estimates for 8B/500B. (d) The clean
control's accuracy range across dose bins, bounding difficulty imbalance. Result: inflation
rises monotonically with dose in every panel; no single estimator, scale, or task drives
\Cref{fig:m2}.}
\label{fig:e2}
\end{figure}

\subsection{M3 (forecasting twins with injected leakage)}
\label{app:m3}

\paragraph{Base model and training.} Both twins are LoRA continued-training runs of
Qwen3.5-35B-A3B-Base on the Tinker training API: LoRA rank 32, sequence length 2048, batch 16
sequences, learning rate $2\times10^{-4}$ with 100-step linear warmup, gradient clipping at 1.0,
one epoch over the corpus in a seed-fixed order. Each full twin processes approximately 94
million tokens (2{,}900 steps). The corpus is Wikitext-103 filler (95 million tokens, identical
for both twins) with the injected documents interleaved at deterministic schedule positions.

\paragraph{Injected documents.} For each treated question, dated news-style documents are written
by DeepSeek-V3.1 from the question metadata and the realized outcome, in rotating styles (news
article, encyclopedia entry, analyst note), and validated by DeepSeek-V3.2 (the validator must
answer the question correctly from the treatment document alone, and must \emph{fail} to answer
from the scrubbed control document). The control twin sees the same documents with outcome
statements scrubbed; source, topic, length, and schedule positions are matched. Dose
$r\in\{0,1,4,16,64\}$ counts document copies, randomized across questions stratified by the
control twin's per-question honest error tercile and by resolution month.

\paragraph{Screening.} Injection is restricted to questions resolving before the pseudo-cutoff
$T^\ast$ (April 2026) on which the un-tuned base model is not already confidently correct
(excluding questions with base $P(\mathrm{YES})>0.85$ and outcome YES, or $P(\mathrm{YES})<0.15$
and outcome NO). This directly operationalizes the design requirement that the injected documents
be the marginal information source. The screen keeps $988$ of $1{,}103$ pre-$T^\ast$ questions
and was fixed before any twin was trained; the deviation from the coarser pre-registered proxy
screen is documented in the versioned deviations file.

\paragraph{Probability readout.} Probabilities are read from paired ``Yes''/``No'' continuation
log-probabilities under a fixed, generic, outcome-balanced few-shot prefix, giving
$P(\mathrm{YES})$ as the softmax of the two continuation scores. The same readout is used for the
base model, both pilot twins, and both full twins.

\paragraph{Pilot gates.} Before the full runs, a pilot pair was trained on ten percent of
injected questions with 10 million filler tokens and had to pass five pre-registered gates:
(a) memorization (treatment log-probability advantage on injected documents), (b) signal
(positive twin contrast at dose 64), (c) null (twin contrast consistent with zero at dose 0),
(d) direction (treatment probability moves toward realized outcomes), and (e) recency
(control twin beats the base on injected-topic questions). All five passed at the pilot and again
at full scale, where the values quoted in the main text were computed.

\paragraph{Concentration binning variants.} The binning variable is the control twin's realized
per-question error $b_0$, the quantity appearing in the identity $L=b_0\,w(2-w)$; the main text
displays quintiles, and the conclusion is invariant to granularity. Under the pre-registered
tercile binning at dose 64 the observed contrasts are $+0.019$, $+0.055$, $+0.112$ against
homogeneous-$w$ predictions \rev{(with $\hat w$ fixed at the dose-16 fit)} $+0.013$, $+0.044$, $+0.185$: same pattern, with the shortfall
confined to the top bin. Because binning on realized error could in principle inflate top-bin
estimates by selecting on the minuend of the contrast, we also rebin by ex-ante crowd
uncertainty $c_0(1-c_0)$, which involves no realized outcome. This too preserves the conclusion:
at dose 64 the lower two terciles match the prediction ($+0.068$ observed vs $+0.054$ predicted;
$+0.068$ vs $+0.069$) while the hardest tercile falls short ($+0.053$ observed, CI
$[+0.010,+0.097]$, against $+0.127$ predicted); at dose 16 the hardest ex-ante tercile is null
($-0.001$) against a $+0.125$ prediction. Solving $L/b_0=w(2-w)$ per tercile gives extraction
weights $0.55$, $0.44$, $0.18$ at dose 64 (per quintile: ${\approx}0.5$ on the lower three,
$0.30$, $0.16$ on the top two): extraction declines with question difficulty and rises with
dose. \rev{These per-bin weights are solved from the observed contrasts, so they describe the misses and do not test the law. At dose 64, the second quintile also misses: its contrast ($+0.035$, CI $[+0.031,+0.039]$) is above the prediction ($+0.029$).}

\paragraph{Recovery decomposition and spillover mechanism.} The boundary statistics in the main
text are: treatment-twin naive pre/post gap $+0.053$; control-twin naive gap $+0.001$;
\rev{pre/post jumps within $\pm90$ days of $T^\ast$}, $+0.052$ (treatment) and $+0.012$
(control); placebo boundary at $T^\ast-120$ days $-0.014$. The per-item ground truth, averaged
over the injected pool, is $+0.024$ \rev{($+0.028$ after dose-0 differencing)}. \rev{Part of the gap} between the
boundary estimate and the per-item truth is a base-rate spillover, measured directly: the
treatment twin's mean probability of YES shifts relative to the control twin by $-0.088$ on
injected questions, $-0.096$ on dose-zero questions, and $-0.097$ on never-injected
post-$T^\ast$ questions ($n=543$)---a uniform lean matching the $83\%$ NO rate of the injected
outcomes. The lean is nearly cost-free on pre-$T^\ast$ pools (YES rate $0.170$--$0.183$, hence
the dose-zero null of $-0.003$\rev{, CI $[-0.025,+0.017]$}) and costs $0.028$ Brier on the post pool (YES rate $0.355$).
Dose-zero differencing removes the uniform lean exactly, which is why the law analyses are
unaffected. Concentrated single-style document injection plausibly exaggerates this prior shift
relative to organic contamination; we flag it as a scope note for reading wild boundary
estimates as totals rather than pure pre-side inflation.

\rev{\paragraph{What the jump identifies under spillover.} The twins give both estimands at once:
the counterfactual per-item contrast of \Cref{sec:decomposition} ($+0.024$) and the boundary jump
($+0.052$). Spillover separates them. It acts mostly after the cutoff, where the support restriction of
\Cref{ass:struct} sets leakage to zero. We therefore drop this restriction and derive what the
jump measures.}

\rev{We write the risk of the contaminated model as
$m=m^{\mathrm{c}}-L^{\mathrm{mem}}\,\mathbf{1}[g<0]-L^{\mathrm{sp}}$. Here $m^{\mathrm{c}}$ is the
risk of the same model trained without the leaked documents, and $L^{\mathrm{mem}}\ge0$ is the
gain from memorized outcomes. $L^{\mathrm{sp}}$ is the spillover: the loss reduction on questions
with outcomes that the model did not memorize, on each side of the cutoff. It is negative when
the spillover harms. If $m^{\mathrm{c}}$ is continuous at the cutoff (\Cref{ass:cont} for the
clean model), then
\[
  J(x)=L^{\mathrm{mem}}(x,0^-)+L^{\mathrm{sp}}(x,0^-)-L^{\mathrm{sp}}(x,0^+).
\]
Thus $J$ is the total contamination footprint at the cutoff.}

\rev{Without spillover, $J$ is the boundary leakage of \Cref{thm:rd}. Spillover that is equal on
both sides cancels. Otherwise the spillover term can have either sign. Without contamination, $L^{\mathrm{mem}}=L^{\mathrm{sp}}=0$, so a nonzero jump
still shows contamination. In M3, within $\pm90$ days of $T^\ast$, the jump of $+0.052$ splits exactly into three terms: memorization $+0.025$ (CI $[-0.009,+0.064]$), the spillover change $+0.015$ (CI $[-0.032,+0.057]$), and the jump of the control twin $+0.012$ (CI $[-0.022,+0.045]$). The last term is the pre/post difference of the control twin. It combines sampling noise, drift within the window, and any recency that the scrubbed documents add before $T^\ast$. The identity assumes that it is zero. No term is significant alone. We estimate the pre-cutoff spillover from dose-zero questions, so the memorization term is the injected share times the injected-minus-dose-zero gain. Over all questions on each side, the lean costs $0.003$ before $T^\ast$ and $0.028$ after it, so the spillover term is positive in point estimate. Here $m_0$ in \Cref{ass:struct} is the risk of the clean model, and spillover after the cutoff then violates the support restriction. Under diffuse organic contamination, the lean and its post-cutoff cost are
probably smaller.}

\paragraph{Attenuation mechanisms the twins cannot exhibit.} The spillover above biases a
wild boundary jump \emph{upward} \rev{in M3}; two mechanisms present in wild deployments bias it
\emph{downward}, and injection that is sharp at $T^\ast$ produces neither. First,
quasi-resolved outcomes: a question resolving shortly after a cutoff can be effectively
decided in pre-cutoff text (an election with decisive polls, a lawsuit after closing
arguments). The support restriction of \Cref{ass:struct} classifies such information as
recency, so this is not a violation, but it raises the honest post-boundary score
$m(x,0^+)$ and shrinks the observed jump. Second, ingestion lag: text describing outcomes
realized just before the cutoff has had little time to be crawled and trained on, thinning
$L(x,g)$ as $g\to0^-$ and smoothing the left limit. Both mechanisms attenuate $\widehat J$
toward zero, and neither is covered by the cutoff-placement sensitivity checks, which
perturb the boundary's location rather than the information density around it. The reading
rules of \Cref{sec:wild} therefore treat \rev{a starred jump as a detection}: a starred jump survives attenuation, while a boundary null does not
certify the absence of leakage sitting deeper in the pre-cutoff period.

\paragraph{PRC protocol.} Three paraphrases per question are generated by DeepSeek-V3.1 and
validated by DeepSeek-V3.2 (semantic equivalence and preserved resolution criteria); $971$ of
$988$ questions retain all three. Both twins score all paraphrases with the same readout.
Per-question consensus probabilities are temperature-calibrated per twin on dose-0 anchors, and
the reported statistic is the twin-differenced residual covariance with paired question bootstrap
($4{,}000$ resamples): $+0.0069^\ast$ (CI $[+0.0035,+0.0107]$) on injected questions, null at
dose zero ($-0.0057$, CI $[-0.0128,+0.0016]$), and monotone in dose with the dose-64 cell
starred ($+0.0121^\ast$). Raw, undifferenced PRC is negative in every cell (treatment/control
$\times$ injected/dose-0), reproducing the miscalibration failure predicted by \Cref{rem:calib}.
\Cref{fig:m3prc} shows the calibrated twin-differenced statistic by cell.

\begin{figure}[t]\centering
\figorbox[0.62\linewidth]{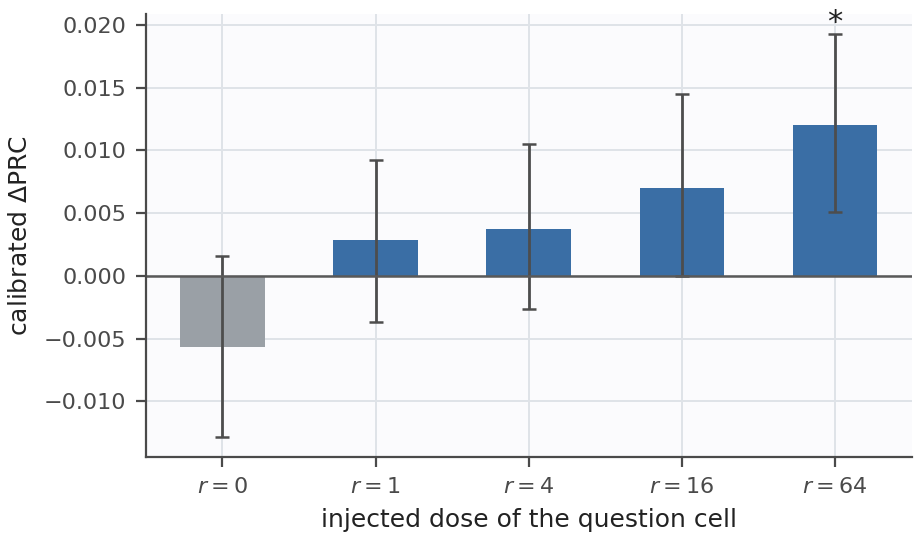}
\caption{\textbf{M3: the calibrated, twin-differenced PRC fires only where leakage was
injected.} Bars: the calibrated twin-differenced PRC per injection cell (dose zero and each
positive dose), computed from three validated paraphrases per question; positive values signal
evidence leakage. Whiskers: paired-bootstrap 95\% CIs; $^\ast$: CI excludes zero. Result: null
at dose zero, monotone in dose, starred at dose 64 ($+0.0121^\ast$); the raw statistic without
calibration or differencing is negative everywhere, the miscalibration failure
\Cref{rem:calib} predicts.}
\label{fig:m3prc}
\end{figure}

\subsection{M4-F (forecasting-panel matrix and clean-anchor arm)}
\label{app:m4}

\paragraph{Specificity matrix.} Targets are DeepSeek-V3.1 (documented cutoff March 2025),
Kimi-K2.6 (April 2025), GPT-5.4 (August 2025), and GPT-5.5 (December 2025); Gemini-3.1-Pro
(January 2025) is exploratory because the panel has thin mass before its boundary. Controls are
MiniMax-M3 and Claude-Opus-4.7, whose January 2026 cutoffs leave no leakage discontinuity \rev{at
the tested boundaries}. For each target and each assumed boundary, the statistic is the pre/post jump
(90-day windows) in the per-question target-minus-control score difference, using the mean of the
available controls on each question. Confidence intervals use the source$\times$month cluster
bootstrap. Joint inference over the matrix uses a permutation test that reassigns cutoff dates to
models and compares the observed diagonal-minus-off-diagonal contrast with its permutation
distribution. \rev{One off-diagonal cell in the top block is starred (DeepSeek-V3.1 at Aug'25, $-0.065$). A
negative jump is opposite to the leakage sign.}

\paragraph{Exploratory Gemini row.} Gemini-3.1-Pro's January 2025 cutoff precedes every tested
boundary, so all of its cells are placebos: $+0.141^\ast$ (Mar'25), $+0.042$ (Apr'25), $-0.050$
(Aug'25), $+0.026$ (Dec'25). The one starred value sits at the boundary with the thinnest
pre-boundary mass, which makes the estimate unstable, and it is not at any documented cutoff;
we do not read it as a detection.

\paragraph{Cutoff-resolution robustness for Kimi-K2.6.} Kimi-K2.6's coarser (tier-two) cutoff
documentation is not load-bearing: its cells at the adjacent March and April boundaries
($+0.034$, $+0.022$; \Cref{tab:m4f}) are both unstarred, so its null survives a one-month
cutoff shift.

\paragraph{Clean-anchor arm.} ForecastBench archives every real-time forecast submitted to the
benchmark before question resolution. For a target with cutoff $T$, we select \emph{family
relatives} (same vendor, earlier release) whose archived real-time forecasts cover questions
resolving on both sides of $T$: GPT-5-Mini and GPT-5.1 for GPT-5.5; GPT-4.1 and o4-mini for
GPT-5.4. On each covered question the statistic is our retrospective target score minus the
relative's archived real-time score; the crowd anchor cancels in this difference. The estimand is
the jump of this difference at $T$. \rev{The pre-specified unbanded variant is primary, with the $\pm90$-day window of \Cref{tab:m4f}.
Archived forecasts for questions that resolve late can come from more recent submission rounds,
with shorter horizons. Thus we also report a variant that restricts anchor horizons to 10 to 75
days, which balances the mean horizons before and after the cutoff. We adopted this band after we
saw the horizon imbalance. We logged it as a dated deviation in the versioned deviations file
before we read any banded estimate.}

\rev{Two protocol checks use models that also submitted real-time forecasts. For GPT-5 and
DeepSeek-V3.1, our retrospective forecasts exceed the archived real-time forecasts by $+0.029$
(CI $[+0.010,+0.049]$, $n=251$) and $+0.021$ (CI $[+0.010,+0.032]$, $n=908$) in mean probability.
We flag these shifts as protocol sensitivity. The leakage estimand is a \emph{jump}, so a constant
protocol shift cancels unless it changes in time across the window.}

\rev{The placebo-jump check tests this directly on the score difference between the
retrospective and real-time forecasts of DeepSeek-V3.1 ($n=908$, mean level $-0.016$ in anchored-Brier units, $\pm90$-day windows). At the December
2025 boundary, the jump is $+0.031$ (CI $[-0.019,+0.328]$,
$n_{\mathrm{pre}}/n_{\mathrm{post}}=39/327$), which is null. With $\pm120$-day windows, this jump is also null ($+0.026$, CI $[-0.029,+0.120]$).}

\rev{Two other boundaries give starred jumps at $\pm90$ days. One is negative, $-0.039$ (March 2026), about half the primary GPT-5.5 signal in size. The other is $+0.045$ at November 2025, from only 12
pre-boundary questions. At $\pm120$ days, April 2026 ($-0.027$) is also starred.}

\rev{Matrix power: we give the null diagonal cells of the specificity matrix power floors ($2.8$ times the cluster-bootstrap SE). The floors are $0.077$ (DeepSeek-V3.1), $0.111$ (Kimi-K2.6),
$0.098$ (GPT-5.4), and $0.053$ (GPT-5.5) anchored-Brier units. We compute them from the
cluster-bootstrap standard errors of the own-cutoff cells.}

\paragraph{Reading disagreement between the two designs.}
For the same model and boundary (GPT-5.5 at December 2025), the matrix reports $+0.034$
(unstarred) and the anchor arm \rev{$+0.078^\ast$}. The designs differ in the two respects that
matter. Their baselines differ in cleanliness: the control models are guaranteed only to
have no \emph{discontinuity} at December 2025---their January 2026 cutoffs let them memorize
the same pre-boundary outcomes, and any inflation they carry is subtracted from the target's
cell\rev{. Their cutoffs fall inside the post-cutoff window, so they can carry leakage on all pre-cutoff questions but on only part of the post-cutoff questions}---whereas an archived real-time forecast cannot contain leakage at all. And the paired
per-question difference removes the cross-vendor cluster variance that sets the matrix's
power floor (\rev{$0.053$} at this cell). \rev{The first force pulls the matrix cell toward zero, and the second widens its CI}, so the
pattern is attenuation plus lower power, not contradiction. The auditor's rule: a starred
own-cutoff jump that survives placebo dates is a detection scoped to its design's baseline; a
null matrix cell bounds the effect but certifies nothing. Taken together, the matrix \rev{has power against}
leakage among 2025-cutoff flagships \rev{above its power floors ($0.053$ to $0.111$ anchored Brier)}, and the within-family
route detects a leakage signature for GPT-5.5 exactly at its documented cutoff.

\subsection{M4-C (LiveCodeBench code arm)}
\label{app:e3}

\paragraph{LiveCodeBench.} LiveCodeBench \citep{livecodebench2025} is a coding benchmark that
collects problems from competitive programming contests (LeetCode, Codeforces, AtCoder). Each
problem carries its \emph{contest release date}, so a model can only have trained on a problem's
solution if the contest occurred before the model's training cutoff. The project publishes official
per-problem pass@1 (fraction of sampled code completions passing all unit tests) for many models,
and refreshes regularly with new contests to stay contamination-free. We use problems from May 2023
through April 2025 ($713$--$1055$ contest-dated problems per model, filtered to those present in all
evaluated models). All scores are taken from the public submissions repository; no additional
inference is run.

\paragraph{Models and cutoffs.} The \emph{continuity targets}, evaluated entirely from published
per-problem data, are GPT-4o-2024-08-06 (cutoff $\approx$Oct 2023) and Claude-3.5-Sonnet-20240620
(cutoff $\approx$Apr 2024). Composition controls have cutoffs \emph{after} the latest problem
(Gemini-2.5-Pro, DeepSeek-R1, and the published 2025-cutoff pool). Because control cutoffs lie
outside the evaluation window, they have no leakage discontinuity at any target's cutoff
(\Cref{prop:did-boundary}); their pre/post gap at any evaluation-window date captures only
composition drift (later contests are harder). Cutoff dates are documented approximations at
month resolution.

\paragraph{Self-generated extension.} To cover 2024 training cutoffs with current models, we
additionally generate one greedy completion per problem through OpenRouter for four targets with
documented 2024 cutoffs---GPT-5 (Sep 2024), GPT-5-Mini (May 2024), DeepSeek-V3-0324 (Jul 2024),
and Kimi-K2-0905 (Dec 2024)---and two clean controls with 2025 cutoffs (DeepSeek-V3.2 and
MiniMax-M1), and grade locally against the official unit tests with the
official checker semantics (stdin and functional harnesses, per-problem time limits). The local
grader was validated by regrading published DeepSeek-V3 completions: verdict agreement is
$98.3\%$, within the pre-registered three-point tolerance. The same $\pm160$-day window DiD and
date-permutation machinery as the continuity arm is applied at the targets' documented cutoffs.

\paragraph{DiD estimation.} For each target model and candidate cutoff date $T$, define a symmetric
window $[T{-}h,\ T{+}h]$ days (default half-width $h=160$, giving $n_{\mathrm{pre}}\approx207$ and
$n_{\mathrm{post}}\approx194$ problems for GPT-4o at Oct 2023). Split problems into pre ($g<0$) and
post ($g\ge0$). The within-model gap is
$\mathrm{gap}^{m}=\overline{\mathrm{pass@1}}^{m}_{\mathrm{pre}}-\overline{\mathrm{pass@1}}^{m}_{\mathrm{post}}$.
The pooled composition control's gap is the same quantity computed on the average pass@1 of the two
2025-cutoff models. The boundary-leakage estimate is
$\widehat J=\mathrm{gap}^{M}-\mathrm{gap}^{M_0}$: the target's excess gap beyond composition drift.
Bootstrap CIs: $4000$ resamples of problem IDs (resampled jointly for target and control), 95\%
percentile interval. The robustness battery (control smoothness, separate controls, date
specificity, bandwidth, difficulty stratification) is in \Cref{app:supp-e3}; the global-$B$
sensitivity analyses are in \Cref{app:supp-e3global}.

\paragraph{Results of the self-generated extension.} The 2024-cohort targets do not reproduce the
GPT-4o wild positive. At their documented cutoffs, GPT-5-Mini ($+0.005$, CI $[-0.058,+0.072]$) and
Kimi-K2-0905 ($-0.002$, CI $[-0.078,+0.075]$) are null. DeepSeek-V3-0324 is marginally positive
($+0.051$, CI $[-0.013,+0.115]$; date-permutation $p=0.050$), and its only starred cells sit at
the April and May 2024 placebo dates ($+0.079$, $+0.103$) rather than at its own July 2024
cutoff---a pattern consistent with diffuse contamination of spring-2024 contest solutions rather
than a sharp training-cutoff boundary, so we do not claim a detection. GPT-5 is significantly
\emph{negative} ($-0.119$, CI $[-0.185,-0.051]$): the 2025-cutoff control pool improves more on
post-September-2024 problems than GPT-5 does, a violation of the no-differential-trend condition
of \Cref{prop:did-boundary} for reasoning-era model pairs, and a reminder that the DiD sign test
is only interpretable when the control-trend assumption holds. We therefore report the extension
as bounded nulls with one flagged assumption failure, and rest the code-domain detection claim on
the published-data continuity targets above.

\subsection{M5 (disciplined null on forecasting)}
\label{app:e4}\label{app:m5}

\paragraph{Program reading (demoted from the main-text close).}
Each claim has one primary experiment whose headline result is that claim
(\Cref{tab:contract}): M1 establishes the confound on cutoff-clean models; M2 anchors dose
sensitivity in randomized controlled pretraining; M3 supplies ground truth for the estimators
and the laws; M4 shows what honest wild detection looks like, starred where the evidence
supports it and null elsewhere; and M5 shows the discipline, with the weak-control failure
mode demonstrated rather than assumed and power analysis quantifying what the nulls exclude.

\paragraph{Data.} The shared forecasting panel of \Cref{app:m1}, restricted to questions
resolving between January 2025 and July 2026: $265$ pre-boundary and $1{,}207$ post-boundary
questions, evaluated pairwise (every model on every question).

\paragraph{Boundary.} The primary target Qwen3.5-35B-A3B has only a year-level documented cutoff
(2026), so we place the boundary at July 2025, inside its plausible training window, and report a
boundary sweep from $-180$ to $+180$ days around it in 30-day steps. The adjusted estimate is
null at every placement.

\paragraph{Control selection.} The control is chosen \emph{before} any DiD is computed, by a
pre-registered profile-matching protocol: on post-boundary questions that are clean for every
candidate, compute each candidate's distance to the target (squared-error distance between
calibration profiles plus squared-error distance between per-source mean scores). Between the two
flagships with post-window-clean cutoffs, GPT-5 (distance
$0.0051$) is selected over Gemini-3.1-Pro ($0.0088$). GPT-3.5-Turbo (cutoff September 2021)
serves as the deliberately mismatched weak control. Secondary targets are Kimi-K2.6,
DeepSeek-V3.2, GLM-4.7, and MiniMax-M3, all queried retrospectively through the shared protocol.

\paragraph{Paired DiD estimation.} For each question $i$, form the paired difference
$D_i=L_i^{\mathrm{target}}-L_i^{\mathrm{control}}$ of crowd-anchored Brier reductions. The
adjusted estimate is the pre/post difference in mean $D_i$, with source$\times$month
cluster-bootstrap confidence intervals ($4{,}000$ resamples). \rev{We report each null with its power floor, $2.8$ times the cluster-bootstrap standard error of its estimate.} The boundary
sweep and power floors are in \Cref{app:supp-e4}.

\paragraph{Results.} \Cref{tab:m5} reports the per-model naive gaps and adjusted DiD estimates
summarized in \Cref{sec:exp:m5} and plotted in \Cref{fig:m5dissolve}.

\begin{table}[t]\centering
\caption{\textbf{M5: full estimates behind \Cref{fig:m5dissolve}.} Column 2: the naive pre/post
gap $\Delta$ at the July 2025 boundary; positive reads as leakage. Column 3: the paired DiD
against the pre-registered profile-matched control (GPT-5); for a leak-free model the ideal
value is zero. Cluster-bootstrap 95\% CIs; $^\ast$: CI excludes zero. Result: no adjusted
estimate is significantly positive \rev{(the two significant negatives are consistent with overcorrection by the control, which biases toward a null)}; the weak-control row substitutes GPT-3.5-Turbo and re-inflates the
estimate, the control-mismatch failure mode of \Cref{ass:transport}.}
\label{tab:m5}
\begin{tabular}{lcc}
\toprule
model & naive gap [CI] & adjusted DiD [CI] \\
\midrule
Qwen3.5-35B-A3B (target) & $+0.060\ [-0.008,+0.108]$ & $+0.020\ [-0.063,+0.074]$ \\
Kimi-K2.6                & $+0.024\ [-0.008,+0.055]$ & $-0.016\ [-0.058,+0.010]$ \\
DeepSeek-V3.2            & $+0.006\ [-0.021,+0.027]$ & $-0.034^\ast\ [-0.088,-0.000]$ \\
GLM-4.7                  & $+0.009\ [-0.025,+0.042]$ & $-0.031^\ast\ [-0.076,-0.002]$ \\
MiniMax-M3               & $+0.039^\ast\ [+0.008,+0.074]$ & $-0.000\ [-0.039,+0.026]$ \\
\midrule
weak control (GPT-3.5-Turbo) & --- & $+0.055\ [-0.028,+0.151]$ \\
\bottomrule
\end{tabular}
\end{table}

\paragraph{Provenance of the fitted overconfidence temperature.} The value $T\approx8$ cited in
\Cref{rem:calib} and used to set the T4 sweep range ($T_{\mathrm{over}}\in\{2,4,8\}$) was fit on
this forecasting pipeline's data: temperature scaling
$p_{\mathrm{cal}}=\sigma(\mathrm{logit}(p)/T)$ with $T$ chosen to minimize log-loss on Qwen3.5's
\emph{leakage-free} post-boundary control forecasts ($n=168$ ex-ante-uncertain factual questions,
crowd anchor frozen in $[0.25,0.75]$). The fit reached the upper end of the search interval
$(0.3,8]$, so $T\approx8$ is best read as a lower bound on the distortion: Qwen's raw probabilities
are severely overconfident, which is why the raw residual covariance is negative and PRC requires
recalibration before use.

\section{Complementary experiments and robustness analyses}
\label{app:supp}

The synthetic validation study E1 and its two supplementary arms come first; the remaining
subsections report one robustness analysis each, in a fixed format: the \emph{intuition} (what
question it answers and which claim it defends), the \emph{setup} (a short description or a
pointer to \Cref{app:experiments} where shared), and the \emph{results}.

\subsection{E1 (synthetic validation)}
\label{app:e1}

This synthetic study validates the estimators under known data-generating assumptions. It
formerly opened the main-text experiment section and is retained here in full; the real-data
ground-truth role is now carried by M3. All sub-tests share a common data-generation process,
described first, with per-test variations noted below; results follow at the end of this
subsection, and the supplementary arms T5--T6 follow in \Cref{app:supp-t5,app:supp-t6}.

\paragraph{Common data generation.}
For each question $i$, draw a crowd anchor $c_{0,i}\sim\mathrm{Uniform}(0.15,0.85)$ and a binary
outcome $Y_i\sim\mathrm{Bernoulli}(c_{0,i})$. The \emph{honest forecast} is
$P_{\mathrm{hon},i}=\mathrm{clip}\!\big(c_{0,i}+s\,(Y_i-c_{0,i})+\epsilon_i,\;[10^{-3},1{-}10^{-3}]\big)$,
where $s\in[0,1]$ is a skill parameter controlling how much the model improves on the crowd, and
$\epsilon_i\sim\mathcal{N}(0,0.05)$ is idiosyncratic noise. For a leaked question with extraction
weight $w\in[0,1]$, the observed forecast is the convex blend
$P_i=\mathrm{clip}\!\big((1-w)\,P_{\mathrm{hon},i}+w\,Y_i\big)$; for a clean question, $P_i=P_{\mathrm{hon},i}$.
An optional overconfidence parameter $T_{\mathrm{over}}\ge1$ transforms $P_i$ via
$P_{\mathrm{obs},i}=\sigma(\mathrm{logit}(P_i)\cdot T_{\mathrm{over}})$, where $\sigma$ is the
logistic function; $T_{\mathrm{over}}=1$ leaves the forecast unchanged.

\paragraph{T1 (recovery and coverage).}
For each extraction weight $w\in\{0,0.2,0.4,0.7,1.0\}$, generate $n=400$ leaked questions and
$n=400$ clean-control questions (both with skill $s=0.3$, no overconfidence). The true inflation is
$B_{\mathrm{true}}=\tfrac1n\sum_i\big[(P_{\mathrm{hon},i}-Y_i)^2-(P_i-Y_i)^2\big]$, computed from
the known honest forecast. The estimated inflation is the Brier-reduction DiD:
$\widehat B=\bar S_{\mathrm{leak}}-\bar S_{\mathrm{clean}}$, where $\bar S=\tfrac1n\sum(c_0-Y)^2-(P-Y)^2$.
Bootstrap CIs: 200 problem-level resamples, 95\% percentile interval. We run 300 Monte Carlo
repetitions per $w$ and report the mean bias $\widehat B-B_{\mathrm{true}}$ and the fraction of
repetitions in which the CI covers $B_{\mathrm{true}}$.

\paragraph{T2 (RD recovery and placebo).}
Each question $i$ has a resolution-time gap $\delta_i\sim\mathrm{Uniform}(-180,180)$ (days to the
cutoff). The honest skill varies smoothly with recency:
$s(\delta)=0.4-0.25\,|\delta|/180$, so the model is better near its cutoff (pure recency, no
leakage). Leakage of weight $w_0\in\{0,0.3,0.6\}$ is injected only on the leaked side
($\delta<0$). The RD jump is estimated by local-linear regression with a triangular kernel and
half-bandwidth $h=120$ days: fit $\ell_i=(P_i-Y_i)^2\approx\beta_0+\beta_1\,\mathbf{1}[\delta<0]+
\beta_2\,\delta+\beta_3\,\delta\cdot\mathbf{1}[\delta<0]$ on the weighted sample $|\delta|\le h$;
the jump estimate is $-\hat\beta_1$. The \emph{placebo} repeats the same procedure with no leakage
injected ($w_0=0$ forced, though recency is still present). $n=1500$ per run, 300 repetitions.

\paragraph{T3 (complementarity).}
For each extraction weight $w\in\{0,0.25,0.5,0.75,1.0\}$, generate $n=8000$ questions. For each
question, query $K=12$ independent ``paraphrases'' (simulated by adding independent noise
$e_k\sim\mathcal{N}(0,0.12)$ to the crowd anchor before blending with $w$). The consensus is
$\bar y_i=K^{-1}\sum_k P_i^{(k)}$. Compute the PRC as
$\widehat\Delta=\widehat{\mathrm{Cov}}(\bar y,\,Y-\bar y)$ and the leakage score
$L=\tfrac1n\sum\big[(c_0-Y)^2-(\bar y-Y)^2\big]$.

\paragraph{T4 (calibration under overconfidence).}
This sub-test isolates the calibration step. The honest forecast is set to the crowd anchor itself,
$P_{\mathrm{hon}}=c_0$, with $Y\sim\mathrm{Bernoulli}(c_0)$, so the clean forecaster is calibrated
by construction. Leaked forecasts are $P=(1-w)\,P_{\mathrm{hon}}+w\,Y$ with $w=0.8$. Reported
probabilities are then distorted by overconfidence:
$P_{\mathrm{obs}}=\sigma(T_{\mathrm{over}}\,\mathrm{logit}(P))$, for
$T_{\mathrm{over}}\in\{1,2,4,8\}$. For each repetition, we draw $n=600$ leaked and $n=600$ clean
evaluation questions, plus a \emph{separate} leakage-free calibration set of size
$n_{\mathrm{cal}}\in\{100,250,500,1000,2000,5000\}$. Temperature scaling ($\hat T$ minimizing
log-loss) is fit \emph{only} on the calibration set and then applied to both evaluation sets before
computing DiD. The \emph{naive} estimator skips calibration and uses $P_{\mathrm{obs}}$ directly.
We report mean bias, mean absolute bias, and RMSE over 200 repetitions per
$(T_{\mathrm{over}},\,n_{\mathrm{cal}})$ cell.

\paragraph{T5 (no free inflation; results in \Cref{app:supp-t5}).}
Honest forecasts are generated as in T1 with $w=0$ (zero true leakage). The ``perturbed'' side
replaces $P_{\mathrm{hon}}$ by $\sigma(\mathrm{logit}(P_{\mathrm{hon}})+\eta)$ with
$\eta\sim\mathcal N(0,\sigma_\eta^2)$ drawn independently of $Y$; logit-space noise avoids the
outcome dependence that clipping would induce at the probability boundaries. Noise scales
$\sigma_\eta\in\{0.25,0.5,1.0\}$; $n=400$ per side; the same Brier-reduction DiD as T1; 300 Monte
Carlo repetitions per scale.

\paragraph{T6 (concentration across surprise strata; results in \Cref{app:supp-t6}).}
Fixed extraction weight $w=0.5$, $n=4000$ leaked and $4000$ clean questions per repetition, 300
repetitions. Questions are stratified into quartiles of the \emph{observable} crowd surprise
$(c_0-Y)^2$ (quartile cuts computed on the pooled sample). Within each stratum, the stratified DiD
is the mean Brier reduction of the leaked group minus that of the clean group; the prediction is
the stratum mean of $b_0\,w(2-w)$ with $b_0=(P_{\mathrm{hon}}-Y)^2$, per \Cref{prop:perq}.

\paragraph{Results (T1--T4; \Cref{fig:e1}).}
\emph{T1 (DiD recovery and coverage).} Across five leakage weights
($w\in\{0,0.2,0.4,0.7,1.0\}$, true $B\in\{0,0.038,0.067,0.095,0.105\}$), the DiD estimator is
essentially unbiased (mean bias ${\approx}9\times10^{-4}$ at every level) and the 95\% bootstrap
CIs cover the true $B$ in $94$--$100\%$ of repetitions; the no-leak case returns a null.
\emph{T2 (RD recovery and placebo).} With a smooth recency surface and leakage injected only on
the pre-cutoff side, the RD jump tracks the injection ($+0.000$, $+0.039$, $+0.065$ for
$w_0=0,0.3,0.6$) while the no-leakage placebo jump is ${\approx}0.000$ at every level: continuous
recency alone does not fabricate a discontinuity (\Cref{thm:rd}).
\emph{T3 (RD/PRC complementarity).} Sweeping $w$ confirms \Cref{prop:connection}: PRC is
hump-shaped (peak $+0.052$ near $w{=}0.5$, vanishing at $w{=}1$) while the RD signal grows
monotonically to $+0.209$ at $w{=}1$; recovered ratios $L(w)/L(1)$ match $w(2-w)$
(\Cref{tab:funcform}).
\emph{T4 (calibration under overconfidence).} Distorting a calibrated clean forecaster by
$T_{\mathrm{over}}\in\{2,4,8\}$ produces naive DiD bias $0.022$, $0.055$, $0.085$; temperature
scaling fit on $n_{\mathrm{cal}}$ leakage-free anchors reduces it to ${\approx}0.008$ at
$n_{\mathrm{cal}}=100$ and ${\approx}0.001$ at $n_{\mathrm{cal}}=5000$. Calibration works but is
a real resource requirement (\Cref{rem:calib}).

\begin{figure}[t]\centering
\figorbox[0.85\linewidth]{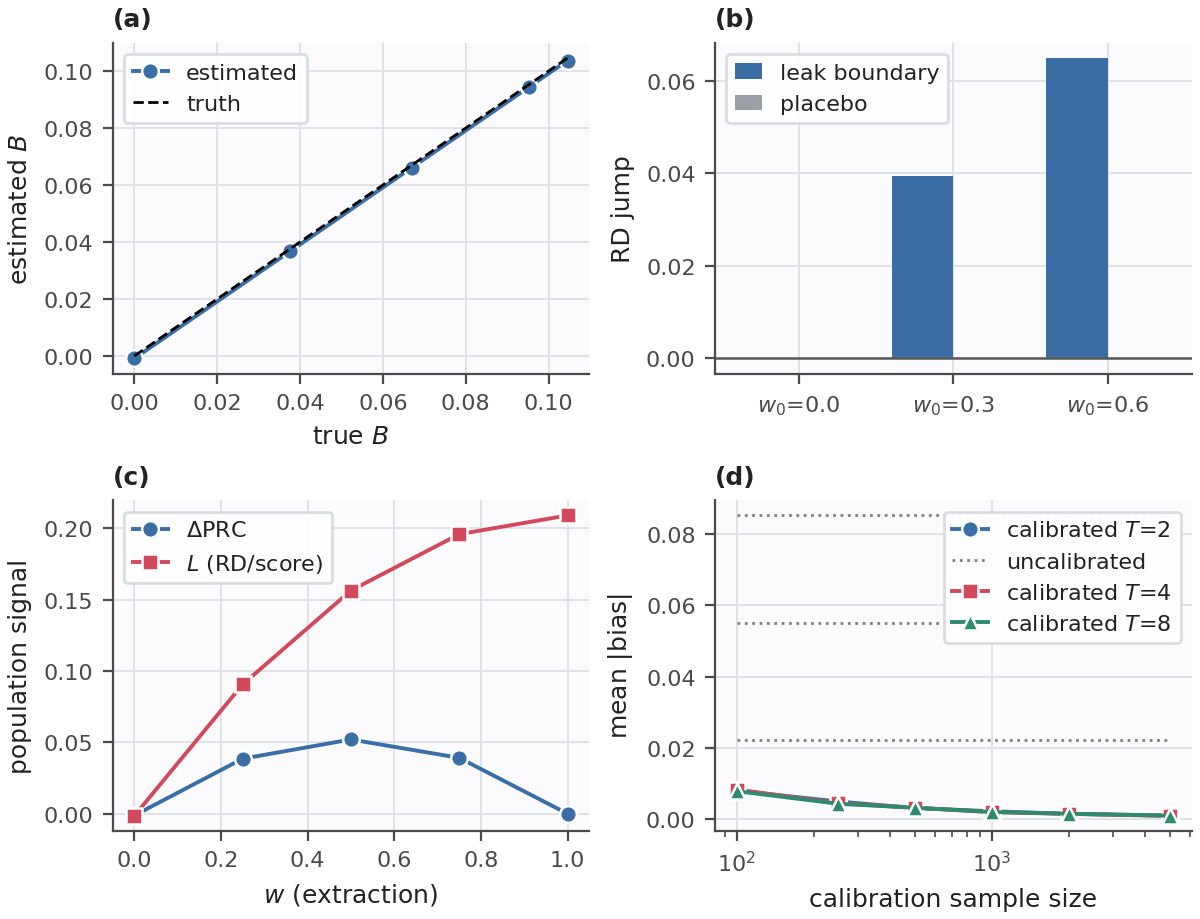}
\caption{\textbf{E1: the estimators recover known ground truth in simulation.}
(a) DiD estimates lie on the truth line (bootstrap coverage $94$--$100\%$ across leakage
levels). (b) RD jumps track injected boundary leakage; no-leak placebos stay at zero under
smooth recency. (c) PRC is hump-shaped and vanishes at pure answer leakage ($w{=}1$) while the
score-based signal grows monotonically---the two detectors are complements. (d) Overconfidence
biases the naive statistic; temperature calibration on leakage-free anchors removes the bias as
$n_{\mathrm{cal}}$ grows. Result: each estimator behaves as
\Cref{prop:did,thm:rd,prop:efficiency,rem:calib} predict under known data-generating
assumptions.}
\label{fig:e1}
\end{figure}

\subsection{No free inflation under outcome-orthogonal noise (E1 arm T5)}
\label{app:supp-t5}

\paragraph{Intuition.} \Cref{cor:no-free} states that a predictor can raise its score only through
a positive loading on the leaked signal: perturbations that are independent of the outcome cannot
inflate. A useful falsification check on the estimator is therefore whether it can be fooled into
booking outcome-orthogonal noise as leakage---if it could, positive estimates in M2--M4 would be
uninformative.

\paragraph{Setup.} As defined in \Cref{app:e1} (T5): honest forecasts with zero true leakage,
perturbed by zero-mean logit-space noise independent of $Y$ at scales
$\sigma_\eta\in\{0.25,0.5,1.0\}$; the same Brier-reduction DiD estimator as T1; 300 Monte Carlo
repetitions per scale.

\paragraph{Results.} The estimate is deflationary at every scale, increasingly so as noise grows:
mean $\widehat B=-0.003$ (95\% CI of the mean $[-0.004,-0.002]$) at $\sigma_\eta=0.25$;
$-0.012$ ($[-0.013,-0.011]$) at $\sigma_\eta=0.5$; and $-0.044$ ($[-0.045,-0.043]$) at
$\sigma_\eta=1.0$. Individual repetitions can come out slightly positive under sampling noise at
the smallest scale ($34\%$ of repetitions; $5\%$ at $\sigma_\eta=0.5$; none at $\sigma_\eta=1.0$),
but the mean effect is strictly negative throughout, matching \Cref{cor:no-free}: noise costs
Brier score, it cannot fake inflation (\Cref{fig:e1-supp}a).

\subsection{Leakage concentration across surprise strata (E1 arm T6)}
\label{app:supp-t6}

\paragraph{Intuition.} The per-question law $L=b_0\,w(2-w)$ (\Cref{prop:perq}) makes two testable
predictions beyond monotonicity in $w$: at fixed extraction, leakage should \emph{concentrate} on
questions where the crowd is surprised (large $b_0$), and across extraction levels the total
leakage should follow the $w(2-w)$ shape exactly. This licenses the main-text statement that
leakage is heterogeneous and concentrated rather than a uniform rate.

\paragraph{Setup.} As defined in \Cref{app:e1} (T6): fixed $w=0.5$, quartiles of observable crowd
surprise $(c_0-Y)^2$, stratified DiD against the stratum-mean prediction $\E[b_0\,w(2-w)]$; 300
repetitions. The functional-form check reuses the T3 sweep (\Cref{app:e1}), comparing recovered
leakage ratios $L(w)/L(1)$ with the predicted $w(2-w)$.

\paragraph{Results.} The stratified DiD estimate rises monotonically and steeply across surprise
quartiles ($0.018$, $0.041$, $0.083$, $0.174$ for Q1 through Q4) and matches the predicted stratum
means ($0.018$, $0.041$, $0.083$, $0.174$) to within Monte Carlo error (\Cref{fig:e1-supp}b): the
top quartile carries roughly ten times the leakage of the bottom quartile at the same extraction
weight. The functional form is confirmed across the full sweep (\Cref{tab:funcform}): observed
ratios track $w(2-w)$ to within $0.005$ at every $w$. A caveat applies to both arms: the
synthetic generator implements the convex blend of \Cref{ass:pull}, so these checks establish
that the estimators recover the law when it holds---they are consistency checks of the
machinery, not tests of the convex-pull assumption itself, whose real-data validation remains
open (\Cref{sec:discussion}).

\begin{table}[h]\centering
\caption{\textbf{E1: the recovered leakage profile matches $w(2-w)$.} Leakage ratios
$L(w)/L(1)$ from the synthetic extraction-weight sweep (arm T3) against the prediction of
\Cref{prop:perq}. Result: observed and predicted agree within $0.007$ at every $w$.}
\label{tab:funcform}
\begin{tabular}{lccccc}
\toprule
extraction weight $w$ & 0 & 0.25 & 0.5 & 0.75 & 1.0 \\
\midrule
observed $L(w)/L(1)$ & $-0.007$ & $0.433$ & $0.748$ & $0.937$ & $1.000$ \\
predicted $w(2-w)$   & $0.000$  & $0.438$ & $0.750$ & $0.938$ & $1.000$ \\
\bottomrule
\end{tabular}
\end{table}

\begin{figure}[t]\centering
\figorbox[0.85\linewidth]{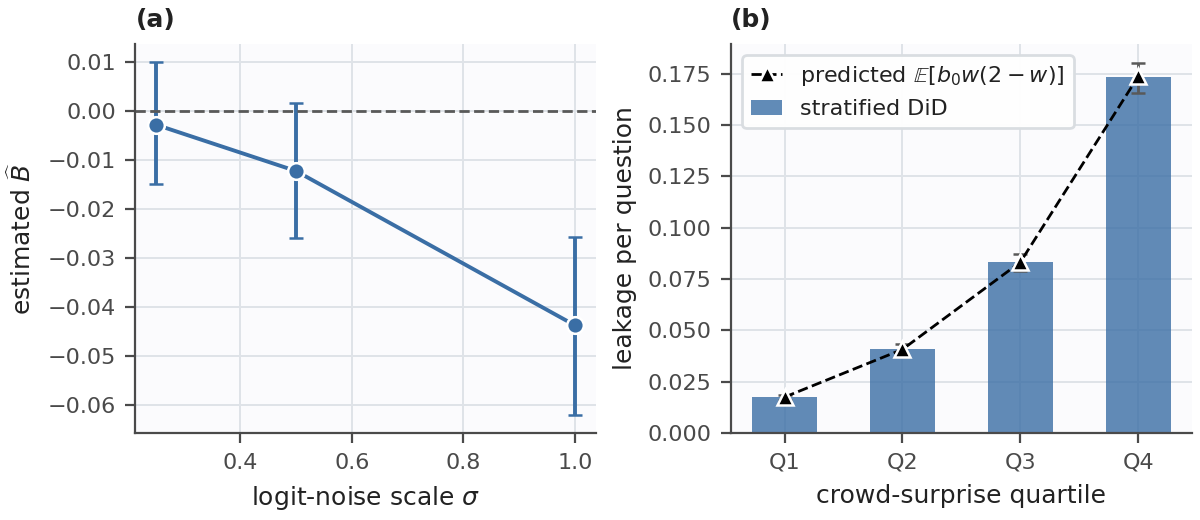}
\caption{\textbf{E1 supplementary arms: noise cannot fake inflation, and leakage concentrates
where the law predicts.} (a) T5: under outcome-orthogonal logit noise with zero true leakage,
the DiD estimate stays at or below zero at every noise scale, as \Cref{cor:no-free} requires
(whiskers: 2.5--97.5 percentile range across repetitions). (b) T6: the stratified DiD (bars)
tracks the predicted $\E[b_0\,w(2-w)]$ (dashed) across crowd-surprise quartiles at fixed
$w=0.5$.}
\label{fig:e1-supp}
\end{figure}

\subsection{Hubble robustness (M2)}
\label{app:supp-hubble}

\paragraph{Intuition.} The M2 dose-response supports the mechanism only if it is not an artifact of
item difficulty, a single model configuration, or a single task.

\paragraph{Setup.} All quantities are computed from the published Hubble accuracies as configured
in \Cref{app:e2}.

\paragraph{Results.}
\emph{Difficulty-matching check.} The standard model's accuracy varies by only
$0.037$--$0.051$ across duplication bins (per task: MMLU $0.046$, PIQA $0.046$, HellaSwag $0.051$,
WinoGrande $0.037$). This range is much smaller than the moderate/high-dose effects, bounding
difficulty imbalance as an explanation; it is comparable to the low-dose $r{=}4$ effect, so low-dose
conclusions should be interpreted more cautiously.

\emph{Robustness across scale and tokens.} The dose-response persists across configurations:
8B-100B shows stronger low-dose sensitivity (less token dilution reduces the contamination signal per
document), 8B-500B remains strong at moderate/high dose, and 1B-500B shows large effects only at
high duplication (lower capacity limits usable extraction). This heterogeneity is consistent with the
extraction factor $w=c\cdot u\cdot\kappa$: exposure must be both absorbed and usable.

\emph{Per-task consistency.} For 8B-500B, the high-dose effect is not driven by a single task:
MMLU, HellaSwag, and WinoGrande all show strong increases, while PIQA rises more mildly and
saturates earlier. Task heterogeneity is expected because $w$ can vary by task format.

\subsection{LiveCodeBench robustness battery (M4-C)}
\label{app:supp-e3}

\paragraph{Intuition.} The M4-C boundary estimate is credible only if the no-discontinuity
control assumption holds empirically, the signal is not an artifact of one control model or one
bandwidth, the effect is specific to documented cutoffs, and difficulty imbalance is excluded.
The named continuity violations of Route 1---composition shifts and corpus-density jumps at
the cutoff---are exactly what the placebo and bandwidth checks below target.

\paragraph{Setup.} All checks use the DiD estimator and bootstrap of \Cref{app:e3}, computed on
the published per-problem data.

\paragraph{Results.}
\emph{R1 (control smoothness).} The no-discontinuity assumption is directly tested by estimating
local-linear jumps for each control at each target cutoff. Pooled-control jumps: Oct'23
$-0.010$ [$-0.093$,$+0.077$]; Dec'23 $-0.012$ [$-0.097$,$+0.070$]; Apr'24 $-0.063$
[$-0.180$,$+0.040$]. All intervals include zero, supporting the identification condition.

\emph{R2 (separate-control robustness).} Own-cutoff estimates are computed against each control
individually. For GPT-4o-0806: $+0.094$ vs.\ Gemini, $+0.086$ vs.\ DeepSeek, $+0.090$ pooled. The
signal is not driven by a single control model.

\emph{R3 (date specificity).} A monthly cutoff scan finds no month outside $\pm1$ month of
GPT-4o's documented cutoff with a larger DiD. A permutation test over model-cutoff assignments in
the specificity matrix yields empirical $p=0.025$ across all targets.

\emph{R4 (bandwidth).} Own-cutoff DiD vs.\ pooled controls across half-widths $h\in\{100,120,
160,200,240,300\}$ days: GPT-4o point estimates are stable at $+0.084$ to $+0.100$; significance
appears for $h\ge160$ days.

\emph{R5 (difficulty stratification).} DiD stratified by problem difficulty (easy/medium/hard):
GPT-4o-0513 $+0.096$ [$+0.018$,$+0.169$]; GPT-4o-0806 $+0.089$ [$+0.015$,$+0.166$];
Claude-3.5 $+0.079$ [$-0.002$,$+0.156$]. GPT-4o survives cleanly; Claude becomes marginal after
stratification, so per-date inference for Claude leans on the date-permutation test
($p=0.017$).

\emph{Specificity.} For each continuity target, the DiD is computed at each candidate cutoff
date. Positive effects concentrate at the model's own documented cutoff (main text,
\Cref{sec:exp:e3}); off-diagonal cells are null, and monthly scans find no larger effect away
from the documented cutoffs.

\begin{figure}[t]\centering
\includegraphics[width=0.85\linewidth]{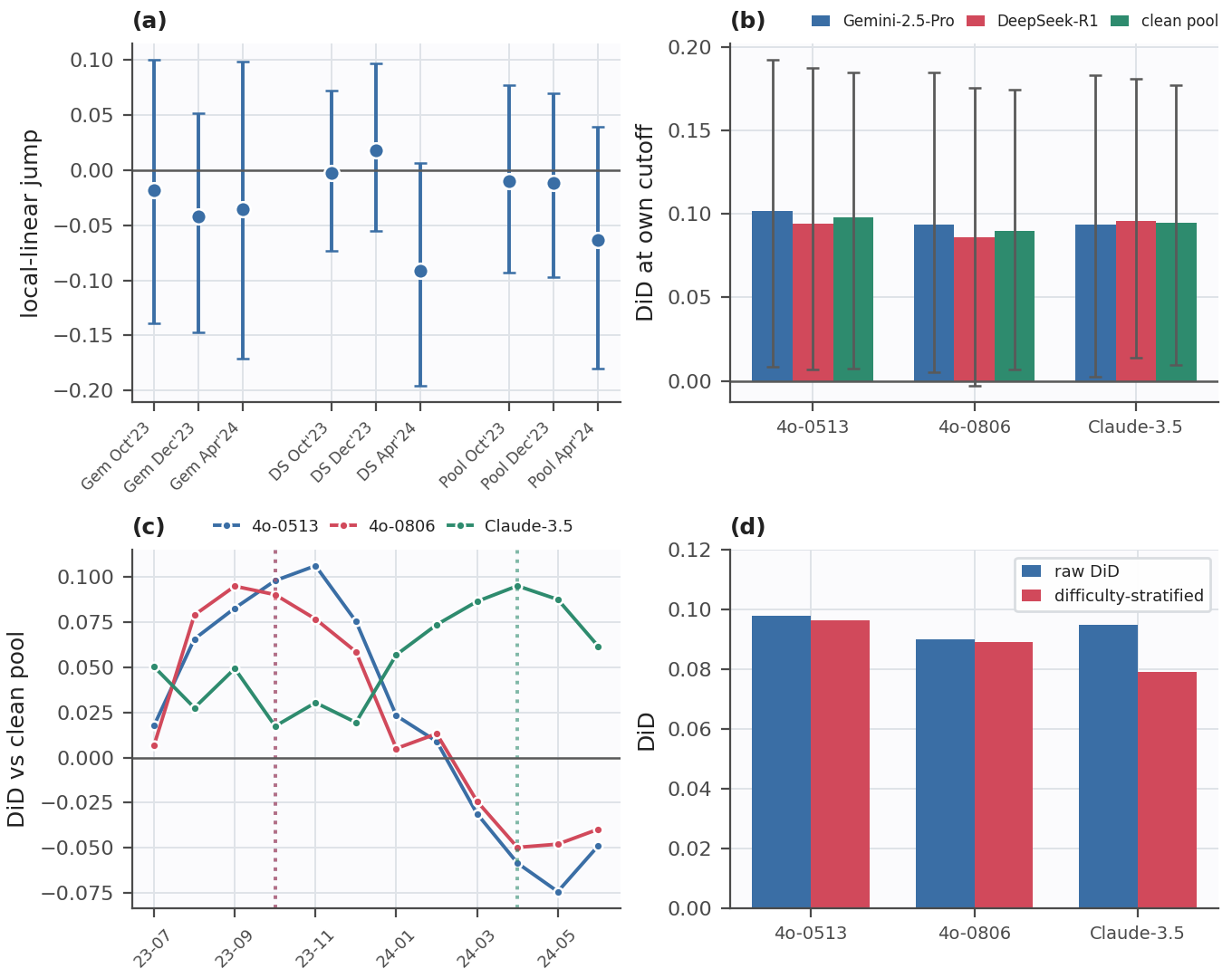}
\caption{\textbf{M4-C: the LiveCodeBench boundary signal survives the robustness battery.}
(a) Composition controls show no local discontinuity at target cutoffs. (b) Own-cutoff
estimates are stable across individual controls. (c) Monthly scans concentrate positive effects
at documented cutoffs. (d) Difficulty stratification preserves GPT-4o's signal and weakens
Claude's. Result: GPT-4o's boundary leakage is robust on every check; Claude's is reported as
supporting evidence only.}
\label{fig:e3-robustness}
\end{figure}

\subsection{Why M4-C does not identify global \texorpdfstring{$B$}{B}}
\label{app:supp-e3global}

\paragraph{Intuition.} Converting the boundary jump $\widehat J$ into a global leakage-adjusted
score requires either a strictly clean, recency-matched control (\Cref{ass:transport}) or the
extrapolation assumption (\Cref{ass:smooth}). Both routes fail on this dataset; documenting the
failure disciplines the boundary-only scope of M4-C's claim.

\paragraph{Setup.} Same data and estimator as \Cref{app:e3}; the strictly clean control is
GPT-4-0613 (cutoff before all problems); extrapolations fit time polynomials plus
difficulty/platform indicators on the clean side.

\paragraph{Results.} GPT-4-0613 is strictly clean on all LiveCodeBench questions and produces
large DiD estimates ($+0.145$ to $+0.168$), but its quarterly performance profile correlates only
$0.36$--$0.39$ with GPT-4o (centered RMSE $0.074$--$0.076$), violating the matched-recency
requirement and exhibiting the weak-control trap. Clean-side extrapolation is unstable: implied
corrections range from $-0.175$ to $+0.102$ for GPT-4o-0513 and from $-0.105$ to $+0.172$ for
GPT-4o-0806 across polynomial degrees and post-cutoff horizons. We therefore report boundary $J$
and do not convert it into a global leakage-adjusted score.

\begin{figure}[t]\centering
\includegraphics[width=0.75\linewidth]{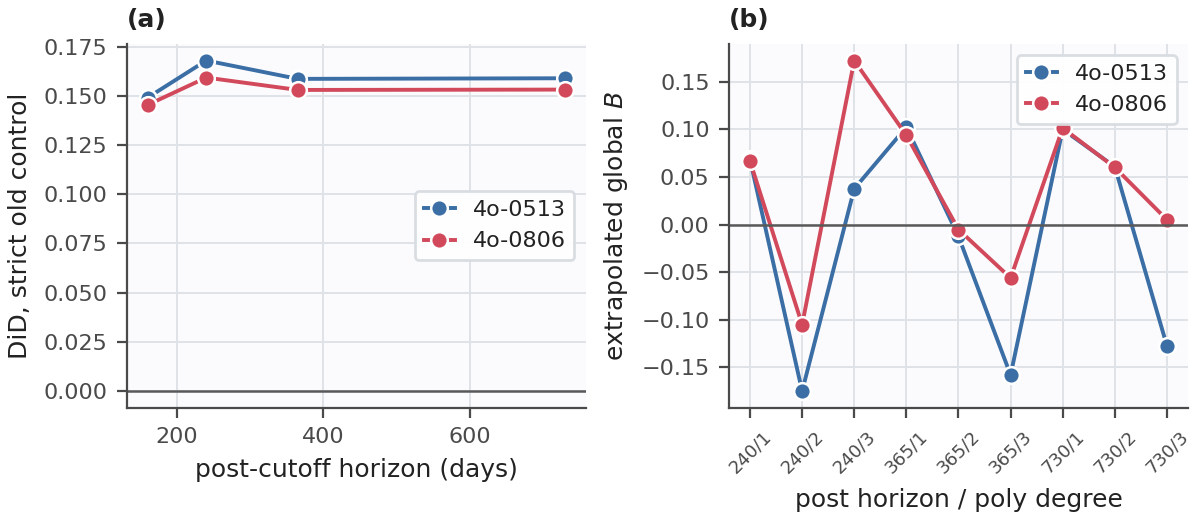}
\caption{\textbf{M4-C identifies the local boundary jump, not a global correction.}
(a) The strictly clean GPT-4-0613 control yields a large contrast but fails temporal-profile
matching (\Cref{ass:transport}). (b) Clean-side extrapolated corrections change sign across
horizons and polynomial degrees (\Cref{ass:smooth}). Result: both routes to a global $B$ fail
on this dataset, so the M4-C claim is scoped to the boundary.}
\label{fig:e3-global}
\end{figure}

\subsection{Forecasting robustness and power (M5)}
\label{app:supp-e4}

\paragraph{Intuition.} A null is only informative if the control is demonstrably well matched, the
result is stable across boundary placements, and the design has power to detect effects of
practical size.

\paragraph{Setup.} All quantities use the paired cluster-bootstrapped DiD of \Cref{app:m5}.

\paragraph{Results.}
\emph{Control validation.} The pre-registered profile-matching protocol selects GPT-5 over
Gemini-3.1-Pro (post-boundary profile distance $0.0051$ against $0.0088$) before any DiD is read.
Substituting the deliberately mismatched GPT-3.5-Turbo control flips the primary adjusted
estimate from $+0.020$ to $+0.055$, the spurious positive that an ill-matched control
manufactures under \Cref{ass:transport}.

\emph{Boundary sensitivity.} Adjusted estimates for boundaries swept from $-180$ to $+180$ days
around July 2025 in 30-day steps range from $-0.035$ to $+0.057$ and are never significantly
positive; placements more than 90 days before the reference boundary retain fewer than 62
pre-boundary questions and are reported for completeness only. The null does not depend on exact
placement of Qwen3.5's undocumented knowledge horizon. The sweep is informative only about
leakage that \emph{differs} across the assumed boundary: because every tested placement lies
inside the plausible training window, inflation roughly uniform over the window would produce a
null at all of them.

\emph{Power analysis.} \rev{We report each null with its power floor, the normal-approximation minimum detectable
effect for 80\% power in a two-sided 5\% test. This floor is $2.8$ times the cluster-bootstrap standard error of its estimate. The
floors are $0.098$ (Qwen3.5), $0.047$ (Kimi-K2.6), $0.062$ (DeepSeek-V3.2), $0.052$ (GLM-4.7), and
$0.046$ (MiniMax-M3). \Cref{tab:g2} lists them with the M4-F floors. A null can miss effects below its floor. Its CI upper bound gives the effects that the data exclude, if the control is valid.}

\emph{Second candidate control.} \rev{Against Gemini-3.1-Pro, the candidate that profile matching ranked second, all five adjusted estimates are also null. They range from $-0.004$ to $+0.050$ (Qwen3.5, CI $[-0.030,+0.095]$).}

\subsection{Matched real-data PRC sensitivity}
\label{app:prc-real}

\paragraph{Intuition.} E1-T3 shows PRC detects evidence leakage in principle; the question is
whether Route~3 yields a real calibrated positive on frontier data, under a pre-specified
robustness gate.

\paragraph{Setup.} A separate matched experiment was pre-specified before querying: Qwen3.5 with
$K=8$ paraphrases on outcome-balanced treatment/control cohorts in two domains with temporal
support on both sides, finance ($150/150$) and other ($130/130$); 552 of 560 questions had at
least six valid paraphrase probabilities.

\paragraph{Estimator and uncertainty.} Calibration was learned only from clean controls:
five-fold cross-fitting gave out-of-sample control predictions, and the five control-trained
calibrators were ensembled on treatment questions. Calibration was applied to each paraphrase
before computing consensus and within-question variance. The estimator was
$\widehat\Delta_{\mathrm{bc}}
=\widehat{\Cov}(\bar P,Y-\bar P)+\tfrac1n\sum_i\widehat W_i/K_i$.
Nested question bootstraps refit the calibrator and recomputed the finite-$K$ correction in every
draw; Platt scaling was primary, isotonic the pre-specified sensitivity.

\paragraph{Results.} Under Platt scaling, finance yielded excess PRC $+0.0116$
[$+0.0009,+0.0169$], other yielded $+0.0084$ [$-0.0046,+0.0164$], and the equal-weight pooled
excess was $+0.0100$ [$+0.0018,+0.0148$]. Clean-control point estimates were near zero. Under
isotonic calibration, however, finance was $+0.0142$ [$-0.0097,+0.0188$], other reversed to
$-0.0026$ [$-0.0264,+0.0124$], and the pooled excess was $+0.0058$
[$-0.0120,+0.0123$]. The experiment therefore failed its pre-specified calibration-stability
gate and is reported as suggestive sensitivity evidence only; Route~3 remains a
calibration-gated secondary diagnostic.

\begin{figure}[t]\centering
\includegraphics[width=0.80\linewidth]{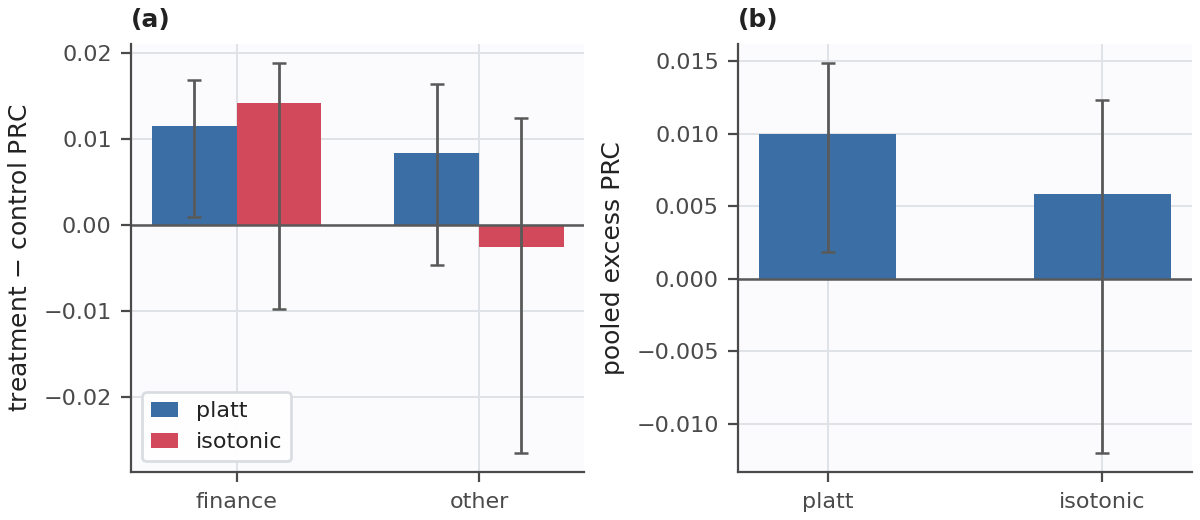}
\caption{\textbf{Matched real-data PRC: the calibration-stability gate fails, so PRC remains
secondary.} Treatment--control PRC excess by domain and pooled, under Platt and isotonic
calibration. Result: Platt yields positive excesses, but isotonic does not replicate the pooled
significance and flips the other-domain point estimate; by the pre-specified gate, PRC is
reported as suggestive sensitivity evidence only.}
\label{fig:prc-matched}
\end{figure}

\subsection{\texorpdfstring{\rev{GPT-5.5 robustness, power floors, and recency slopes}}{GPT-5.5 robustness, power floors, and recency slopes}}
\label{app:e9-robust}

\rev{\paragraph{Intuition.} This section reports three checks of the deployment results. The first check
asks if the GPT-5.5 jump depends on analysis choices (\Cref{tab:g1}). The second check gives the
power floor of each null (\Cref{tab:g2}). The third check asks if a gap in honest
trends between target and reference can produce or hide a jump (\Cref{tab:g3}). All three checks use
released outputs, with no new model queries.

\paragraph{GPT-5.5 sensitivity.} The primary estimate is the pre-specified unbanded jump of
\Cref{sec:exp:m4anchor} at $\pm90$ days. \Cref{tab:g1} changes the window, the horizon band, the
anchor, the control, and the assumed boundary. For a shifted boundary, the prediction is the
primary jump times the difference between the shares of pre-cutoff questions on the two
sides of the shifted boundary. This prediction assumes a single jump on 1 December 2025.

\paragraph{Trend test.} If the honest advantage of the target over its reference changes by $s$
per month, a pre/post estimate shifts by $-s\,\Delta t$. So $s<0$ overstates leakage, and $s>0$ hides it. Here $\Delta t$ is the time in months between the mean resolution dates after and before the
cutoff. We estimate $s$ by least squares of
the per-question difference on the resolution date, with intervals from the source-by-month
cluster bootstrap. We use only questions that are clean for both models:
\begin{itemize}
\item December 2025 to July 2026 for the archived-forecast design ($n=992$ unbanded, $738$
banded).
\item February 2026 to July 2026 for M4-F ($n=721$), after the January 2026 cutoffs of the
controls.
\item From 24 February 2026 for M5 ($n=685$ to $686$), after the open-weight release of
Qwen3.5-35B-A3B, because its frozen weights cannot contain later outcomes.
\end{itemize}

For GPT-5.5, the clean window starts at the boundary. We apply the slope to the whole window,
and the adjusted jump is $+0.074$ (CI $[+0.052,+0.123]$). For the M4-F and M5 boundaries, the
clean window starts months later, so the slope shows only the direction of the bias. In M4-F, the joint permutation test of \Cref{tab:m4f} compares the own cutoff of each
model with its other boundaries. Thus most of a linear gap cancels. For Qwen3.5 against GPT-5, the
measured slope is negative. If this sign also holds at the boundary, it can only inflate the estimate. For GPT-5.4 in M4-F, and for DeepSeek-V3.2, GLM-4.7, and MiniMax-M3 in M5, the CI of $s$ includes the needed value. So these nulls cannot exclude a trend gap that hides leakage.

\paragraph{Placebos and shifted boundaries.} The 1 April 2026 placebo gives a significant jump of $+0.030$ (\Cref{tab:g1}). It is also starred with $\pm60$-day ($+0.041$) and $\pm120$-day ($+0.021$) windows. Over all windows, jumps at dates without a cutoff reach $0.048$ in absolute size, including the DeepSeek-V3.1 protocol placebos of \Cref{app:m4}. The primary jump ($+0.078$) and its lower CI bound ($+0.057$) are above all of them. Of the shifted boundaries, $+30$ and $+60$ days agree with the prediction, and $-30$ days is consistent with it but also with zero. At $-60$ days, with only 18 questions before the shifted boundary, the CI excludes the prediction.}

\setcounter{savedtable}{\value{table}}
\setcounter{table}{0}
\renewcommand{\thetable}{G\arabic{table}}

\begin{table}[H]\centering
\caption{\rev{\textbf{GPT-5.5 jump under alternative choices.} Unbanded unless stated. Control rows use the
cross-model design of \Cref{tab:m4f}. Predicted: the jump that a single cutoff on 1 December 2025
gives at a shifted boundary. $^\ast$: the 95\% cluster-bootstrap CI excludes zero.}}
\label{tab:g1}
\footnotesize
\rev{\begin{tabular}{llccc}
\toprule
Check & Setting & Jump [95\% CI] & $n$ pre/post & Predicted \\
\midrule
Primary & $\pm90$ d & $+0.078^\ast$ $[+0.057,+0.124]$ & 60/339 & \\
Window & $\pm60$ d & $+0.087^\ast$ $[+0.067,+0.218]$ & 47/284 & \\
 & $\pm120$ d (submitted setting) & $+0.061^\ast$ $[+0.024,+0.093]$ & 65/434 & \\
Horizon band (deviation) & $\pm90$ d & $+0.108^\ast$ $[+0.057,+0.192]$ & 31/300 & \\
 & $\pm120$ d (submitted setting) & $+0.106^\ast$ $[+0.056,+0.188]$ & 31/361 & \\
Anchor & GPT-5-Mini only & $+0.086^\ast$ $[+0.063,+0.136]$ & 60/339 & \\
Control & MiniMax-M3 & $+0.036$ $[-0.003,+0.065]$ & 130/438 & \\
 & Claude-Opus-4.7 & $+0.021$ $[-0.048,+0.127]$ & 43/113 & \\
Boundary & $-60$ d & $-0.017$ $[-0.257,+0.028]$ & 18/101 & $+0.042$ \\
 & $-30$ d & $+0.052$ $[-0.075,+0.119]$ & 35/314 & $+0.070$ \\
 & $+30$ d & $+0.037^\ast$ $[+0.009,+0.080]$ & 101/380 & $+0.036$ \\
 & $+60$ d & $+0.005$ $[-0.021,+0.044]$ & 314/319 & $+0.007$ \\
Placebo & 1 Mar 2026 & $+0.007$ $[-0.020,+0.041]$ & 339/400 & $0$ \\
 & 1 Apr 2026 & $+0.030^\ast$ $[+0.004,+0.065]$ & 285/404 & $0$ \\
\bottomrule
\end{tabular}}
\end{table}

\begin{table}[H]\centering
\caption{\rev{\textbf{Results without significant positive leakage and their power floors.}
Floor: the normal-approximation minimum detectable effect ($2.8\times$ cluster-bootstrap
SE) for 80\% power in a two-sided 5\% test. A null can miss effects below its floor. The CI upper bound gives the effects that the data exclude.
DeepSeek-V3.2 and GLM-4.7 are significantly negative (\Cref{tab:m5}).}}
\label{tab:g2}
\footnotesize
\rev{\begin{tabular}{llcc}
\toprule
Design & Model & Estimate [95\% CI] & Floor \\
\midrule
M4-F own cutoff & DeepSeek-V3.1 & $-0.009$ $[-0.055,+0.055]$ & $0.077$ \\
 & Kimi-K2.6 & $+0.022$ $[-0.092,+0.066]$ & $0.111$ \\
 & GPT-5.4 & $-0.051$ $[-0.130,+0.007]$ & $0.098$ \\
 & GPT-5.5 & $+0.034$ $[-0.007,+0.070]$ & $0.053$ \\
M5 adjusted DiD & Qwen3.5 & $+0.020$ $[-0.063,+0.074]$ & $0.098$ \\
 & Kimi-K2.6 & $-0.016$ $[-0.058,+0.010]$ & $0.047$ \\
 & DeepSeek-V3.2 & $-0.034$ $[-0.088,-0.000]$ & $0.062$ \\
 & GLM-4.7 & $-0.031$ $[-0.076,-0.002]$ & $0.052$ \\
 & MiniMax-M3 & $-0.000$ $[-0.039,+0.026]$ & $0.046$ \\
\bottomrule
\end{tabular}}
\end{table}

\begin{table}[H]\centering
\caption{\rev{\textbf{Measured trend gap against the gap that changes each conclusion.} $s$ in score units
per month, $\Delta t$ in months. The needed $s$ moves the lower CI bound of the detection to zero,
or, for a null, gives a bias of $-0.05$.}}
\label{tab:g3}
\footnotesize
\rev{\begin{tabular}{lccc}
\toprule
Result & $\Delta t$ & Measured $s$ [95\% CI] & Needed $s$ \\
\midrule
GPT-5.5 detection, unbanded & 2.2 & $-0.0015$ $[-0.0062,+0.0031]$ & $\le-0.026$ \\
GPT-5.5 detection, banded & 2.4 & $-0.0002$ $[-0.0070,+0.0068]$ & $\le-0.024$ \\
M4-F DeepSeek-V3.1 & 3.7 & $+0.0045$ $[-0.0029,+0.0119]$ & $\ge+0.013$ \\
M4-F Kimi-K2.6 & 1.9 & $-0.0037$ $[-0.0160,+0.0079]$ & $\ge+0.026$ \\
M4-F GPT-5.4 & 3.2 & $+0.0051$ $[-0.0096,+0.0210]$ & $\ge+0.016$ \\
M4-F GPT-5.5 & 2.5 & $-0.0007$ $[-0.0160,+0.0133]$ & $\ge+0.020$ \\
M5 Qwen3.5 & 9.5 & $-0.0149$ $[-0.0379,-0.0026]$ & $\ge+0.005$ \\
M5 Kimi-K2.6 & 9.5 & $-0.0075$ $[-0.0304,+0.0026]$ & $\ge+0.005$ \\
M5 DeepSeek-V3.2 & 9.5 & $+0.0013$ $[-0.0209,+0.0140]$ & $\ge+0.005$ \\
M5 GLM-4.7 & 9.5 & $-0.0022$ $[-0.0211,+0.0120]$ & $\ge+0.005$ \\
M5 MiniMax-M3 & 9.5 & $-0.0059$ $[-0.0296,+0.0085]$ & $\ge+0.005$ \\
\bottomrule
\end{tabular}}
\end{table}

\setcounter{table}{\value{savedtable}}
\renewcommand{\thetable}{\arabic{table}}

\section{Experiment provenance}
\label{app:provenance}

\begin{table}[H]\centering
\footnotesize
\caption{\textbf{Reproducibility map.} Authoritative scripts and outputs for reported experiments;
paths are relative to the code repository.}
\begin{tabular}{@{}>{\raggedright\arraybackslash}p{3.2cm}p{6.7cm}p{5.5cm}@{}}
\toprule
Result & Script & Primary output \\
\midrule
Panel build & \path{experiments/build_panel.py} & \path{experiments/data/panel_market.jsonl} \\
M1 clean flagships & \path{experiments/analysis/m1_analyze.py} & \path{experiments/m1/results.json} \\
M2 Hubble & \path{hubble/run_m2.py} & \path{hubble/results.json} \\
M3 corpus and twins & \path{experiments/m3/build_corpus.py}, \path{train_twin.py} & (rebuilt from the panel) \\
M3 analysis & \path{experiments/m3/analyze_m3.py} & \path{experiments/m3/analysis_full.json} \\
M3 PRC & \path{experiments/m3/prc.py} & \path{experiments/m3/prc_differenced.json} \\
M4-F matrix, anchor & \path{experiments/analysis/m4f_{analyze,anchor}.py} & \path{experiments/m4f/*.json} \\
M4-C & \path{experiments/m4c/analyze.py} & \path{experiments/m4c/m4c_results.json} \\
M5 & \path{experiments/analysis/m5_analyze.py} & \path{experiments/m5/results.json} \\
E1 synthetic (T1--T6) & \path{synthetic/run_synthetic.py} & \path{synthetic/results.json} \\
Matched PRC & \path{prc_matched/run_prc_matched.py} & \path{prc_matched/results.json} \\
Main-text figures & \path{experiments/analysis/paper_figures.py} & \path{experiments/figures/*.png} \\
Appendix figures & \path{experiments/analysis/appendix_figures.py} & \path{experiments/figures/*.png} \\
\rev{Trend test (E.9)} & \rev{\path{experiments/analysis/e9_trends.py}} & \rev{\path{experiments/e9/trends.json}} \\
\rev{GPT-5.5 sensitivity (E.9)} & \rev{\path{experiments/analysis/e9_gpt55_sensitivity.py}, \path{e9_gpt55_prediction.py}} & \rev{\path{experiments/e9/gpt55_*.json}} \\
\rev{Power floors (E.7, E.9)} & \rev{\path{experiments/analysis/e9_power_controls.py}} & \rev{\path{experiments/e9/power_controls.json}} \\
\rev{Anchor-arm windows (D.5)} & \rev{\path{experiments/analysis/e9_anchor_windows.py}} & \rev{\path{experiments/e9/anchor_windows.json}} \\
\rev{Jump decomposition under spillover (D.4)} & \rev{\path{experiments/analysis/e9_spillover.py}} & \rev{\path{experiments/e9/spillover.json}} \\
\bottomrule
\end{tabular}
\end{table}

\end{document}